\pdfoutput=1
 \def\arXiv{1} 
 \documentclass[11pt]{article} 
\newcommand{\notarxiv}[1]{foo}
\newcommand{\arxiv}[1]{ba}
\ifdefined \arXiv
\renewcommand{\arxiv}[1]{#1}%
\renewcommand{\notarxiv}[1]{\ignorespaces}%
\else%
\renewcommand{\arxiv}[1]{\ignorespaces}%
\renewcommand{\notarxiv}[1]{#1}%
\fi%

\usepackage{thm-restate}
\usepackage{xparse}
\usepackage{enumerate}
\usepackage{amssymb}
\usepackage{amsbsy}
\usepackage{amsthm}
\usepackage{amsfonts}
\usepackage{latexsym}
\usepackage{graphicx}
\usepackage{mathtools}
\usepackage[dvipsnames]{xcolor}
\usepackage{enumitem}
\usepackage{thmtools}
\usepackage{graphicx}
\usepackage{color}
\usepackage{bbm}
\usepackage{xifthen}
\usepackage{xspace}
\usepackage{mathtools}
\usepackage{titlesec}
\usepackage{setspace}
\usepackage{etoolbox}
\usepackage{xfrac}
\usepackage{esvect}
\usepackage{nicefrac}
\usepackage{cases}
\usepackage{empheq}
\usepackage{booktabs}
\usepackage{multirow}
\usepackage{caption}
\usepackage{bbm}
\usepackage{array}
\usepackage{amsthm}
\usepackage{xurl}
\usepackage[linesnumbered,ruled,vlined,boxed,algo2e]{algorithm2e}
\usepackage{subfig}
\usepackage{comment}

\usepackage{makecell}
\usepackage{tabularx}
\usepackage{arydshln}
\usepackage{soul}
\usepackage{adjustbox}
\usepackage{pdflscape}

\notarxiv{
	\usepackage{hyperref}
}
\arxiv{
	\usepackage[hidelinks]{hyperref}
	\hypersetup{
	colorlinks=true,
	linkcolor=blue!70!black,
	citecolor=blue!70!black,
	urlcolor=blue!70!black}
}

\usepackage{cleveref}
\crefname{assumption}{Assumption}{Assumptions}

 \notarxiv{
 	\theoremstyle{plain}
 	\newtheorem{theorem}{Theorem}[section]
 	\newtheorem{proposition}[theorem]{Proposition}
 	\newtheorem{lemma}[theorem]{Lemma}
 	\newtheorem{corollary}[theorem]{Corollary}
 	\theoremstyle{definition}
 	\newtheorem{definition}[theorem]{Definition}
 	\newtheorem{assumption}[theorem]{Assumption}
 	\theoremstyle{remark}
 	\newtheorem{remark}[theorem]{Remark}

}

\arxiv{
	\usepackage[top=1in, right=1in, left=1in, bottom=1in]{geometry}
	\usepackage[numbers,square]{natbib}

	\theoremstyle{plain}
	
	\newtheorem{theorem}{Theorem}
	\newtheorem{lemma}{Lemma}
	
	\newtheorem{assumption}{Assumption}
	\newtheorem{proposition}{Proposition}
	\newtheorem{corollary}{Corollary}
	
	\theoremstyle{definition}
	
	\newtheorem{definition}{Definition}

	\newtheorem*{example*}{Example}

}

\DeclarePairedDelimiter{\abs}{\lvert}{\rvert} %
\DeclarePairedDelimiter{\brk}{[}{]}
\DeclarePairedDelimiter{\crl}{\{}{\}}
\DeclarePairedDelimiter{\prn}{(}{)}

\DeclarePairedDelimiter{\norm}{\|}{\|}

\DeclarePairedDelimiter{\ceil}{\lceil}{\rceil}
\DeclarePairedDelimiter{\floor}{\lfloor}{\rfloor}

\newcommand{\overeq}[1]{\overset{#1}{=}}
\newcommand{\overle}[1]{\overset{#1}{\le}}

\NewDocumentCommand\Ex{s O{} m }{%
	\mathbb{E}%
	\begingroup
	\IfBooleanTF{#1}
	{\ExInn*{#3}}
	{\ExInn[#2]{#3}}%
	\endgroup
}

\DeclarePairedDelimiterX\ExInn[1]{[}{]}{%
	\activatebar
	#1%
}

\RenewDocumentCommand\Pr{sO{}r()}{%
	\mathbb{P}%
	\begingroup
	\IfBooleanTF{#1}
	{\PrInn*{#3}}
	{\PrInn[#2]{#3}}%
	\endgroup
}

\DeclarePairedDelimiterX\PrInn[1](){%
	\activatebar
	#1%
}

\newcommand{\activatebar}{%
	\begingroup\lccode`~=`|
	\lowercase{\endgroup\def~}{\;\delimsize\vert\;}%
	\mathcode`|=\string"8000
}

\newcommand\numberthis{\addtocounter{equation}{1}\tag{\theequation}}

\newcommand{\R}{\mathbb{R}} %
\newcommand{\N}{\mathbb{N}} %
\newcommand{\E}{\mathbb{E}} %
\renewcommand{\P}{\mathbb{P}}	%
\newcommand{\e}{e}

\usepackage{xcolor}

\SetCommentSty{mycommfont}
\SetKwInput{KwInput}{Input} 
\SetKwInput{KwParameter}{Parameters}                %
\SetKwInput{KwOutput}{Output}              %
\SetKwInput{KwReturn}{Return}              %
\SetKwComment{Comment}{$\triangleright$\ }{}
\SetKwInOut{KwParameters}{Parameters}
\SetCommentSty{mycommfont}
 \DeclareMathOperator*{\argmax}{arg\,max}
 \DeclareMathOperator*{\argmin}{arg\,min}

\renewcommand{\dim}{m}
\newcommand{\X}{\mathcal{X}}

\providecommand{\abs}{\mathop{\rm abs}}

\newcommand{\half}{\frac{1}{2}}

\newcommand{\defeq}{\coloneqq}

\newcommand{\grad}{\nabla}

\newcommand{\xset}{\mathcal{X}}

\newcommand{\inner}[2]{\left<#1,#2\right>}

\newcommand{\numSteps}{T}

\newcommand{\G}[1][\numSteps]{G_{#1}}
\newcommand{\rbar}[1][\numSteps]{\bar{r}_{#1}}
\newcommand{\Lbar}[1][\numSteps]{\bar{\ell}_{#1}}
\newcommand{\Lfunc}[0]{\ell}
\newcommand{\LL}[0]{L}
\newcommand{\LLstar}[0]{\LL_\star}
\newcommand{\dbar}[1][\numSteps]{\bar{d}_{#1}}
\newcommand{\xbar}[1][t]{\bar{x}_{#1}}
\renewcommand{\d}[1][0]{d_{#1}}
\renewcommand{\r}[1][0]{r_{#1}}
\newcommand{\TimeUniformLog}[1][t,\delta]{\theta_{#1}}
\newcommand{\hx}{\hat{x}}

\newcommand{\D}{3 \d[0]}
\newcommand{\etaTilde}{\tilde{\eta}}

\newcommand{\filt}{\mathcal{F}}
\NewDocumentCommand{\llft}{ O{\delta} O{t} }{\lambda_{#2}(#1)}



\newcommand{\opt}{_\star}
\newcommand{\xopt}{x\opt}
\newcommand{\fopt}{f\opt}

\makeatletter
\def\IfEmptyTF#1{%
	\if\relax\detokenize{#1}\relax
	\expandafter\@firstoftwo
	\else
	\expandafter\@secondoftwo
	\fi}
\makeatother

\newcommand{\gradientOracle}[1]{\IfEmptyTF{#1}{\mathcal{G}}{\mathcal{G}(#1)}}

\newcommand{\reps}{r_\epsilon}

\newcommand{\Proj}[2]{\mathbf{\mathrm{Proj}}_{#1}(#2)}

\newcommand{\DoG}{\textsc{DoG}\xspace}
\newcommand{\TDoG}{\textsc{T-DoG}\xspace}
\newcommand{\LDoG}{\textsc{L-DoG}\xspace}

\notarxiv{\usepackage[accepted]{icml2023}}

\notarxiv{
	\icmltitlerunning{\DoG is SGD's Best Friend}
	
\begin{document}
		
		\twocolumn[
		\icmltitle{\DoG is SGD's Best Friend: A Parameter-Free Dynamic Step Size Schedule}

		\begin{icmlauthorlist}
			\icmlauthor{Maor Ivgi}{ta}
			\icmlauthor{Oliver Hinder}{pitt}
			\icmlauthor{Yair Carmon}{ta}
		\end{icmlauthorlist}
		
		\icmlaffiliation{ta}{Tel Aviv University}
		\icmlaffiliation{pitt}{University of Pittsburgh}
		
		\icmlcorrespondingauthor{Maor Ivgi}{maor.ivgi@cs.tau.ac.il}

		\icmlkeywords{Machine Learning, ICML}
		
		\vskip 0.3in
		]

		\printAffiliationsAndNotice{} %
}

\arxiv{
	\title{\DoG is SGD's Best Friend:\\A Parameter-Free Dynamic Step Size Schedule}
	
	\author{
	Maor Ivgi\\	\href{mailto:maor.ivgi@cs.tau.ac.il}{\texttt{maor.ivgi@cs.tau.ac.il}}
	\and
 	Oliver Hinder \\
	\href{ohinder@pitt.edu}{\texttt{ohinder@pitt.edu}} 
	\and
	Yair Carmon\\
	\href{mailto:ycarmon@tauex.tau.ac.il}{\texttt{ycarmon@tauex.tau.ac.il}}
}
	
	\date{}
	
	\begin{document}
		
	\maketitle
}

\begin{abstract}
We propose a tuning-free dynamic SGD step size formula, which we call Distance over Gradients (DoG). The DoG step sizes  depend on simple empirical quantities (distance from the initial point and norms of gradients) and have no ``learning rate'' parameter. Theoretically, we show that a slight variation of the DoG formula enjoys strong parameter-free convergence guarantees for stochastic convex optimization assuming only \emph{locally bounded} stochastic gradients.\footnote{The ICML version of our paper uses the more conventional and restrictive assumption of globally bounded stochastic gradients.} Empirically, we consider a broad range of vision and language transfer learning tasks, and show that DoG's performance is close to that of SGD with tuned learning rate. We also propose a per-layer variant of DoG that generally outperforms tuned SGD, approaching the performance of tuned Adam.
A PyTorch implementation is available at \href{https://github.com/formll/dog}{https://github.com/formll/dog}.
\end{abstract}

\section{Introduction}

\begin{figure}[t]
	\begin{center}
		\centering
		\arxiv{
			\includegraphics[width=0.35\columnwidth]{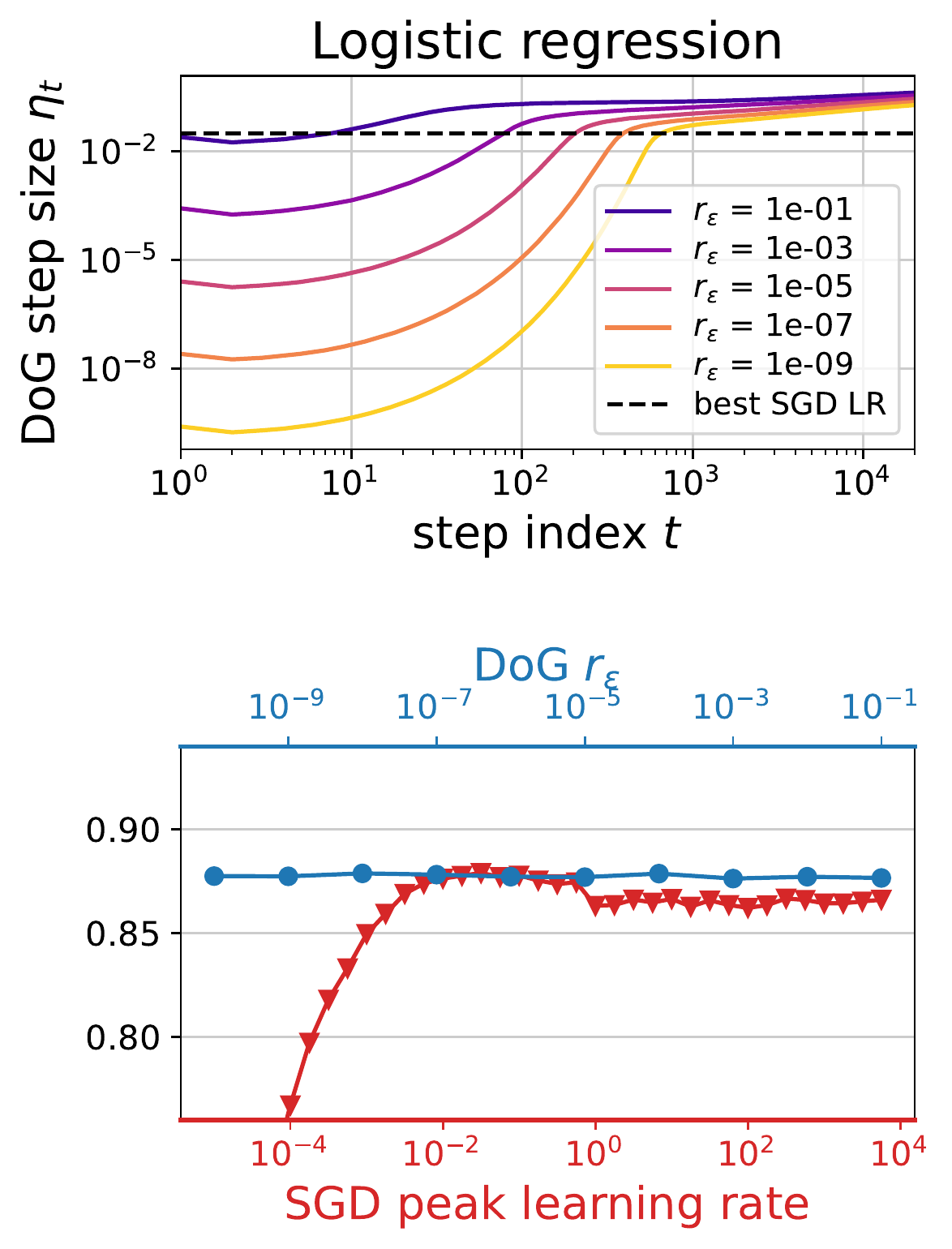}
			\includegraphics[width=0.35\columnwidth]{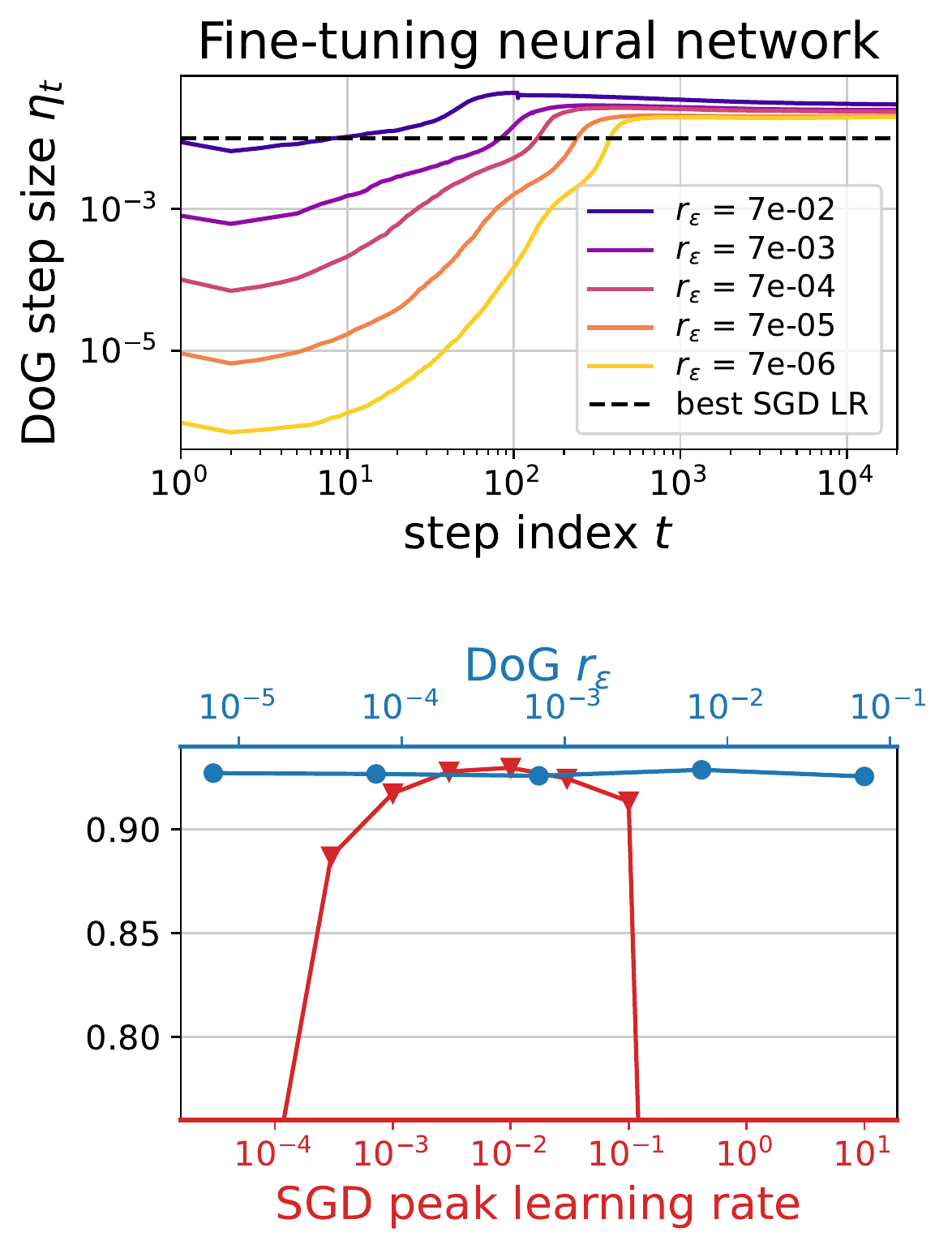}
		}
		\notarxiv{
			\includegraphics[width=0.49\columnwidth]{figures/fig1_cvx.pdf}
			\includegraphics[width=0.49\columnwidth]{figures/fig1_not_cvx.pdf}
			\vspace{-8pt}
		}
		\caption{Illustration of \DoG for CIFAR-100 classification using logistic regression on last-layer features of a 	pre-trained ViT-B/32 (left) or end-to-end fine-tuning of the model (right). The top row shows the \ref{eq:dog} step size sequence $\eta_t$ for different values of the initial movement $\reps$, and the bottom row shows that \DoG attains test error on par with carefully tuned SGD (with cosine annealing), even when varying $\reps$ by several orders of magnitude. See details in \Cref{app:experiments-figure-1}.}
		\label{fig:1}
	\end{center}
	\notarxiv{\vspace{-18pt}}
\end{figure}

While stochastic optimization methods drive continual improvements in machine learning, choosing the optimization parameters---and particularly the learning rate---remains a difficulty. Standard methodologies include searching over a set of learning rates, or simply picking the learning rate from prior work. The former incurs a substantial computational overhead, while the latter risks training a suboptimal model. 

The rich literature on adaptive gradient methods (AdaGrad, Adam, and their many variants) 
offers optimization algorithms that better exploit problem structure~\citep[e.g.,][]{duchi2011adaptive,kingma2015adam,gupta2018shampoo,shazeer2018adafactor,loshchilov2019decoupled}.
However, these methods still have a learning rate parameter that requires tuning. The theoretically-optimal value of this parameter depends on unknown problem properties. For example, on convex problems the optimal learning rate of AdaGrad is related to the distance between the initial point and the optimal solution,
while in non-convex settings it is related to the function's smoothness and initial optimality gap~\citep{gupta2017unified,ward2019adagrad,faw2022power}. 

\emph{Parameter-free} optimization aims to remove the need for such tuning by designing algorithms that achieve a near-optimal rate of convergence with almost no knowledge of the problem properties~\citep{streeter2012no}. Most works in this field \citep[e.g.,][]{luo2015achieving,orabona2016coin,cutkosky2018black,mhammedi2020lipschitz,bhaskara2020online, jacobsen2022parameter, zhang2022pde} 
use advanced online learning techniques to construct algorithms that, for the fundamental setting of  stochastic convex optimization (SCO) with bounded stochastic gradients, achieve optimal rates of convergence up to logarithmic factors. 
While practical parameter-free algorithms exist \citep[e.g.][]{orabona2014simultaneous,orabona2017training,kempka2019adaptive,chen2022better}, there is little research into practical parameter-free step size selection methods for SGD. 
Recently, \citet{carmon2022making} have shown that performing a careful bisection over the SGD step size yields a parameter-free optimization method that is optimal for SCO up to a double-logarithmic factor. While theoretically novel, on a practical level the result leaves much to be desired, as it essentially prescribes the standard recipe of running SGD multiple times with different learning rates.

\paragraph{Proposed algorithm.}
In this work, we use key insights from \citet{carmon2022making} to go a step further and develop a parameter-free step size schedule. For SGD iterations of the form $x_{t+1}=x_t - \eta_t g_t$, where $x_t$ denotes the model parameters at the $t$'th iteration and $g_t$ denotes the stochastic gradient of the loss function, our proposed steps size sequence is (for all $t\ge 1$)
\begin{flalign}\tag{\DoG}\label{eq:dog}
	\eta_t = \frac{\max_{i\le t}\norm{x_i - x_0}}{\sqrt{\sum_{i\le t}\norm{g_i}^2}}.
\end{flalign} 
In words, the step size at iteration $t$ is the maximum distance to between the initial point and observed iterates, divided by the sum of squared stochastic gradient norms, i.e., Distance over Gradients (\DoG). At the first step, we set $\eta_0$ to be $\reps/\norm{g_0}$,  i.e., we take a normalized gradient step of size $\reps$; we show that, as long as $\reps$ is small, its precise setting has only mild effect.

Crucially, \DoG has no multiplicative ``learning rate'' parameter: if one considers step sizes of the form  $\eta_t = c \cdot \frac{\max_{i\le t}\norm{x_i - x_0}}{\sqrt{\sum_{i\le t}\norm{g_i}^2}}$
then $c=1$ is a universally good setting (see \Cref{sec:derivation} for a heuristic justification and \Cref{subsec:DoG-sensitivity} for empirical evidence for this claim).
\Cref{fig:1} highlights key aspects of \DoG. The top row shows the \DoG step size  sequence for different values of $\reps$ in convex (left) and non-convex (right) stochastic optimization problems. The \DoG step size increases rapidly (note the logarithmic $x$ scale) and stabilizes around values close to the optimal SGD step size with little dependence on $\reps$. 
The bottom row of the figure compares the test errors of \DoG and SGD with various step sizes, showing that (for all choices of $\reps$) \DoG is on par with well-tuned SGD. 

\subsection{Summary of results}

\newcommand{\Bset}{\mathcal{B}}
\paragraph{Theoretical guarantees.}
In \Cref{sec:theory} we analyze \DoG  for stochastic convex optimization with bounded stochastic gradients and a (potentially unbounded) closed convex domain. To present our results, let $\Bset$ denote a ball around the initial point $x_0$ with radius $3\d[0]$, where $\d[0]$ is the distance between $x_0$ and an optimum.

First, we show that if the iterates of \DoG remain in $\Bset$, then with high probability  \DoG achieves a convergence rate that is optimal up to a factor of $O\prn{\log \prn{ 1+ \frac{\d[0]}{\reps}}}$. In practice, \DoG appears to indeed be stable as long as $\reps$ is sufficiently small. However,  \DoG is not always stable: on pathological functions its iterates can move far from the optimum.

To address this, we consider a theoretical, tamed variant of \DoG, which we call \TDoG, whose step sizes are smaller by a logarithmic factor. We prove that, with high probability, the \TDoG iterates never leave $\Bset$. Thus, 
we obtain a high probability parameter-free convergence guarantee that is optimal up logarithmic factors.

To our knowledge, this is the first dynamic SGD step size schedule to attain such theoretical guarantee, and only the third high probability parameter-free guarantee in the literature \citep[following][]{carmon2022making,zhang2022parameter}.
Moreover, it is the first parameter-free result assuming only locally bounded stochastic gradients (i.e., in the set $\Bset$). This is significant since the usually-assumed global stochastic gradient bound does not exist in many problems, including least squares.

\paragraph{Empirical study.}
Our experiments in \Cref{sec:experiments} focus on fine-tuning neural networks, because this is a practically important setting that still allows for thorough experiments at a reasonable computational budget. 
We also perform a small-scale experiment with training a neural network from scratch.
Our experiments span 23 natural language understanding and image classification tasks and 8 popular model architectures.

Our results indicate that, compared to \DoG, SGD with a cosine step size schedule and tuned base learning 
rarely attains a relative error improvement of more than 5\% (e.g., the difference between accuracy 95\% and 95.25\%). 
For convex problems (linear probes), the relative difference in errors is below 1\%.
In our testbed, well-tuned Adam tends to outperform both SGD and \DoG, but a layer-wise version of \DoG (which we call \LDoG) closes some of this performance gap. 

We also test the sensitivity of \DoG to the value of $\reps$. We find that for most model/task combinations, \DoG performs consistently well across a wide range of $\reps$ values as our theory predicts. However, in certain cases, choosing $\reps$ to be too low results in poor performance. We provide some preliminary findings showing that this is due in part to batch normalization. %

\arxiv{\vspace{\baselineskip}}

Put together, our theory and experiments suggest \DoG has the potential to save significant computation currently spent on learning rate tuning at little or no cost in performance---especially if we reinvest some of the saved computation in training a larger model on more data.
\section{Algorithm Derivation}\label{sec:derivation}

Before providing rigorous theoretical guarantees for \DoG, in this section we explain the origin of the algorithm. 
Our starting point is the following result by~\citet{carmon2022making}. 
Suppose we run $T$ iterations of SGD with fixed step size $\eta$, i.e., the recursion $x_{t+1}=x_t - \eta g_t$, where $x_t$ is the SGD iterate and $g_t$ is the stochastic gradient at step $t$. If, for some $c\in (0,1)$, it happens to hold that
\begin{equation}\label{eq:making-sgd-parameter-free}
	\eta = c  \cdot 
	\frac{\max_{k\le T} \norm{x_k-x_0}}{\sqrt{\sum_{k\le T}\norm{g_k}^2}},
\end{equation}
then the averaged iterates satisfies an excess loss bound that is at most a factor $\frac{1}{c(1-c^2)}$ larger than the worst-case optimal bound achieved by perfectly tuned SGD.\footnote{This results holds in the non-stochastic case~\citep[Proposition 1]{carmon2022making}, but a qualitatively similar results holds with high probability in the stochastic case as well \citep[Proposition 3]{carmon2022making}.}

The condition~\eqref{eq:making-sgd-parameter-free} is an implicit equation: it allows us to check whether the choice of step size $\eta$ is good only after running $T$ steps of SGD using that $\eta$. Solving this implicit equation therefore requires multiple calls to SGD. We derive the  \DoG step size sequence by making the equation explicit: we choose $\eta_t$ so that equation~\eqref{eq:making-sgd-parameter-free} holds at each step. For $c=1$, this yields the step size formula~\eqref{eq:dog}. Our reason for choosing $c=1$ is that it is the threshold under which a solution to the implicit equation yields an optimal rate of convergence. 
 Therefore, in practice we expect $1$ to be close to the highest stable value of $c$, and thus obtain the best performance; we verify this empirically in \Cref{subsec:DoG-sensitivity}.
\section{Theoretical Analysis}\label{sec:theory}

\subsection{Preliminaries}

\paragraph{Problem setting.}
Our goal is to minimize a loss function $f : \X \rightarrow \R$ where $\X \subseteq \R^{\dim}$ (including the unconstrained setting $\X=\R^\dim$ as an important special case). We perform our analysis under the following standard convexity assumption.

\begin{assumption}[Convexity]\label{ass:convex}
	The function $f$ is convex, its domain $\X$ is closed and convex, and its minimum is attained at some $\xopt\in\xset$, i.e., $\fopt \defeq \inf_{x\in\X}f(x)=f(\xopt)$.
\end{assumption}

In \Cref{app:beyond-convexity} we discuss a possible relaxation of convexity under which our results continue to hold.

To minimize $f$ we assume access to a \emph{stochastic gradient oracle}  $\gradientOracle{}$. When queried at a point $x\in\X$ the oracle returns a stochastic (sub)gradient estimator $\gradientOracle{x}$ satisfying $\Ex{\gradientOracle{x} | x } \in \partial f(x)$. With slight abuse of notation, we write $\grad f(x) \defeq \E[\gradientOracle{x} \mid x ]$. We make the following assumption, where $\norm{\cdot}$ denotes the Euclidean norm.

\begin{assumption}[Pointwise bounded stochastic gradients]\label{ass:bounded}
	There exists some continuous function $\Lfunc : \mathcal{X} \rightarrow \R$ such that
	$\| \gradientOracle{x} \| \le \Lfunc(x)$ almost surely.
\end{assumption}

Assumption~\ref{ass:bounded}, which appeared in prior work~\citep{cutkosky2019artificial,mhammedi2020lipschitz}, is weaker than conventional assumptions in  parameter-free stochastic optimization, which either uniformly bound the stochastic gradients, i.e., $\| \gradientOracle{x} \| \le \LL$ for all $x \in \X$ \citep[see, e.g.,][]{orabona2017training,cutkosky2018black}, or uniformly bound the gradient variance \citep{jun2019parameter}.
However, even least squares problems (with $\gradientOracle{x} = (\inner{a}{x} -b) a$ for random $a\in \R^\dim$ and $b\in\R$) violate both uniform bounds.  In contrast, $\Lfunc$ is finite under the mild assumption that $a,b$ are bounded random variables; see \Cref{app:least-squares-example} for additional discussion.

\paragraph{Algorithm statement.} We study (projected) SGD with dynamic learning rate schedule $\{\eta_t\}$, i.e.,
\[
x_{t+1} = \Proj{\X}{x_{t} - \eta_t g_t }
\]
where $x_0$ is a given initialization, $g_k \defeq \gradientOracle{x_k}$, and $\Proj{\X}{\cdot}$ is the Euclidean projection onto $\X$. To succinctly state and analyze \DoG, we define the following quantities:
\[
r_t \defeq \| x_t - x_0  \| ~\mbox{,}~
\rbar[t] =  \max_{k\le t} r_k \vee \reps %
~\mbox{and}~
\G[t] \defeq \sum_{k=0}^t \| g_t \|^2,
\]
where $a\vee b \defeq \max\{a,b\}$ and $\reps$ is a small user-specified initial movement size parameter. With this notation, we define a family of \DoG-like learning rate schedules.

\begin{definition}\label{property:DoG-style-step-sizes}
	A step size schedule is \emph{\DoG-like} if
	\begin{equation*}
		\eta_t = \frac{\rbar[t]}{\sqrt{\G[t]'}}
	\end{equation*}
	for a positive nondecreasing sequence $\G[t]'$ that depends only on $x_0, g_0, \ldots, g_t$ and satisfies $\G[t]' \ge \G[t]$.
\end{definition}
$\DoG$ corresponds to simply setting $\G[t]' = \G[t]$; in \Cref{sec:iterate-stability-bound} we consider a theoretical (or tamed) $\DoG$-like algorithm for which we guarantee bounded iterates by making $\G[t]'$ larger than $\G[t]$ by polylogarithmic factors. Throughout, we bound the error of the weighted average sequence
\begin{equation}\label{eq:aveged-iterate}
	\xbar \defeq \frac{1}{\sum_{k=0}^{t-1} \rbar[k]} \sum_{k=0}^{t-1} \rbar[k] x_k.
\end{equation}
Finally, to streamline the analysis we define:
\[
\d[t] \defeq \| x_t - \xopt \| ~,~ \dbar[t] \defeq \max_{k \le t} \d[k] ~\mbox{,}~\Lbar[t] \defeq \max_{k\le t}\Lfunc(x_k),
\]
and
\begin{equation*}
	\TimeUniformLog \defeq \log\left( \frac{60 \log(6 t)}{\delta} \right).
\end{equation*}

\paragraph{Logarithm conventions.}
Throughout the paper $\log$ is base $\e$ and $\log_{+}\prn*{\cdot} \defeq 1 + \log(\cdot)$.

\subsection{Optimality gap bounds assuming bounded iterates}\label{sec:error-bound-assuming-bounded-iterates}

In this section, we bound the optimality gap attained by any \DoG-like algorithm. Our bounds depend on the quantities $\rbar[T]$ and $\G[T]$, and are nearly optimal when $\rbar[T]=O(d_0)$ (i.e., the \DoG iterates don't move too far away from $x_0$) and $\G'$ is not much larger than $\G$. In the next section we describe a specific \DoG-like algorithm that is guaranteed to satisfy both requirements.

Convexity and Jensen's inequality imply that $\xbar$ satisfies
\begin{flalign}
f(\xbar[t]) - \fopt 
\le \frac{1}{\sum_{k=0}^{t-1} \rbar[k]}  \sum_{k=0}^{t-1} \rbar[k] \inner{\grad f(x_k)}{x_k - \xopt}.
\label{eq:convexity-to-true-regret}
\end{flalign}
The sum in the RHS decomposes to two components:
\begin{flalign}\label{eq:error-bound-strategy}
	\underbrace{\sum_{k=0}^{t-1} \rbar[k] \inner{ g_k }{ x_k - \xopt }}_{\text{weighted regret}} - \underbrace{\sum_{k=0}^{t-1} \rbar[k] \inner{\Delta_k}{x_k - \xopt}}_{\text{noise}},
\end{flalign} 
where $\Delta_k \defeq  g_k - \grad f(x_k)$. 
We give probability 1 bounds for the weighted regret (Lemma~\ref{lem:bound-empirical-error}) and high probability bounds for the noise term  (Lemma~\ref{lem:bound-generalization-error}). In each case, the key challenge is replacing a-priori bounds on $d_0$ (or the domain size) with the empirically observed $\rbar[T]$. We present and discuss each lemma in turn.

\begin{lemma}[Weighted regret bound]\label{lem:bound-empirical-error}
	If $\xset$ is a closed convex set then any \DoG-like scheme (\Cref{property:DoG-style-step-sizes}) satisfies %
	$\sum_{k=0}^{t-1} \rbar[k] \inner{g_k}{x_k - \xopt } \le  \rbar[t] (2 \dbar[t] + \rbar[t]) \sqrt{\G[t-1]'}$, $\forall t\ge1$.
\end{lemma}

\notarxiv{
The proof of \Cref{lem:bound-empirical-error} appears in \Cref{sec:lem:bound-empirical-error-proof}.
While it is similar to the analysis of adaptive SGD (where $\eta_t = \frac{\rho}{\sqrt{G_t}}$  \citep{gupta2017unified}), there are a couple of key differences. First, the \DoG step sizes can increase, which typically makes adaptive gradient methods difficult to analyze \citep{reddi2018convergence}. We bypass this difficulty by considering regret  weighted by $\rbar[k]$, which factors out the increasing portion of the step size. Second, the standard  adaptive SGD analysis yields a bound proportional to $\dbar[t]^2$ (typically further bounded using the domain diameter) rather than $\rbar[t]\dbar[t]$ as in our bound. This is a crucial difference, since---as we soon argue---$\rbar[t]$ ``cancels'' when dividing through by $\sum_{k<t}\rbar[k]$, while $\dbar[t]$ does not. We obtain the improved result by keeping around the  
 last term in a telescoping sum,
a trick similar to \citet[Lemma 1]{carmon2022making}.
}
\arxiv{
\begin{proof}
	Using $x_{k+1} = \Proj{\X}{x_{k} - \eta_k g_k }$ we obtain the standard inequality
	$\d[k+1]^2 \le \| x_k - \eta_k g_k - \xopt \|^2 = \d[k]^2 - 2 \eta_k \inner{g_k}{x_k - \xopt} + \eta_k^2 \| g_k \|^2$.
	Rearranging this gives:
	\begin{flalign}\label{eq:classic-subgradient-inequality}
		\inner{ g_k }{ x_k - \xopt } \le \frac{\d[k]^2 - \d[k+1]^2}{2 \eta_k} + \frac{\eta_k \| g_k \|^2}{2}.
	\end{flalign}
	Therefore, $\sum_{k=0}^{t-1} \rbar[k] \inner{g_k}{x_k - \xopt }$ is at most
	\[
	\half \underbrace{\sum_{k=0}^{t-1} \frac{\rbar[k]}{\eta_k} (d_k^2 - d_{k+1}^2) }_{(A)}
	+
	\half \underbrace{\sum_{k=0}^{t-1} \rbar[k] \eta_k \norm{g_k}^2}_{(B)}.
	\]
	We bound the terms $(A)$ and $(B)$ in turn, beginning with the former:
	\begin{flalign*}
		(A)&= \sum_{k=0}^{t-1} \sqrt{\G[k]'} (d_{k}^2 - d_{k+1}^2)
		 =  \d[0]^2 \sqrt{\G[0]'}  - \d[t]^2  \sqrt{\G[t-1]'} + \sum_{k=1}^{t-1} \d[k]^2 \left(\sqrt{G_k'} - \sqrt{G_{k-1}'}\right) \\
		&\overle{(i)} \dbar[t]^2 \sqrt{\G[0]'} - \d[t]^2 \sqrt{\G[t-1]'}  + \dbar[t]^2 \sum_{k=1}^{t-1} \left(\sqrt{G_k'} - \sqrt{G_{k-1}'}\right)  
		= \sqrt{G_{t-1}'}\left( \dbar[t]^2 - \d[t]^2 \right) \overle{(ii)} 4\rbar[t] \dbar[t] \sqrt{G_{t-1}'}.
	\end{flalign*}
	Inequality $(i)$ uses $\d[k] \le \dbar[t]$ and that $\G[k]'$ is nondecreasing as per \Cref{property:DoG-style-step-sizes}. Inequality $(ii)$ holds since, for $s \in \argmax_{k \le t} \d[k]$, we have
	$\dbar[t]^2 - \d[t]^2 = \d[s]^2 - \d[t]^2 = (\d[s] - \d[t]) (\d[s] + \d[t]) \le \| x_s - x_t \| (d_s + d_t)
	\le (\rbar[s] + \rbar[t]) (\d[s] + \d[t]) \le 4 \rbar[t] \dbar[t]$.
	Bounding the second term $(B)$, we have:
	\begin{flalign*}
		(B) &= \sum_{k=0}^{t-1}\frac{\rbar[k]^2 \| g_k \|^2}{\sqrt{G_k'}} \le \sum_{k=0}^{t-1}\frac{\rbar[k]^2 \| g_k \|^2}{\sqrt{G_k}}
		 \le \rbar[t]^2 \sum_{k=0}^{t-1}\frac{\| g_k \|^2}{\sqrt{G_k}} \le 2\rbar[t]^2 \sqrt{G_{t-1}},
	\end{flalign*}
	where the final inequality uses the standard Lemma~\ref{lem:adagrad-algebra} with $a_k = G_k=\sum_{i\le k}\norm{g_i}^2$.
\end{proof} 
While the proof of \Cref{lem:bound-empirical-error} is similar to the analysis of adaptive SGD where $\eta_t = \frac{\rho}{\sqrt{G_t}}$  \citep{gupta2017unified}, there are a couple of key differences. First, the \DoG step sizes can increase, which typically makes adaptive gradient methods difficult to analyze \citep{reddi2018convergence}. We bypass this difficulty by considering regret  weighted by $\rbar[k]$, which factors out the increasing portion of the step size. Second, the standard  adaptive SGD analysis yields a bound proportional to $\dbar[t]^2$ (typically further bounded using the domain diameter) rather than $\rbar[t]\dbar[t]$ as in our bound. This is a crucial difference, since---as we soon argue---$\rbar[t]$ ``cancels'' when dividing through by $\sum_{k<t}\rbar[k]$, while $\dbar[t]$ does not. We obtain the improved result by keeping around the
term $-\d[t]^2 \sqrt{\G[t-1]'}$ in the bound for $(A)$ above; 
a trick similar to \citet[Lemma 1]{carmon2022making}.
}

Next, we handle the noise term in~\eqref{eq:error-bound-strategy}, recalling the notation $\Delta_t \defeq g_t - \grad f(x_t)$ and  $ \TimeUniformLog[t, \delta]\defeq \log \frac{60\log(6t)}{\delta}$.
\begin{lemma}[Noise bound]\label{lem:bound-generalization-error}
	Under \Cref{ass:bounded},	for all $\delta \in (0,1)$, $\numSteps \in \N$ and $\LL > 0$ we have
	\notarxiv{
	\[
	\P\prn*{
		\exists t \le \numSteps: \abs[\Bigg]{\sum_{k=0}^{t-1} \rbar[k] \inner{\Delta_k}{x_k - \xopt}}
		\ge
		b_t \hspace{-0.05cm}  } \hspace{-0.1cm} \le \delta + \P\prn*{ \Lbar[\numSteps] > \LL }
	\]
	where $b_t = 8 \rbar[t-1] \dbar[t-1] \sqrt{ \TimeUniformLog[t, \delta] \G[t-1] +  \TimeUniformLog[t, \delta]^2 \LL^2}$.
}
	\arxiv{
		\begin{equation*}
				\P\prn*{
				\exists t \le \numSteps: \abs[\Bigg]{\sum_{k=0}^{t-1} \rbar[k] \inner{\Delta_k}{x_k - \xopt}}
				\ge
			8 \rbar[t-1] \dbar[t-1] \sqrt{ \TimeUniformLog[t, \delta] \G[t-1] +  \TimeUniformLog[t, \delta]^2 \LL^2  }} \le \delta + \P\prn*{ \Lbar[\numSteps] > \LL }.
		\end{equation*}
	}
\end{lemma}

The proof of \Cref{lem:bound-generalization-error} appears in Appendix~\ref{app:cor:product-mg-concentration} and is based on a new concentration bound, \Cref{cor:product-mg-concentration}, which allows us to 
bound the noise term despite having no deterministic bound on the magnitude of the martingale difference sequence $\rbar[k] \inner{\Delta_k}{x_k - \xopt}$. 
 The proof of \Cref{cor:product-mg-concentration} involves combining
time-uniform Bernstein bounds \cite{howard2021time} and a general bound on the cumulative sums of sequence products (\Cref{lem:many-sequences-nondecreasing}), which may be of independent interest.

Combining the above results, we obtain the following.
\begin{proposition}\label{prop:error-bound}
For all $\delta \in (0,1)$ and $\LL > 0$,
if \Cref{ass:convex}, \Cref{ass:bounded}, and \Cref{property:DoG-style-step-sizes} hold then with probability at least $1 - \delta - \P\prn*{ \Lbar[\numSteps] > \LL }$, for all $t \le T$ the optimality gap $f(\xbar[t]) - \fopt$ is 
	\begin{flalign*}
		 O\prn*{\frac{ (\d[0] + \rbar[t]) \sqrt{\G[t-1]' + \G[t-1] \TimeUniformLog[t,\delta] + \LL^2 \TimeUniformLog[t,\delta]^2}}{\sum_{i < t} \rbar[i] / \rbar[t]}}.
	\end{flalign*}
\end{proposition}

\begin{proof}
Follows from \Cref{eq:convexity-to-true-regret,eq:error-bound-strategy}, \Cref{lem:bound-empirical-error}, \Cref{lem:bound-generalization-error} and the fact that $\dbar[t] \le \d[0] + \rbar[t]$.
\end{proof}

The following algebraic fact shows that there is always an iteration $\tau\le T$ where the denominator $\sum_{i<t} \frac{\rbar[i]}{\rbar[t]} \ge\Omega(T/\log\frac{\rbar[T]}{\reps})$; see \Cref{app:bound-a-ratios} for proof.

\begin{lemma}\label{lem:bound-a-ratios}
	Let $s_0,s_1,\ldots,s_{\numSteps}$ be a positive nondecreasing sequence. Then
\begin{equation*}
		\max_{t\le \numSteps} \sum_{i < t} \frac{s_i}{s_t} \ge \frac{1}{e}\prn*{ \frac{\numSteps}{\log_{+}(s_{\numSteps} / s_0)} -1}.
\end{equation*}
\end{lemma}

Combining \Cref{prop:error-bound} and \Cref{lem:bound-a-ratios} yields the following (see short proof in \Cref{app:coro-error-bound-proof}).

\begin{corollary}\label{coro:error-bound}
Under \Cref{ass:convex,ass:bounded}, for any $D\ge\d[0]$, let $\LL_{D}\defeq \max_{x\in\xset: \norm{x-x_0} \le D} \Lfunc(x)$. Then, for all $\delta\in(0,1)$ and for 
$\tau \in \argmax_{t\le \numSteps} \sum_{i < \tau} \frac{\rbar[i]}{\rbar[t]}$, 
with probability at least  $1-\delta-\P(\rbar[T] > D)$, 
the \ref{eq:dog} iterates satisfy the optimality gap bound 
	\begin{equation*}
		f(\xbar[\tau])-\fopt = 
		O\prn*{ \frac{D\sqrt{ \G[\tau-1] \TimeUniformLog[\tau,\delta] + \LL_D^2 \TimeUniformLog[\tau,\delta]^2}}{T}
		\log_{+}\prn*{ \frac{D}{\reps} }}
		= 
		O\prn*{ \frac{D\LL_D}{\sqrt{T}}
		\TimeUniformLog[\tau,\delta]\log_{+}\prn*{ \frac{D}{\reps} }
		}
		.
	\end{equation*}
\end{corollary}
\noindent
\Cref{coro:error-bound} is immediately useful when $\X$ is bounded but its exact diameter is unknown, for example when $\X$ is a polytope as is common in two-stage stochastic programming \cite{nemirovski2009robust}.

\paragraph{Simplifying the bound for typical \DoG trajectories.} 
Suppose that the \DoG iterates satisfy $\rbar[T] \le 3d_0$, which implies that $\Lbar[T] \le \LLstar \defeq \LL_{3\d[0]}$ and therefore (for \DoG) $\G[t]' = \G[t] \le \LLstar^2 T$. Substituting into \Cref{coro:error-bound} yields an optimality gap bound of $O\prn*{ \frac{d_0 \LLstar}{\sqrt{T}} \theta_{T,\delta} \log \frac{\rbar[T]}{\reps}}$, which is minimax optimal up to a term double-logarithmic in $T$ and logarithmic in $\frac{1}{\reps}$  \citep{agarwal2012information}.

Furthermore, in realistic \DoG trajectories, even the multiplicative term $\log \frac{\rbar[T]}{\reps}$ is likely too pessimistic. This is because $\rbar[t]$ typically increases rapidly for $t_0 < 1000$ steps and then plateaus (see \Cref{fig:rbar} in the appendix). Consequently, $\rbar[i] / \rbar[t] \ge 1/10$ for most of the optimization trajectory, and  $\sum_{i<t} \frac{\rbar[i]}{\rbar[t]} \ge t/10 - t_0$. Substituting back into \Cref{prop:distance-bound}, we get that $\xbar[T]$ is $O\prn*{ \frac{d_0 \LLstar}{\sqrt{T - t_0}} \theta_{T,\delta} }$ suboptimal.

\paragraph{\DoG can run wild.}
While \DoG is empirically stable, there exist (non-stochastic) examples where $\rbar[t]$ grows much larger than $\d[0]$: in Appendix~\ref{app:unstable-distance-DoG} we describe a variant of Nemirovski's function~\cite{nemirovski1983problem,nemirovski1994parallel}
for which $\rbar[t] = \reps\sqrt{t}$ and therefore $\rbar[t]/\d[0]$ diverges as $t$ grows. 
Next, we show that by slightly decreasing the \DoG step sizes we can guarantee that $\rbar[T] / \d[0] \le 3$ with high probability.

\subsection{Iterate stability bound}\label{sec:iterate-stability-bound}

This section introduces a new \emph{\DoG-like} step size scheme whose iterates are guaranteed to remain bounded with high probability.
We call this scheme $\TDoG$, where the T stands for ``theoretical'' or ``tamed.'' The step sizes are given by $\eta_t={\rbar[t]}/{\sqrt{\G[t]'}}$, where
\begin{equation}\tag{\TDoG}\label{eq:T-DoG}
	\G[t]' = 8^4 \TimeUniformLog[\numSteps,\delta]^2 \log_{+}^2\prn*{1+\frac{t \Lbar[t]^2}{\Lbar[0]^2}}(\G[t-1] + 16 \Lbar[t]^2),
\end{equation}
using $\G[-1]\defeq0$, and recalling that $\Lbar[t] \defeq \max_{i \le t} \Lfunc(x)$ for a function $\Lfunc$ satisfying \Cref{ass:bounded}. 
The \TDoG formula depends weakly on the iteration budget $T$ and the failure probability $\delta$ via $\theta_{t,\delta}\defeq \log\prn*{\frac{\log(6t)}{\delta}}$; as we show below, in the non-stochastic setting we may simply replace $\TimeUniformLog$ with 1. Moreover, the term $16\Lbar[t]$ typically grows slowly with $t$, becoming negligible compared to $\G[t-1]$. Notably, the \ref{eq:T-DoG} step size requires no global upper bound on stochastic gradient norms.

We are ready to state \TDoG's key property: guaranteed iterate stability.

\begin{proposition}\label{prop:distance-bound}
Suppose that \Cref{ass:convex,ass:bounded} hold and $\reps \le 3 \d[0]$. 
For any $\delta \in (0,1)$, and $\numSteps \in \N$, the iterations of \ref{eq:T-DoG} satisfy $\Pr*(\rbar[\numSteps]  > 3 \d[0] ) \le \delta$.
\end{proposition}

We defer the full proof to  \Cref{sec:proof-prop:distance-bound} and proceed to highlight the key argument by proving the result in the noiseless case.

\begin{proof}[Proof of \Cref{prop:distance-bound} in the noiseless case] 
	In the noiseless case we have $g_k=\grad f(x_k)$ and therefore
	$\inner{g_k}{ x_k- \xopt}\ge f(x_k)-\fopt \ge 0$. 	Substituting into  \eqref{eq:classic-subgradient-inequality} and rearranging gives 
	$d_{k+1}^2 - d_k^2 \le  \eta_k^2 \| g_k \|^2$. Assuming by induction that $\rbar[t] \le 3\d[0]$ and telescoping yields
	$d_{k+1}^2 - d_k^2 \le  \eta_k^2 \| g_k \|^2$. Assuming by induction that $\rbar[t] \le 3\d[0]$ and telescoping yields
	\begin{flalign*}
		d_{t+1}^2 - d_0^2 &\le \rbar[t]^2 \sum_{k=0}^t \frac{\| g_k \|^2}{\G[k]'}  \overle{(i)} \frac{ \rbar[t]^2}{8^4} \sum_{k=0}^t \frac{G_k - G_{k-1}}{(G_k+\Lbar[k]^2) \log_+^2 \frac{G_k+\Lbar[k]^2}{\Lbar[0]^2}}
		\overle{(ii)} \frac{ \rbar[t]^2}{8^4} 	\overle{(iii)} \frac{ 9 \d[0]^2}{8^4} \implies \d[t+1] 	\le 2 \d[0],
	\end{flalign*}
where $(i)$ uses that $\norm{g_k}^2 = \G[k] - \G[k-1]$ (with the shorthand $\G[-1] \defeq  0$) and 
\begin{equation*}
	\G[k]' 
	\ge 8^4 (\G[k-1] + \norm{g_k}^2 + \Lbar[k]^2) \log_+^2\prn*{\frac{\sum_{i \le t} \Lbar[t]^2}{\Lbar[0]^2} } \ge 8^4  (\G[k] + \Lbar[k]^2) \log_+^2 \frac{\G[k] + \Lbar[k]^2}{\Lbar[0]^2}
\end{equation*}
by \Cref{ass:bounded} which implies $\norm{g_k} \le \Lbar[k]$ for all $k$,
	 $(ii)$ uses \Cref{lem:bound-a-k-infinite-sum} with $a_k = \G[k] + \Lbar[k]^2$, and $(iii)$ uses the inductive assumption $\rbar[t] \le 3 \d[0]$. 
	 Therefore, $\r[t+1] \le \d[t+1] + \d[0] \le 3\d[0]$ by the triangle inequality, completing the induction step.
	 Note that this proof ignored the $\TimeUniformLog$ term in \eqref{eq:T-DoG}, demonstrating it is not necessary in the noiseless case.
\end{proof}

Given \Cref{ass:bounded} we define
\begin{equation}\label{eq:LLstar-def}
	\LLstar = \max_{x \in \X:  \norm{x-x_0}\le 3\norm{x_0 - \xopt}} \Lfunc(x).
\end{equation}
With all the ingredients in hand, we state the main guarantee for \TDoG.

\begin{theorem}\label{thm:final-bound}
	Suppose that \Cref{ass:convex,ass:bounded} hold.
	For any $\delta \in (0,\frac{1}{2})$, $\numSteps \in \N$, consider $\numSteps$ iterations of \ref{eq:T-DoG} with $\reps \le 3 \d[0]$.
	Then
	for $\tau \in \argmax_{t\le \numSteps} \sum_{i < \tau} \rbar[i] / \rbar[t]$ we have, with probability at least $1- 2\delta$, that
	\begin{flalign*}
		f(\xbar[\tau]) - \fopt = 
		 O\prn*{c_{\delta,\reps,\numSteps}\frac{ \d[0] \sqrt{\G[\tau-1] + \LLstar^2}}{\numSteps}} = O\prn*{ c_{\delta,\reps,\numSteps}  \frac{\d[0] \LLstar}{\sqrt{\numSteps}}},
	\end{flalign*}
	where $c_{\delta,\reps,\numSteps} = \log_{+}\prn*{ T \frac{ \d[0] \LLstar}{f(x_0)-\fopt}} \log_{+}\prn*{\frac{\d[0]}{\reps}} \log\left( \frac{\log_{+}( \numSteps)}{\delta} \right)$.
\end{theorem}

\begin{proof}
	The theorem follows from \Cref{coro:error-bound}, \Cref{prop:distance-bound} and the definition of \ref{eq:T-DoG}, where we note that \Cref{ass:bounded} and convexity of $f$ imply $\Lbar[0] \ge \norm{\grad f(x_0)} \ge (f(x_0)-f(x\opt)) / \d[0]$, while $\rbar[T] \le 3\d[0]$ gives $\Lbar[T] \le \LLstar$. Therefore, $\log_+\prn*{1 + \frac{T \Lbar[T]^2}{\Lbar[0]^2} } = O\prn*{ \log_{+}\prn*{T \frac{ \d[0] \LLstar}{f(x_0)-\fopt}}}$.
\end{proof}

\Cref{thm:final-bound} yields the optimal convergence bound \cite{agarwal2012information} up to logarithmic factors. To the best of our knowledge this is the first parameter-free stochastic optimization method that does not require the stochastic gradients to be uniformly bounded across the domain $\X$ and instead produces a bound that depends on the `local' gradient bound $\LLstar$. Crucially, the \ref{eq:T-DoG} step size formula does not require advance knowledge of $\LLstar$.\footnote{There is prior work that develop methods with steps that do not require a global Lipschitz bound~\citep{cutkosky2019artificial,mhammedi2020lipschitz}, but these methods
do not guarantee that iterates remain in a ball of radius $O(\d[0])$ around the initial point. Consequently, the rates of convergence of these methods cannot be expressed in terms of a quantity like $\LLstar$.}

\paragraph*{Extension to unweighted iterate averaging.}
While the weighted iterate average~\eqref{eq:aveged-iterate} is convenient to our analysis, bounds similar to \Cref{prop:error-bound}, \Cref{coro:error-bound} and \Cref{thm:final-bound} hold also for the standard unweighted iterate average $\hx_T = \frac{1}{T}\sum_{t=0}^{T-1}x_t$. For $\hx_T$ it is also straightforward to show a $1/T$ error bound for \DoG in the smooth noiseless case. See \Cref{app:unweighted} for details.

\section{Experiments}\label{sec:experiments}

To test \DoG{} in practical scenarios, we perform extensive experiments over a diverse set of tasks and model architectures in both the vision and language domains. We construct a testbed that consists of over 20 tasks and 7 model architecture, covering natural language understanding and computer vision (\Cref{subsec:settings}). In this testbed we compare \DoG to SGD and Adam (\Cref{subsec:results}), showing that \DoG performs on par with tuned SGD, but not as well as tuned Adam. Nevertheless, a per-layer version of \DoG (defined below) closes much of this gap with Adam without requiring tuning. We also use our testbed to analyze the sensitivity of \DoG{} to its fixed parameters (\Cref{subsec:DoG-sensitivity}), and demonstrate its effectiveness in convex logistic regression settings (\Cref{subsec:convex-opt}). Finally, we apply \DoG and \LDoG to fine-tuning a CLIP model on ImageNet (\Cref{subsec:imagenet}) and training a CIFAR10 model from scratch (\Cref{subsec:from-scratch}), and provide preliminary comparison to previously-proposed tuning free methods (\Cref{subsec:other-optimizers}). A PyTorch implementation of \DoG 
is available at \url{https://github.com/formll/dog}.

\paragraph{Layer-wise \DoG.}
Neural models in general and transformer-based models in particular often benefit from using a per-parameter or per-layer step sizes \citep{kingma2015adam,you2019large}. %
With this in mind, we consider a per-layer version of \DoG, which we call \LDoG, where we apply the \eqref{eq:dog} formula separately for every layer. Namely, if we consider $x_t^l$ to be the weights in layer\footnote{More precisely, our implementation treats each element in the PyTorch \texttt{.parameters()} list as a separate layer.}  $l$ at step $t$, then we set the learning rate for that layer to be $\eta_t^l = \frac{\max_{i\le t}\norm{x_i^l - x_0^l}}{\sqrt{\sum_{i\le t}\norm{g_i^l}^2+\epsilon}}$, where $\epsilon=10^{-8}$ is added to the denominator for numerical stability.
While we do not provide theoretical guarantees for \LDoG, we show below that it performs well in practice.

\subsection{Fine-tuning testbed}\label{subsec:settings}
Our main experiments focus on fine-tuning pre-trained models, which allows us to experiment with advanced models while also thoroughly tuning the learning rate for the baseline optimizers, using an academic computational budget.

\paragraph{Common hyperparameters.}
For each baseline algorithm, we use best-practice learning rate schedule (cosine annealing for all experiments, with a warmup stage for language experiments) and sweep over the peak learning rate for each model/task pair. We give each pair a fixed step budget designed to suffice for convergence, performing evaluation throughout the training. In all cases, we use polynomial decay averaging\footnote{We apply the weight averaging with a fixed parameter ($\gamma=8$, following~\cite{levy2020large}); we did not try any other parameter in our experiments.} as proposed by \citet{shamir2013stochastic}, and select the best checkpoint (either averaged or not) based on evaluation performance.
We repeat relevant learning setups with 5 different seeds, and report the mean performance across the seeds.
For simplicity, we do not use weight decay throughout.
The complete set of hyper-parameters appears in \Cref{app:exp-details}.

\paragraph{Natural language understanding (NLU).}
To test \DoG{}'s efficacy in modern NLU, we use it to fine-tune transformer language models \citep{vaswani2017attention} on the well-studied GLUE benchmark \citep{wang2018glue} which measures models' performance on diverse text classification tasks (listed in \Cref{app-subset:datasets}).

Additionally, we fine-tune models on SQuAD 1.1, a question answering dataset \citep{rajpurkar2016squad}.
We fine-tune a RoBERTa-base \citep{liu2019roberta} checkpoint
and T5-base \citep{Raffel2020Exploring}.\footnote{Throughout the paper we often use the shorthand names RoBERTa-b and T5-b, respectively.} %
For each task, we use the official evaluation metrics defined in \citet{wang2018glue} and \citet{rajpurkar2016squad} as well as their original proposed splits, and report the results over the evaluation set.

\paragraph{Computer vision.}
We also fine-tune 5 models architectures on 12 different computer vision tasks from the VTAB benchmark \citep{zhai2019large} (see \Cref{app-subset:datasets});
of the other 7 tasks in VTAB, 5 are trivial (accuracy greater than 99\%)
and 2 have small validation splits leading to unreliable model selection.
We follow the training, validation and test splits defined in VTAB, and report performance on the test split (using the validation split for model selection). 
We fine-tune 5 models: VGG11 \citep{simonyan2014very}, ResNet50 \citep{he2015deep}, Densenet121 \citep{huang2016densely}, ViT-B/32 \citep{dosovitskiy2021image}, and ConvNeXt-T \citep{liu2022convnet}, where the ViT model is pre-trained on ImageNet 21K and the others are trained on ImageNet 1K \citep{deng2009imagenet}.

\paragraph{Normalized performance metric.}
 Since the performance metrics in our testbed vary substantially across tasks and models, they are challenging to compare in aggregate. To address this, we consider the following notion of \emph{relative error difference} (RED), that provides a normalized performance difference measure. In particular, given a task and a model architecture, let $\mathrm{err}_{x}$ be the error\footnote{We consider the error to be 1 minus the respective performance metric, as detailed in \Cref{table:tasks-configurations}.} 
  of the  model when trained with optimizer $x$ (Adam or SGD with a certain learning rate, or \LDoG) and let $\mathrm{err}_{\DoG}$ be the error when trained with \DoG. Then
\begin{align*}
	\text{RED}(\mathrm{err}_{x}, \mathrm{err}_{\DoG}) \defeq \frac{\mathrm{err}_{\DoG} - \mathrm{err}_{x}}{\mathrm{err}_{\DoG}}.
\end{align*}
A positive RED value indicates that optimizer $x$ is better than \DoG, and a negative value indicates the opposite. When the absolute value of RED is beneath a few percentage points, the compared methods are nearly equivalent. For example, a 5\% RED is equivalent to the difference between accuracy 95\% and 95.25\%.

\paragraph{Setting $\reps$.} 
Our theoretical analysis suggests that the particular choice of $\reps$ does not matter as long as it is sufficiently small relative to the distance between the weight initialization $x_0$ and the optimum. Consequently, for vision experiments we set $\reps = \alpha\cdot (1+\norm{x_0})$ for $\alpha=10^{-4}$, assuming that the distance to the optimum is more than 0.01\% of the initialization norm. For language experiments, this assumption turned out to be wrong (causing \DoG to diverge in some cases), and we decreased $\alpha$ to $10^{-6}$ for \DoG and to $10^{-8}$ for $\LDoG$, where the additive $10^{-6}$ term was too large in some layers.  
We believe that $10^{-6}$ and $10^{-8}$ should be good defaults for \DoG and \LDoG, respectively, though networks with batch normalization or different initialization schemes could require a larger value; see \Cref{subsec:DoG-sensitivity} for additional discussion.

\subsection{Comparison of fine-tuning performance}
\label{subsec:results}

\begin{figure*}[t]
	\begin{center}
	\centerline{
	\arxiv{\includegraphics[width=1.1\textwidth]{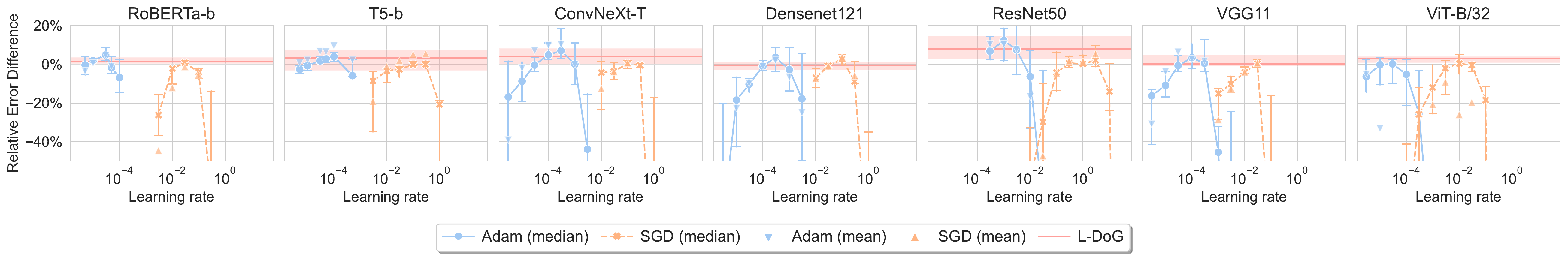}}
	\notarxiv{\includegraphics[width=\textwidth]{figures/mean_results_main_full.pdf}}
	}
	\notarxiv{\setlength{\belowcaptionskip}{-20pt}}
	\caption{Relative error difference statistics (median, mean, and error bars showing IQR) across tasks for each model,  as a function of peak learning rate. The red horizontal line and shaded region indicate the median and IQR RED for \LDoG{}, respectively.
	}
	\label{fig:mean-results}
	\end{center}
	\arxiv{\vspace{-12pt}} %
	\notarxiv{\vspace{-1cm}}
\end{figure*}

\Cref{fig:mean-results} depicts the median, IQR (inter-quantile range) and mean RED of each model,\footnote{When aggregating results over tasks, we always report the RED statistics across tasks, where for each task we average the RED values over seeds. See \Cref{app:experiments-hparams} for details.} when trained with SGD and Adam with different peak learning rates. 
The figure shows that, when comparing across models, there is no good default learning rate for neither SGD nor Adam. Moreover, even for a single model only very specific SGD learning rate performs well, while most are considerably inferior to using \DoG. Even when tuned to the best \emph{fixed} learning-rate value per model (which we refer to as \emph{model tuned LR}), some tasks may still fail (compared to \DoG{}) as indicated by the large IQR and the gap between the mean (triangles) and the median RED (circles) in models such as ViT-B/32 ad Densenet121. 
While Adam also requires tuning, it is somewhat less sensitive than SGD to the choice of peak learning rate.  
For a full breakdown of performance per task, see 
\Cref{fig:per-task-results} 
and \Cref{table:language-results,table:vision-results} in \Cref{app:experiments-breakdown}.

\DoG performs similarly to well-tuned SGD in 79 out of the 80 model/task combinations in our testbed. The one exception is tuning T5-b on CoLA, where \DoG{} behaves erratically while SGD succeeds only with a few learning rates.
In contrast, both Adam and \LDoG achieved reasonable performance consistently. \DoG's poor performance on CoLA results in high RED measures for this case, which draw the mean RED (triangles) above the median one in \Cref{fig:mean-results} for T5-b. We further analyze this exception in \Cref{app:experiments-cola} and show that choosing significantly smaller $\reps$ for \DoG alleviates the problem.

\Cref{fig:hpt-results} (top) compares \DoG to SGD with model tuned LR as defined above, as well as \emph{instance tuned LR}, where for each model/task pair we select the best learning rate, at a computational expense 5--7 times larger than running \DoG. The performance of \DoG remains close to that of SGD with instance-tuned LR, with the largest median RED observed for ResNet50 and ViT-B/32. 

\Cref{fig:hpt-results} (bottom) compares \DoG to model-tuned and instance-tuned Adam, as well as to \LDoG. In a few cases (namely ResNet50 and ConvNeXt-T) the gaps between \DoG and Adam are significant, and favor Adam. We hypothesize this is due to Adam's per-parameter step-sizes and momentum mechanisms, which \DoG does not exploit.  \LDoG, which has per-layer steps, has positive median RED for all models, and narrows the gap between \DoG and Adam, particularly for ResNet50.

The instance-tuned baselines consume significantly more compute than \DoG and \LDoG. In \Cref{app:expeirments-equalized} we equalize the compute budget by reducing the number of steps for SGD and Adam. This makes \DoG outperform instance-tune SGD in most cases, and brings \LDoG substantially closer to Adam. 

\begin{figure}[t]
	\begin{center}
	\centerline{
	\notarxiv{\includegraphics[width=\linewidth]{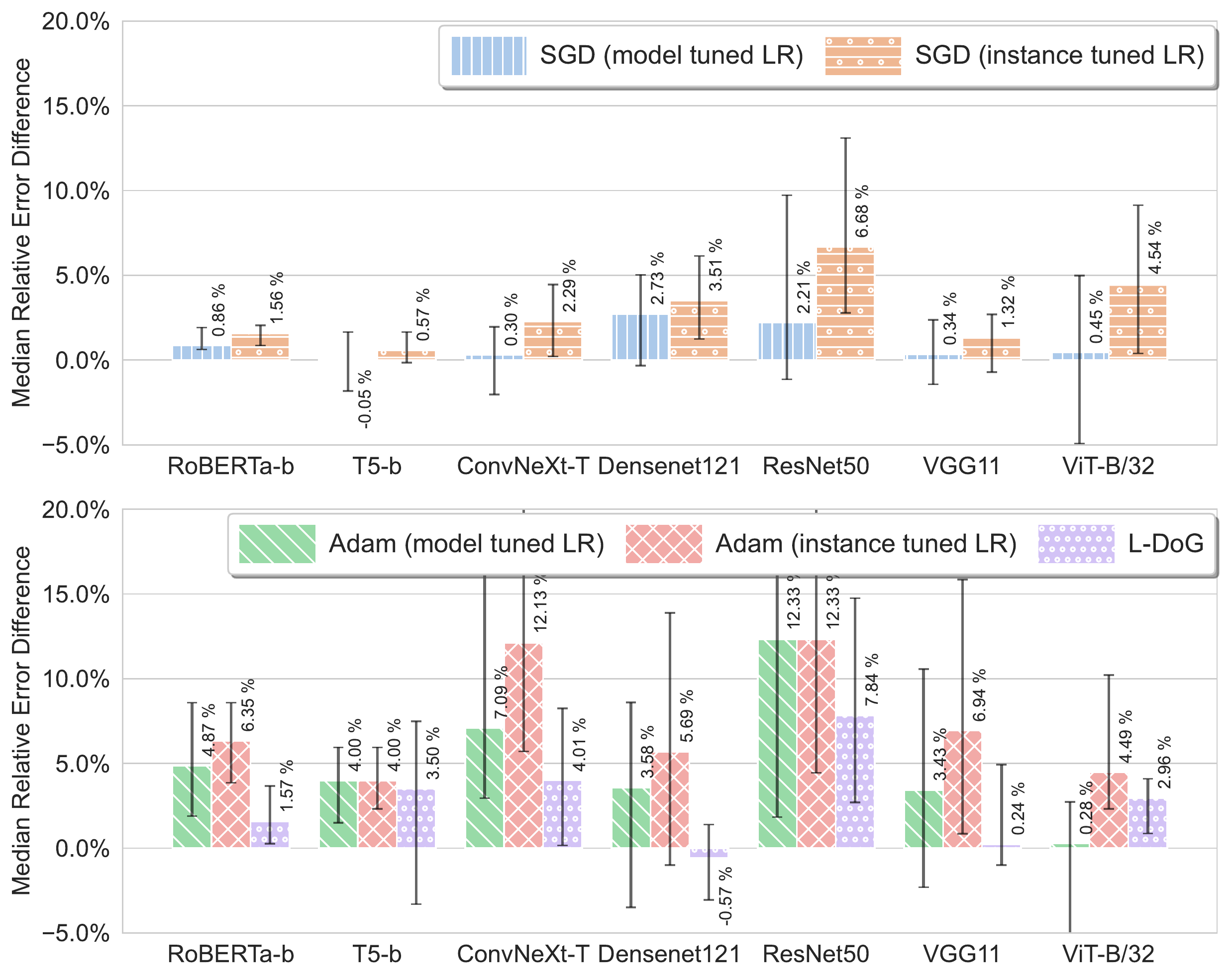}}
	\arxiv{\includegraphics[width=0.6\linewidth]{figures/hpt_results.pdf}}
	}
	\notarxiv{\setlength{\belowcaptionskip}{-20pt}}
	\caption{RED median (bar chart) and IQR (error bars) of each model on the set of applicable tasks. \textbf{Top}: Comparison with SGD when the LR is optimally tuned per model (\emph{model tuned LR}) or per task (\emph{instance tuned LR}). 
	\DoG{} is competitive with model-tuned SGD and often performs nearly as well as instance-tuned SGD. 
	\textbf{Bottom}: Comparison of \DoG{} with adaptive optimizers. 
	\LDoG{} closes most of the gap to Adam.
}
	\label{fig:hpt-results}
	\end{center}
	\notarxiv{\vspace{-1cm}}
\end{figure}

\begin{figure}[t]
	\begin{center}
	\centerline{
	\notarxiv{\includegraphics[width=0.9\linewidth,trim={0cm 0.2cm 0cm 0cm},clip]{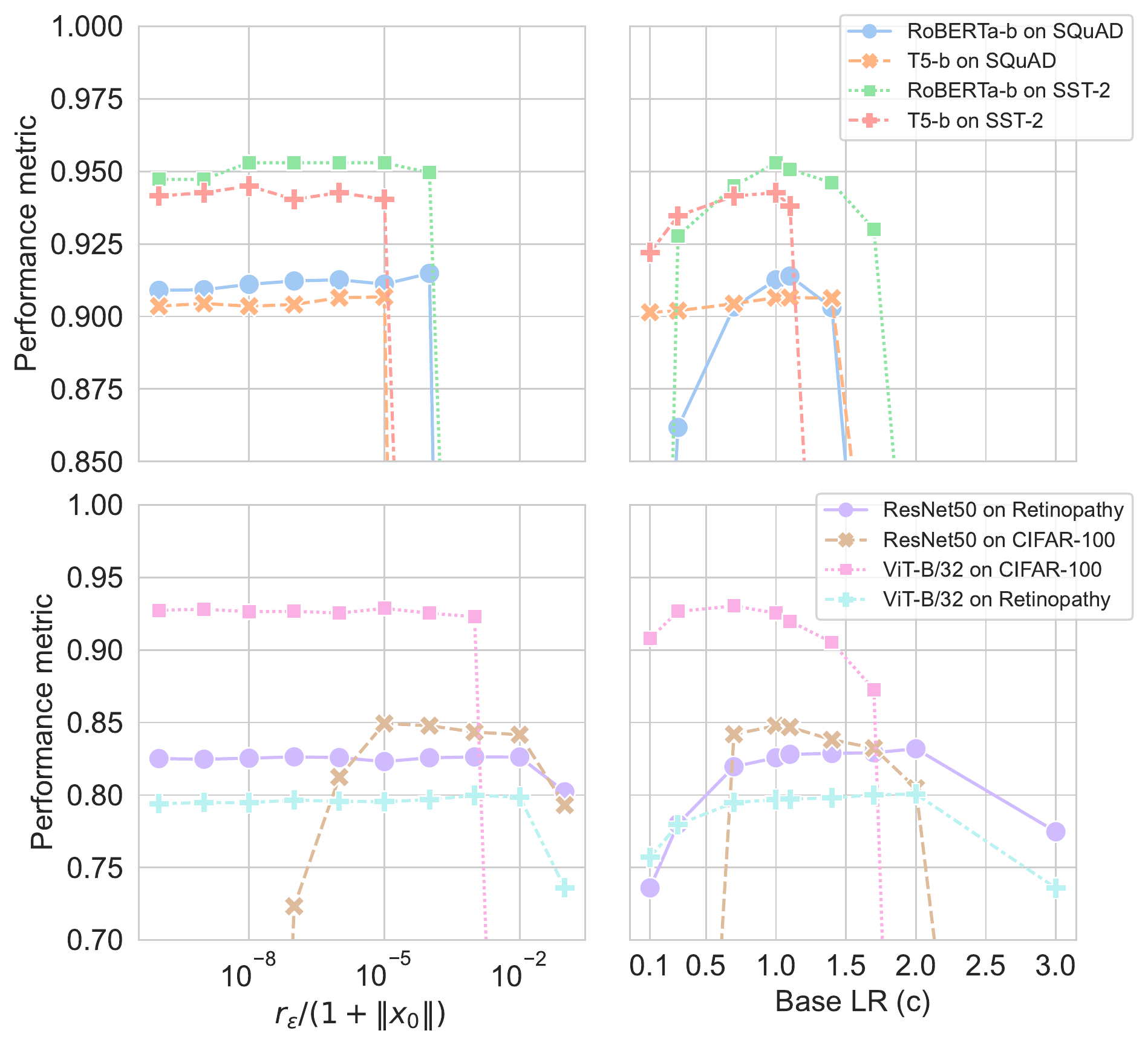}}
	\arxiv{\includegraphics[width=0.5\linewidth,trim={0cm 0.2cm 0cm 0cm},clip]{figures/ablations.pdf}}

	}
	\caption{Performance metrics of models trained with \DoG{} as a function of $\eta_0$ (left) or the base learning rate (right). 
	}
	\label{fig:ablations}
	\end{center}
	\notarxiv{\vspace{-0.75cm}}
\end{figure}

\subsection{Sensitivity of \DoG's fixed parameters}\label{subsec:DoG-sensitivity}

\paragraph{Initial movement size $\reps$.}
Our theory suggests that all sufficiently small choices of $\reps$ should perform similarly, but choosing $\reps$ too large (compared to the initial distance to the optimum) can hurt the performance of the algorithm.
In \Cref{fig:ablations} (left) we plot the test performance as a function of $\reps$ for 8 model/task combinations. For 7 out of the 8, \DoG{} is highly robust to the value of $\reps$ as long as it small enough, as predicted. However, ResNet50 on CIFAR-100 (bottom left) is an exception, where smaller values of $\reps$ result in an accuracy drop. We hypothesize this is due to scale invariance introduced by batch normalization (BN), and provide  supporting evidence for that in \Cref{app:experiments-reps} (\Cref{fig:no-bn}), where we show that \DoG is insensitive to $\reps$ when we turn off BN. In the appendix we also provide a complementary diagnostic for $\reps$ sensitivity by plotting $\eta_t$ vs.\ $\eta_0$ for different values of $t$ (see \Cref{fig:eta-sensitivity}).

\paragraph{Base learning rate.}
For this experiment only, we consider variants of \DoG with different values of base learning, i.e., step sizes of the form $\eta_t = c \cdot \frac{\max_{i\le t}\norm{x_i - x_0}}{\sqrt{\sum_{i\le t}\norm{g_i}^2}}$ with different values of $c$.
We expect optimal performance when $c$ is close to 1. More specifically, we expect the algorithm to be unstable when $c>1$ and to be slower to converge (and less likely to generalize well) when $c<1$. As can be observed in \Cref{fig:ablations} (right), values around $c=1$ perform well for all models. For smaller values, there is indeed inferior performance in some models (mainly ResNet50 and RoBERTa-b)---indicating \TDoG would not work well in practice---while larger values result in divergence (in 6 out of 8 cases).
Hence, the useful range for $c$ is very narrow (about [0.5, 1.5]) and tuning it is not likely to produce significant improvements. This is in contrast to Adam and SGD which generally require searching over a space spanning a few orders of magnitude  to properly train a model.

\arxiv{
\begin{figure}[t]
	\begin{center}
		\centerline{
				\includegraphics[width=0.6\linewidth]{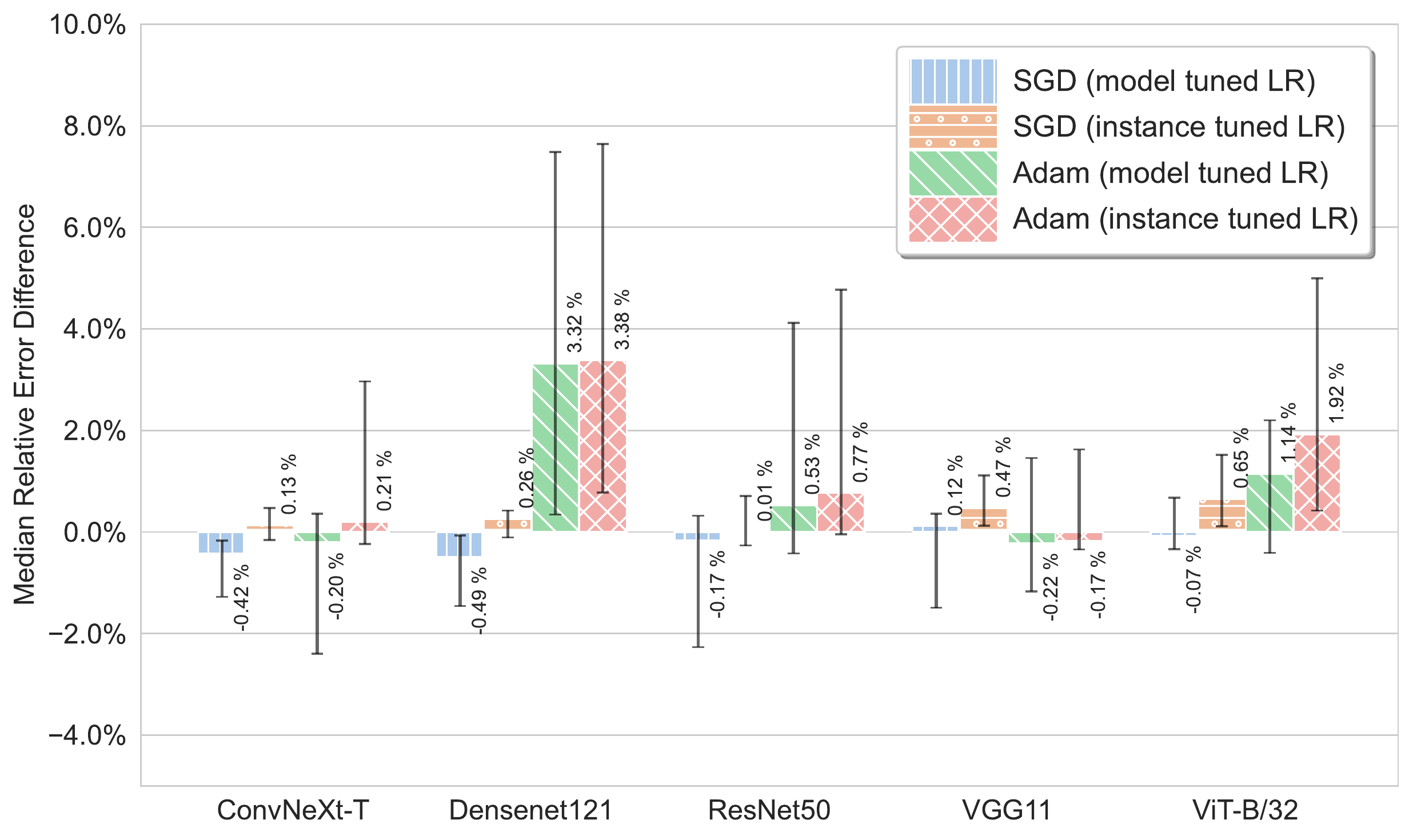}

		}
		\caption{RED median and IQR (as in \Cref{fig:hpt-results}) in tn the \emph{convex optimization} setting (\Cref{subsec:convex-opt}).
		}
		\label{fig:hpt-results-convex}
	\end{center}
\end{figure}

\newcommand{\MeanResultsConvexCaption}{Per-learning rate RED statistics (as in \Cref{fig:mean-results}) in the \emph{convex optimization} setting (\Cref{subsec:convex-opt}).}

	\begin{figure}[t]
		\begin{center}
			\centerline{
				\includegraphics[width=\textwidth,height=\textheight,keepaspectratio]{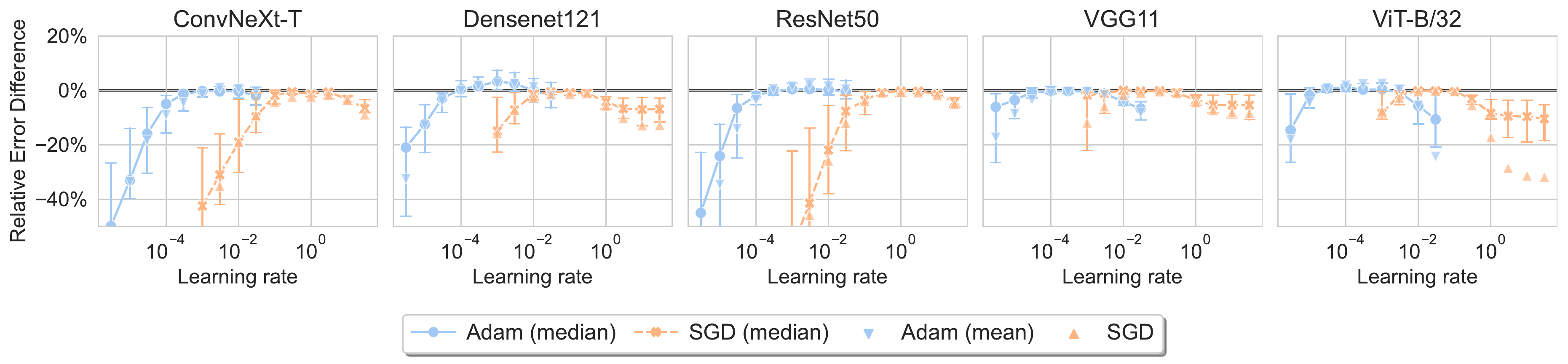}
			}
			\caption{\MeanResultsConvexCaption}
			
			\label{fig:mean-results-convex}
		\end{center}
	\end{figure}
 }

\subsection{Convex optimization}\label{subsec:convex-opt}
We also evaluate \DoG on convex optimization tasks, matching the assumptions of our theoretical analysis.
To do so, we perform multi-class logistic regression on features obtained from the computer vision models in our testbed, i.e., linear probes. 
We find that model-tuned SGD performs on par or worse than \DoG{}, while instance-tuned SGD barely gains any advantage (\Cref{fig:hpt-results-convex}), with RED values well under 1\% (corresponding to the difference between accuracies 90\% and 90.1\%).
Moreover, even in this simple setting, SGD is sensitive to the choice of learning rate, which differ significantly between models (\Cref{fig:mean-results-convex}). 
\newcommand{\ImageNetTableCaption}{\label{table:imagenet}ImageNet top-1 validation accuracies after fine-tuning a CLIP ViT-B/32 model for 25K training steps, with and without polynomial decay averaging (see \Cref{subsec:imagenet}). 
}
\begin{table}
	\centering
	\notarxiv{\caption{\ImageNetTableCaption} \vspace{4pt}\small}
    	\begin{tabular}{@{\extracolsep{4pt}}lccc}
	\toprule
	\arxiv{Algorithm & LR & Acc.\ w/o averaging & Acc.\ with averaging \\ 
	\cline{1-1} \cline{2-2} \cline{3-3} \cline{4-4}\rule{-2pt}{2.6ex}}
	\notarxiv{Algorithm & LR & Acc.\ w/o avg. & Acc.\ w/ avg. \\ 
	\cline{1-1} \cline{2-2} \cline{3-3} \cline{4-4}\rule{-2pt}{2.6ex}}

	\multirow[c]{5}{*}{SGD} & 1e-03 & 60.70\% & 60.49\% \\
	& 3e-03 & 73.62\% & 73.54\% \\
	& 1e-02 & 76.82\% & 76.80\% \\
	& 3e-02 & 77.51\% & 77.54\% \\
	& 1e-01 & 75.73\% & 75.71\% \\
	\midrule
	\DoG & - & 74.78\% & 77.22\% \\
	\midrule
	\multirow[c]{3}{*}{AdamW} & 1e-05 & 78.23\% & 78.25\% \\
	& 3e-05 & 79.04\% & 79.01\% \\
	& 1e-04 & 75.02\% & 74.97\% \\
	\midrule
	\LDoG & - & 78.20\% & \textbf{80.12\%} \\
	\bottomrule
\end{tabular}
\arxiv{\caption{\ImageNetTableCaption}}
\notarxiv{\vspace{-12pt}}
\end{table} 
\subsection{Fine-tuning on ImageNet}\label{subsec:imagenet}
To complement our main fine-tuning testbed, we perform a more limited experiment involving ImageNet as a downstream task, which is more expensive to tune due its larger scale. We fine-tune a ViT-B/32 CLIP model~\citep{radford2021learning} and compare \DoG and \LDoG to training with SGD or AdamW~\citep{loshchilov2019decoupled}. We use a training prescription similar to \citet{wortsman2022model}; see \Cref{app:experiments-imagenet} for additional details. \Cref{table:imagenet} shows the ImageNet top-1 validation accuracies of the final model checkpoints, with and without the polynomial decay averaging used throughout our experiments. \DoG performs similarly to SGD, but both algorithms perform significantly worse than AdamW, perhaps due to an insufficient iteration budget. \LDoG performs well in this setting, improving on AdamW by a little over 1 point.

\arxiv{%
\newcommand{\FromScratchableCaption}{\label{table:from-scratch}CIFAR-10 test accuracies after training a Wide ResNet 28-10 model from scratch for 200 epochs, with and without polynomial decay averaging (see \Cref{subsec:from-scratch}). $\dagger$ denotes the standard training configuration \citep[cf.][Table 2]{cubuk2019autoaugment}.} 
	
\begin{table}
	\centering
	\notarxiv{\caption{\FromScratchableCaption} \vspace{4pt}} 
		\begin{tabular}{@{\extracolsep{4pt}}lcccc}
		\toprule
		\arxiv{Algorithm & LR & Acc.\ w/o averaging & Acc.\ with averaging \\
		\cline{1-1} \cline{2-2} \cline{3-3} \cline{4-4}
		\rule{-2pt}{2.6ex}
	}
		\notarxiv{Algorithm & LR & Acc.\ w/o avg. & Acc.\ w/ avg. \\
		\cline{1-1} \cline{2-2} \cline{3-3} \cline{4-4}
		\rule{-2pt}{2.6ex}
	}

	\multirow[c]{5}{*}{SGD} 
	& 0.1 & 94.9\% & 94.9\% \\
	& 0.3 & 95.8\% & 95.6\% \\
	& 1 & \textbf{96.4\%} & 84.4\% \\
	& 3 & 95.9\% & 21.7\% \\
	& 10 & 10.0\% & 10.0\% \\
	\midrule
	\multirow[c]{5}{*}{\shortstack[l]{SGD w/\\ mom.\ 0.9} }
	& 0.01 & 95.0\% & 95.1\% \\
	& 0.03 & 95.8\% & 95.7\% \\
	& ~~0.1{\small~$\dagger$} & \underline{96.3\%} & 88.5\% \\
	& 0.3 & 95.8\% & 27.5\% \\
	& 1 & 42.0\% & 63.4\% \\
	\midrule
	\DoG & - & 85.2\% & \textbf{96.4\%} \\
	\midrule
	\multirow[c]{4}{*}{Adam} & 3e-05 & 91.1\% & 91.1\% \\
	& 1e-04 & 94.0\% & 94.0\% \\
	& 3e-04 & 93.5\% & 93.8\% \\
	& 1e-03 & 91.4\% & 91.6\% \\
	\midrule
	\LDoG & - & 83.2\% & 93.5\% \\
	\bottomrule
	\end{tabular}
    \arxiv{\caption{\FromScratchableCaption}}
\end{table} %
}

\subsection{Training from scratch}\label{subsec:from-scratch}
We conduct a preliminary experiment with training a model from scratch, specifically a Wide ResNet 28-10~\cite{zagoruyko2016wide} on  CIFAR-10~\cite{krizhevsky2009learning}; see \Cref{app:experiments-from-scratch} for details. \Cref{table:from-scratch} shows the test accuracy of the final checkpoint, with and without the polynomial averaging used throughout our experiments. Here \DoG performs on par with the setting's canonical training prescription of SGD with momentum 0.9 and learning rate 0.1~\cite{cubuk2019autoaugment}. In this setting Adam produces poorer results, and $\LDoG$ is 0.5 point worse than tuned Adam with the best learning rate, perhaps due to not reaching convergence.

\subsection{Comparison to other tuning-free methods}\label{subsec:other-optimizers}

We perform preliminary comparisons between  \DoG and \LDoG and other methods for removing the learning rate parameter: the Stochastic Polyak Step~\citep{loizou2021stochastic}, D-Adaptation~\citep{defazio2022parameter} and Continuous Coin Betting (COCOB)~\citep{orabona2017training}. 
In all cases, we find that \DoG and \LDoG provide better performance on most tasks and on average (see \Cref{table:language-comp-results,table:vision-comp-results}).
We provide detailed results in \Cref{app:other-optimizers}, where we also discuss the practical prospects of the bisection procedure of~\citet{carmon2022making}. 

\section{Related Work}

Previous attempts to design theoretically principled and practical optimization algorithms that do not require learning rate tuning approach the problem from a variety of perspectives, resulting in a large variety of proposed algorithms. \citet{rolinek2018l4,vaswani2019painless,paquette2020stochastic} lift classical line search technique from non-stochastic optimization to the stochastic setting, while~\citet{berrada2020training,loizou2021stochastic} do the same for the classical Polyak step size~\citep{polyak1987,hazan2019revisiting}. \citet{asi2019importance} develop a class of algorithms based on  stochastic proximal methods and demonstrate their improved robustness both theoretically and empirically. \citet{schaul2013no} use a stochastic quadratic approximation for designing learning rates that maximize the expected one-step objective decrease. \citet{chandra2022ultimate} nest hypergradient descent to make a method that is insensitive to initial hyper-parameter choices. However, none of these results are parameter-free in the same sense as \DoG: they either do not have converges guarantees, or have suboptimality bounds that blow up polynomially when the method's parameters do not match a problem-dependent value. In contrast, parameter-free methods have convergence rates that depend at most logarithmically on algorithmic parameters.

While the parameter-free optimization literature has focused mainly on theoretical schemes, a number of works also include empirical studies~\citep{orabona2014simultaneous,orabona2017training,kempka2019adaptive,chen2022better}. In particular, \citet{orabona2017training} build on coin-betting schemes to design an algorithm for training neural networks that has AdaGrad-style convergence guarantees for quasi-convex functions, showing promising results on neural network training problems.
In recent work \citet{chen2022better} obtain improved empirical results with an algorithm that leverages coin betting and truncated linear models. However, this method lacks theoretical guarantees.

In recent independent work~\citet{defazio2022parameter} propose a parameter-free dynamic step size schedule of dual averaging.
While our work has the same motivation and shares a number of technical similarities (including the use of weighted regret bounds and an independently obtained \Cref{lem:bound-a-ratios}), the proposed algorithms are quite different, and dual averaging is rarely used in training neural networks. (See additional discussion in \Cref{app-subsec:dadaptation}). 
 Moreover, \citet{defazio2022parameter} only prove parameter-free rates of convergence in the non-stochastic setting, while we establish high probability guarantees in the stochastic setting. Concurrently with our work, \citet{defazio2023learning} heuristically extended their dual averaging scheme to SGD- and Adam-like algorithms, reporting promising experimental results. 

Finally, a number of neural network optimization methods---LARS~\citep{you2017large}, LAMB~\citep{you2017scaling}, Adafactor~\citep{shazeer2018adafactor}, and Fromage~\citep{bernstein2020distance}---use the norm of neural network weights to scale the step size.  \DoG and \LDoG are similar in also using a norm to scale their step size, but they differ from prior work by considering the distance from initialization rather than the norm of the weights. We believe that this difference is crucial in making \DoG parameter-free, while the above-mentioned method have a learning-rate parameter to tune (though \citet{bernstein2020distance} report that a  single default value works well across different tasks). 

\section{Limitations and Outlook}\label{sec:discussion}
Our theoretical and empirical results place \DoG as a promising step toward a new generation of principled and efficient tuning-free optimization algorithms. However, much additional work is necessary for these algorithms to become ubiquitous. First, it is important to understand how to correctly combine \DoG with proven technique such as momentum, per-parameter learning rates, and learning rate annealing---this is a challenge both from a theoretical and a practical perspective. Second, it is important to gain a better understanding of situations where \DoG is more sensitive to the choice of $\reps$ than theory would have us expect. Our preliminary investigations suggest a connection to batch normalization, and following that lead could lead to even more robust training methods.
Finally, while our experiments aim to cover a broad range of tasks and architectures, future work needs to explore \DoG in additional settings, particularly those involving training from scratch.

\section*{Acknowledgments}
We thank Francesco Orabona, Mitchell Wortsman, Simon Kornblith and our anonymous reviewers for their insightful comments.
This work was supported by the NSF-BSF program, under NSF grant \#2239527
and BSF grant \#2022663.
MI acknowledges support from the Israeli council of higher education.
OH acknowledges support from Pitt Momentum Funds,  and
AFOSR grant \#FA9550-23-1-0242.
YC acknowledges support from the Israeli Science
Foundation (ISF) grant no. 2486/21, the Alon Fellowship, the Yandex Initiative for Machine Learning, and the Len Blavatnik and the Blavatnik Family Foundation. 

\arxiv{
\bibliographystyle{plainnat}
}
\notarxiv{
\bibliographystyle{icml2023}
}

\newpage
\appendix
\notarxiv{\onecolumn}

\section{Relaxing the Convexity Assumption}\label{app:beyond-convexity}

This section describes relaxations of convexity under which our main theoretical results still hold. In particular, our results naturally extend to  star-convex functions \cite{nesterov2006cubic} which satisfy
\[
f(x) -\fopt \le \inner{ \grad f(x) }{ x - \xopt } \quad ~\mbox{for all}~ x \in \X.
\]
Our results also extend (with changed constants) to quasarconvex functions \cite{hinder2020near}, which require that  $f(x) -\fopt \le c\inner{ \grad f(x) }{ x - \xopt }$ holds for some $c<\infty$ and all $x\in\X$. A further relaxation of star convexity requires it to hold only along the optimization trajectory:
\begin{assumption}[{\citet[Definition~2]{zhou2019sgd}}]
	\label{assume:trajectory-convexity}
	There exists $\xopt\in\argmin_{x} f(x)$ and constant $c < \infty$ such that the iterates of SGD satisfy
	\[
	f(x_k) -\fopt \le c\inner{ \grad f(x_k) }{ x_k - \xopt } ~~ \mbox{for all }~ k
	\]
	almost surely.
\end{assumption}
\citet{zhou2019sgd} introduce this notion of a ``star-convex path'' and provide some empirical evidence that it may hold when training deep neural networks with SGD (see also \citet{kleinberg2018alternative} for a related assumption). \citet{zhou2019sgd} also prove that the assumption suffices to prove that SGD converges to the global minimizer; it suffices for \DoG for similar reasons. 

When substituting  \Cref{ass:convex} with  \Cref{assume:trajectory-convexity}, our analysis goes through
unchanged, except we can no longer use Jensen's inequality to argue directly about the suboptimality of the point $\xbar[\tau]$. Instead, \Cref{thm:final-bound} with \Cref{assume:trajectory-convexity} says 
that, with probability at least $1-\delta$,
\begin{equation*}
	\sum_{k=0}^{\tau-1} \omega_k ( f(x_k) - \fopt) \le 
	O\prn*{ c_{\delta,\reps,\numSteps} \cdot \frac{ \d[0] \sqrt{\G[\tau] + \LL^2}}{\numSteps}},
\end{equation*}
with $\omega_k \defeq \frac{\rbar[k]}{\sum_{i=0}^{t-1} \rbar[i]}$, the final iterate tau $\tau$, and the coefficient $c_{\delta,\reps,\numSteps}$ as defined in \Cref{thm:final-bound}. (Note that  \Cref{assume:trajectory-convexity} implies $\sum_{k=0}^{t-1} \omega_k (f(x_k) -\fopt) \le  \sum_{k=0}^{t-1} \omega_k \inner{ \grad f(x_k) }{ x_k - \xopt }$ which replaces \eqref{eq:convexity-to-true-regret}).

We can turn the above bound into a constant-probability guarantee for a specific \TDoG iterate $x_{K}$ by sampling $K\sim \omega$ and using Markov's inequality:
\[
\Pr*(f(x_K) -\fopt \le  \e \sum_{k=0}^{\tau-1} \omega_k (f(x_k) -\fopt)) \le \e^{-1}.
\]

To obtain a high probability guarantee, we can make $l = \ceil{\log\frac{1}{\delta}}$ independent  draws from $\omega$, denoted $K_1,\ldots,K_l$ and use the fact that 
\[
\Pr*(\min_{i \le l} f(x_{K_i}) -\fopt \le  \e \sum_{k=0}^{\tau-1} \omega_k (f(x_k) -\fopt)) \le \delta.
\]
Finding the $i$ that minimizes $f(x_{K_i})$ requires a logarithmic number of evaluations of the exact objective. When this is not feasible, we can instead consider a statistical learning setup where we have sample access to stochastic functions $F(x)$ such that $\E F(x) = f(x)$ for all $x$ and, almost surely,  $F$ is $\LLstar$ Lipschitz in a ball of radius $3\d[0]$ around $x_0$. (The stochastic subgradient oracle $\gradientOracle{x}$ is then implemented by sampling $F$ and returning its subgradient at $x$). We can then sample $T$ new stochastic functions $F_1, \ldots, F_T$ and select 
$K^\star \in \arg\min_{k\in \{K_1,\ldots, K_l\}} \sum_{i=1}^{T} F_{i}(x_k)$. 
Straightforward application of Hoeffding's inequality shows that (when $\rbar[T] \le 3\d[0]$)
\begin{equation*}
	f(x_{K^\star}) - \fopt \le \min_{i \le l} f(x_{K_i}) -\fopt + O\prn*{\frac{\LLstar \d[0]}{\sqrt{T}} \sqrt{\log \frac{1}{\delta}}  }
\end{equation*}
with probability at least $1-\delta$.

We remark that the literature contains a plethora of other convexity relaxations such as quasiconvexity \cite{arrow1961quasi}, pseudoconvexity \cite{mangasarian1975pseudo}, Polyak-\L{}ojasiewicz conditions \cite{karimi2016linear} and weak convexity \cite{davis2019stochastic}. Exploring the convergence of \DoG under these additional convexity relaxations is left to future work.
 \notarxiv{%
\section{Relaxing the Global Stochastic Gradient Bound Assumption}\label{app:local-lipschitz}

Our results continue to hold essentially unchanged if we replace~\Cref{ass:bounded} (globally bounded stochastic gradient norm) with the following.
\begin{assumption}[Pointwise bounded stochastic gradients]\label{ass:pointwise-bounded}
	There exists some continuous function $\Lfunc:\xset\to\R$ such that 
	$\| \gradientOracle{x} \| \le \Lfunc(x)$ almost surely.
\end{assumption}
In particular, if we set
\begin{equation*}
	\LLstar \defeq \max_{x\in\xset:\norm{x-x_0} \le 3\norm{x_0-\xopt}} \Lfunc(x)
\end{equation*}
then \Cref{thm:final-bound} holds under \Cref{ass:pointwise-bounded} for any $\LL \ge \LLstar$. For full details and proof, refer to the arXiv version of our paper, available at \url{https://arxiv.org/abs/2302.12022}, where we use \Cref{ass:pointwise-bounded} throughout. 

\Cref{ass:pointwise-bounded} meaningfully relaxes \Cref{ass:bounded} when $\Lfunc$ can grow very large in $\xset$. For example, consider unconstrained least squares problems (with $\xset=\R^m$) with stochastic gradient oracle $\gradientOracle{x} = (\inner{a}{x} -b) a$ for random $a\in \R^\dim$ and $b\in\R$, such that $\norm{a}\le 1$ and $|b| \le 1$ hold with probability 1. In this case, there is no global upper bound on $\norm{\gradientOracle{x}}$, but \Cref{ass:pointwise-bounded} holds with $\Lfunc(x) = \norm{x} + 1$ and  $\LLstar = 3\norm{\xopt} + 1$ (assuming $x_0=0$).

However, to apply the~\eqref{eq:T-DoG} formula, one still requires an a-priori upper bound on $\LLstar$. In the above least-squares example, it is possible to bound $\LLstar$ given an upper bound  $D$ on $d_0=\norm{\xopt}$. However, if such a bound is available, then it is also possible to constrain the domain to a ball of radius $D$, where a global norm bound holds. Indeed, for all examples that we are aware of that have an a-priori bound on $\LLstar$,  this bound on $\LLstar$ arises from a bound on $d_0$.\footnote{We thank an anonymous reviewer for pointing this out.} In the arXiv version of our paper we solve this issue by changing the \TDoG{} formula to $\eta_t = \rbar[t]/\sqrt{\G[t]'}$ with
\begin{equation*}%
	\G[t]' = 8^4 \TimeUniformLog[\numSteps,\delta]^2 \log_{+}^2\prn*{\frac{t \Lbar[t]^2}{\Lbar[0]^2}}(\G[t-1] + 16  \Lbar[t]^2)
\end{equation*}
where $\Lbar[t] \defeq \max_{i\le t} \Lfunc(x_i)$. This variant of \TDoG{} attains essentially the same guarantees as the one presented here, but requires no prior knowledge of $\LLstar$.

We note that \Cref{ass:pointwise-bounded} and techniques similar to the \TDoG variant described above have previously appeared in the parameter-free online learning literature \citep{cutkosky2019artificial,mhammedi2020lipschitz}. However, these works do not also guarantee that the iterates remain close to the optimal solution and therefore do not obtain our dependence on the local Lipschitz constant $\LLstar$.

}
\section{Useful Algebraic Facts}\label{app:useful-algebra-fact}

\subsection{Proof of Lemma~\ref{lem:bound-a-ratios}}\label{app:bound-a-ratios}

\begin{proof}
	Define $K \defeq \lceil \log(s_{\numSteps} / s_0) \rceil$, and $n \defeq \lfloor \frac{\numSteps}{K} \rfloor$.
	Then, we have
	\[
	\log\left( \frac{s_{\numSteps}}{s_0} \right) \ge \sum_{k=0}^{K-1} \log\left( \frac{s_{n (k+1)}}{s_{n k}} \right) \ge K \min_{k < K} \log\left( \frac{s_{n (k+1)}}{s_{n k}} \right).
	\]
	Rearranging and using the definition of $K$ gives
	\[
	\min_{k < K} \log\left( \frac{s_{n (k+1)}}{s_{n k}} \right) \le \frac{\log\left( \frac{s_{\numSteps}}{s_0} \right)}{K} \le 1 \implies \min_{k < K} \frac{s_{n (k+1)}}{s_{n k}}\le \e.
	\]
	where the implication follows from monotonicity of the exponential function.
	Therefore, 
	\[
	\max_{t\le \numSteps}\sum_{i < t} \frac{s_i}{s_t} \ge \max_{t \in [n, \numSteps]}  n \frac{s_{t-n}}{s_t} = \max_{k \le K} n \frac{s_{n (k-1)}}{s_{n k}} \ge n \e^{-1} = \e^{-1} \floor*{ \frac{T}{\ceil*{ \log(s_{\numSteps} / s_0) }} } \ge \e^{-1} \frac{\numSteps}{\log(s_{\numSteps} / s_0) + 1} - \e^{-1},
	\]
	where the first inequality uses that $s$ is positive nondecreasing sequence and the second inequality uses $\min_{k < K} \frac{s_{n (k+1)}}{s_{n k}}\le \e$ as shown above.
\end{proof}

\subsection{Lemma~\ref{lem:adagrad-algebra}}\label{sec:adagrad-algebra}

\begin{lemma}\label{lem:adagrad-algebra}
	Let $a_0, \dots, a_t$ be a nondecreasing sequence of nonnegative numbers. Then 
	\[
	\sum_{k=1}^{t}\frac{a_k - a_{k-1}}{\sqrt{a_k}} \le 2 (\sqrt{ a_t } - \sqrt{a_0}).
	\]
\end{lemma}

\begin{proof}
	We have
	\[
	\sum_{k=1}^{t}\frac{a_k - a_{k-1}}{\sqrt{a_k}} = \sum_{k=1}^{t}\frac{(\sqrt{a_k} - \sqrt{a_{k-1}} ) (\sqrt{a_k} + \sqrt{a_{k-1}})}{\sqrt{a_k}} \le 2 \sum_{k=1}^t \prn*{\sqrt{a_k} - \sqrt{a_{k-1}}} = 2 (\sqrt{ a_t } - \sqrt{a_0}).
	\]
\end{proof}

\subsection{Lemma~\ref{lem:many-sequences-nondecreasing}}

\begin{lemma}\label{lem:many-sequences-nondecreasing}
	Let $a_1, \ldots, a_T$ and $b_1, \ldots, b_T$ be sequences in $\R$ such that $a_1, \ldots, a_T$
	is nonnegative and nondecreasing. Then, for all $t\le T$,
	\begin{equation*}
		\abs*{\sum_{i=1}^t a_i b_i} \le 2 a_t \max_{i \le t} \abs*{\sum_{j=1}^i b_j}.
	\end{equation*}
\end{lemma}

\begin{proof}
	Let $a'_{i}=a_{i}-a_{i-1}$ and $B_{i}=\sum_{j\le i}b_{j}$. Then (by discrete integration by
	parts)
	\[
	\sum_{i=1}^{t}a_{i}b_{i}=\sum_{i=1}^{t}a_{i}\left(B_{i}-B_{i-1}\right)=\sum_{i=1}^{t-1}\left(a_{i}-a_{i+1}\right)B_{i}+a_{t}B_{t}=a_{t}B_{t}-\sum_{i=1}^{t-1}a'_{i+1}B_{i}.
	\]
	Therefore
	\begin{flalign*}
		\left|\sum_{i=1}^{t}a_{i}b_{i}\right|
		& \overle{(i)}
		\left|a_{t}B_{t}\right|+\left(\sum_{i=1}^{t-1}\left|a_{i+1}'\right|\right)\max_{i<t}\left|B_{i}\right|\le\left(\left|a_{t}\right|+\sum_{i=1}^{t-1}\left|a_{i+1}-a_{i}\right|\right)\max_{i\le t}\left|B_{i}\right|
		\overle{(ii)} 2a_t \max_{i\le t}\left|B_{i}\right|,
	\end{flalign*}
	where $(i)$ uses the triangle and H\"older's inequalities, and $(ii)$ uses that $a_t$ is nonnegative and nondecreasing and therefore $\sum_{i=1}^{t-1}|a_{i+1}-a_i|=a_t - a_1 \le a_t$. 
\end{proof}

\subsection{\Cref{lem:bound-a-k-infinite-sum}}

Recall that $\log_+(z) \defeq 1 + \log(t)$.

\begin{lemma}\label{lem:bound-a-k-infinite-sum}
Let $a_{-1}, a_0, a_1, \dots, a_t$ be a nondecreasing sequence of nonnegative numbers, then
\begin{equation*}
	\sum_{k=0}^t \frac{a_{k} - a_{k-1}}{a_{k} \log_{+}^2(a_{k} / a_{-1})} \le 1.
\end{equation*}
\end{lemma}

\begin{proof}
	We have
\begin{flalign*}
	\sum_{k=0}^t \frac{a_{k} - a_{k-1}}{a_{k} \log_{+}^2(a_{k} / a_{-1})}  & 
	\le \sum_{k=0}^t \int_{a_{k-1} / a_0}^{a_{k} / a_{-1}} \frac{d \alpha}{\alpha \log_{+}^2(\alpha)} 
	=\int_{1}^{a_t / a_{-1}} \frac{d \alpha}{\alpha \log_{+}^2(\alpha)} 
	\\ & \le \int_{1}^{\infty} \frac{d \alpha}{\alpha \log_{+}^2(\alpha)} = \left[ \frac{1}{1 + \log(\alpha)} \right]_{1}^{\infty} = 1.
\end{flalign*}
\end{proof} %
\section{Proofs for \Cref{sec:theory}}

\notarxiv{

\subsection{Proof of \Cref{lem:bound-empirical-error}}\label{sec:lem:bound-empirical-error-proof}
\begin{proof}
	Using $x_{k+1} = \Proj{\X}{x_{k} - \eta_k g_k }$ we obtain the standard inequality
	$\d[k+1]^2 \le \| x_k - \eta_k g_k - \xopt \|^2 = \d[k]^2 - 2 \eta_k \inner{g_k}{x_k - \xopt} + \eta_k^2 \| g_k \|^2$.
	Rearranging this gives:
	\begin{flalign}\label{eq:classic-subgradient-inequality}
		\inner{ g_k }{ x_k - \xopt } \le \frac{\d[k]^2 - \d[k+1]^2}{2 \eta_k} + \frac{\eta_k \| g_k \|^2}{2}.
	\end{flalign}
	Therefore, $\sum_{k=0}^{t-1} \rbar[k] \inner{g_k}{x_k - \xopt }$ is at most
	\[
	\half \underbrace{\sum_{k=0}^{t-1} \frac{\rbar[k]}{\eta_k} (d_k^2 - d_{k+1}^2) }_{(A)}
	+
	\half \underbrace{\sum_{k=0}^{t-1} \rbar[k] \eta_k \norm{g_k}^2}_{(B)}.
	\]
	We bound the terms $(A)$ and $(B)$ in turn, beginning with the former:
	\begin{flalign*}
		(A)&= \sum_{k=0}^{t-1} \sqrt{\G[k]'} (d_{k}^2 - d_{k+1}^2)
		 =  \d[0]^2 \sqrt{\G[0]'}  - \d[t]^2  \sqrt{\G[t-1]'} + \sum_{k=1}^{t-1} \d[k]^2 \left(\sqrt{G_k'} - \sqrt{G_{k-1}'}\right) \\
		&\overle{(i)} \dbar[t]^2 \sqrt{\G[0]'} - \d[t]^2 \sqrt{\G[t-1]'}  + \dbar[t]^2 \sum_{k=1}^{t-1} \left(\sqrt{G_k'} - \sqrt{G_{k-1}'}\right)  
		= \sqrt{G_{t-1}'}\left( \dbar[t]^2 - \d[t]^2 \right) \overle{(ii)} 4\rbar[t] \dbar[t] \sqrt{G_{t-1}'}.
	\end{flalign*}
	Inequality $(i)$ uses $\d[k] \le \dbar[t]$ and that $\G[k]'$ is nondecreasing as per \Cref{property:DoG-style-step-sizes}. Inequality $(ii)$ holds since, for $s \in \argmax_{k \le t} \d[k]$, we have
	$\dbar[t]^2 - \d[t]^2 = \d[s]^2 - \d[t]^2 = (\d[s] - \d[t]) (\d[s] + \d[t]) \le \| x_s - x_t \| (d_s + d_t)
	\le (\rbar[s] + \rbar[t]) (\d[s] + \d[t]) \le 4 \rbar[t] \dbar[t]$.
	Bounding the second term $(B)$, we have:
	\begin{flalign*}
		(B) &= \sum_{k=0}^{t-1}\frac{\rbar[k]^2 \| g_k \|^2}{\sqrt{G_k'}} \le \sum_{k=0}^{t-1}\frac{\rbar[k]^2 \| g_k \|^2}{\sqrt{G_k}}
		 \le \rbar[t]^2 \sum_{k=0}^{t-1}\frac{\| g_k \|^2}{\sqrt{G_k}} \le 2\rbar[t]^2 \sqrt{G_{t-1}},
	\end{flalign*}
	where the final inequality uses the standard Lemma~\ref{lem:adagrad-algebra} with $a_k = G_k=\sum_{i\le k}\norm{g_i}^2$.
\end{proof} }

\subsection{Proof of \Cref{lem:bound-generalization-error}}\label{app:cor:product-mg-concentration}

We begin by citing the following corollary of a general bound due to \citet{howard2021time}.  (Recall that  $ \TimeUniformLog[t, \delta]\defeq \log \frac{60\log(6t)}{\delta}$).
\begin{corollary}[{\citet[Corollary 1]{carmon2022making}}]
	\label{cor:kickass-mg-concentration}
	Let $c > 0$ and $X_{t}$ be a martingale difference sequence adapted to $\filt_{t}$ such that $\left|X_{t}\right| \le c$
	with probability 1 for all $t$.
	Then, for all $\delta\in\left(0,1\right)$,
	and $\hat{X}_{t}\in\filt_{t-1}$ such that $\abs{\hat{X}_{t}} \le c$
	with probability 1,
	\begin{equation*}
		\P\prn*{
			\exists t \le \numSteps :\abs[\Bigg]{\sum_{s=1}^t X_s
			}
			\ge
			4 \sqrt{ \TimeUniformLog[t, \delta] \sum_{s= 1}^t \left(X_{s}-\hat{X}_{s}\right)^{2} + c^2 \TimeUniformLog[t, \delta]^2}
			\,}
		\le \delta.
	\end{equation*}
\end{corollary}

Building on \Cref{cor:kickass-mg-concentration} we prove the following result, which allows for the sequence $X_t$ to be almost-surely bounded by a random (rather than deterministic) quantity. 

\begin{corollary}\label{cor:kickass-mg-concentration-2}
	Let $C_t\in\mathcal{F}_{t-1}$ and let $X_{t}$ be a martingale difference sequence adapted to $\filt_{t}$ such that $\left|X_{t}\right| \le C_t$
	with probability 1 for all $t$.
	Then, for all $\delta\in\left(0,1\right)$, $c > 0$,
	and $\hat{X}_{t}\in\filt_{t-1}$ such that $\abs{\hat{X}_{t}} \le C_t$
	with probability 1,
	\begin{equation*}
		\P\prn*{
			\exists t \le \numSteps :\abs[\Bigg]{\sum_{s=1}^t X_s
			}
			\ge
			4 \sqrt{ \TimeUniformLog[t, \delta] \sum_{s= 1}^t \left(X_{s}-\hat{X}_{s}\right)^{2} + c^2 \TimeUniformLog[t, \delta]^2}
			\,}
		\le \delta + \P\prn*{\exists t \le \numSteps : C_t > c }.
	\end{equation*}
\end{corollary}

\begin{proof}
Define the random variables
\[
W_t \defeq \frac{X_t}{\max\{ c, C_t \}}
~~\mbox{and}~~
 \hat{W}_t \defeq \frac{\hat{X}_t}{\max\{ c, C_t \}}
\]
and note that they satisfy the requirements of \Cref{cor:kickass-mg-concentration}: $W_t$ is a martingale difference sequence adapted to $\filt_t$ while $\hat{W}_t \in \filt_{t-1}$ and they both have absolute value bounded by 1 almost surely. 

Next, define the  events
\[
E_T \defeq \left\{ \exists t< \numSteps :\abs[\Bigg]{\sum_{s=1}^t X_s
}
\ge
4 \sqrt{ \TimeUniformLog[t, \delta] \sum_{s= 1}^t \left(X_{s}-\hat{X}_{s}\right)^{2} + c^2 \TimeUniformLog[t, \delta]^2} \right\}
~~\mbox{and}~~
 H_{\numSteps} \defeq \{ \exists t \le \numSteps : C_t > c \}.
\]
Then we have,
\begin{flalign*}
\P\prn*{
	E_{\numSteps}
	} = \P\prn*{
	E_{\numSteps} \cap \neg H_{\numSteps} } + \P\prn*{ E_{\numSteps} \cap H_{\numSteps}  } \le \P\prn*{
	E_{\numSteps} \cap \neg H_{\numSteps} } + \P\prn*{ H_{\numSteps} }.
\end{flalign*}
Writing $\bar{C}_t = \max\{c,C_t\}$ for short, we have
\begin{flalign*}
	\P\prn*{
		E_{\numSteps} \cap \neg H_{\numSteps} } 
	&= \Pr*(
		\exists t \le {\numSteps} :\abs[\Bigg]{\sum_{s=1}^t \bar{C}_s W_s
	}
	\ge
	4 \sqrt{ \TimeUniformLog[t, \delta] \sum_{s= 1}^t \bar{C}_s^2 \left(W_{s}-\hat{W}_{s}\right)^{2} + c^2 \TimeUniformLog[t, \delta]^2}
	~\cap~
	 \neg H_{\numSteps}
	)
	\\ &
	\overeq{(i)}
	\Pr*(
	\exists t \le {\numSteps} :\abs[\Bigg]{\sum_{s=1}^t W_s
	}
	\ge
	4 \sqrt{ \TimeUniformLog[t, \delta] \sum_{s= 1}^t  \left(W_{s}-\hat{W}_{s}\right)^{2} +  \TimeUniformLog[t, \delta]^2}
	~\cap~
	\neg H_{\numSteps}
	)
	\\ & 
	 \le
	\Pr*(
	\exists t \le {\numSteps} :\abs[\Bigg]{\sum_{s=1}^t W_s
	}
	\ge
	4 \sqrt{ \TimeUniformLog[t, \delta] \sum_{s= 1}^t  \left(W_{s}-\hat{W}_{s}\right)^{2} +  \TimeUniformLog[t, \delta]^2}
	) \overle{(ii)} \delta,
\end{flalign*}
where $(i)$ uses the fact that $\bar{C}_s = c$ for all $s\le T$ when $\neg H_{\numSteps}$ holds, and $(ii)$ uses \Cref{cor:kickass-mg-concentration}.
\end{proof}

Next, we connect \Cref{cor:kickass-mg-concentration-2} with a handy algebraic fact (\Cref{lem:many-sequences-nondecreasing}) to obtain the following result, which underpins \Cref{lem:bound-generalization-error}.

\begin{lemma}\label{cor:product-mg-concentration}
	Let $S$ be the set of nonnegative and nondecreasing sequences.
	Let $C_t\in\mathcal{F}_{t-1}$ and let $X_{t}$ be a martingale difference sequence adapted to $\filt_{t}$ such that $\left|X_{t}\right| \le C_t$
	with probability 1 for all $t$.
	Then, for all $\delta\in\left(0,1\right)$, $c > 0$,
	and $\hat{X}_{t}\in\filt_{t-1}$ such that $\abs{\hat{X}_{t}} \le C_t$
	with probability 1,
	\begin{flalign*}
		\P\prn*{
			\exists t \le \numSteps, \exists \{y_i\}_{i=1}^{\infty} \in S :\abs[\Bigg]{\sum_{i=1}^t y_{i} X_{i}}
			\ge
			8 y_t \sqrt{ \TimeUniformLog[t,\delta]  \sum_{i= 1}^t  \left(X_{i}-\hat{X}_{i}\right)^{2} + c^2 \TimeUniformLog[t,\delta]^2}
			\,}
		\le \delta + \P\prn*{\exists t \le \numSteps : C_t > c }.
	\end{flalign*}
\end{lemma}

\begin{proof}%
	Follows from 
 \Cref{lem:many-sequences-nondecreasing} (with $y_i$ and $X_i$ taking the roles of $a_i$ and $b_i$, respectively), and  \Cref{cor:kickass-mg-concentration-2} that bounds $\max_{i\le t} \abs*{\sum_{i\le t} X_i}$ for all $t\le T$. 
\end{proof}

\begin{proof}[Proof of \Cref{lem:bound-generalization-error}]
For $k \in [\numSteps]$ define the random variables:
\[
Y_k = \rbar[k] \dbar[k], \quad X_{k} =  \inner{\Delta_{k}}{\frac{x_k - \xopt}{\dbar[k]}}, \quad \text{ and } \quad \hat{X}_{k} = -\inner{\grad f(x_{k}) }{\frac{x_k - \xopt}{\dbar[k]}}.
\]
From these definitions we get
\[
\sum_{k=0}^{t-1} Y_k X_k = \sum_{k=0}^{t-1} \rbar[k] \inner{\Delta_k}{x_k - \xopt}.
\]
Therefore,
\begin{flalign*}
&\P\prn*{
	\exists t \le \numSteps:\abs[\Bigg]{ \sum_{k=0}^{t-1} \rbar[k] \inner{\Delta_k}{x_k - \xopt} }
	\ge
	8 \rbar[t-1] \dbar[t-1] \sqrt{ \TimeUniformLog[t, \delta] \G[t-1] + \LL^2 \TimeUniformLog[t, \delta]^2  }
	\,} \\
&\le 	\P\prn*{
	\exists t \le \numSteps:\abs[\Bigg]{\sum_{k=0}^{t-1} Y_k X_k }
	\ge
	8 Y_t \sqrt{ \TimeUniformLog[t, \delta] \sum_{k=0}^{t-1} \left(X_{k}-\hat{X}_{k}\right)^{2} + \LL^2 \TimeUniformLog[t, \delta]^2 }
	\,}  \le \delta + \P\prn*{ \Lbar[\numSteps] > \LL }
\end{flalign*}
where the last inequality uses \Cref{cor:product-mg-concentration}.
\end{proof}

\subsection{Proof of \Cref{coro:error-bound}}\label{app:coro-error-bound-proof}

\begin{proof}
	If $T > 2 \log_{+}(\rbar[T] / \reps)$ then the corollary 
	follows by \Cref{prop:error-bound} and \Cref{lem:bound-a-ratios} with $s_t = \rbar[t]$, noting that the event $\rbar[T] \le D$ implies that $\Lbar[T] \le L_D$. 
	For the corner case when $T \le 2 \log_{+}(\rbar[T] / \reps)$
	we use that $f(\xbar[\tau]) - f_\star \le O(L \dbar[\tau]) \le  O( L (\rbar[\tau] + \d[0] ) )$ where the first inequality uses \eqref{eq:convexity-to-true-regret}, Cauchy-Schwartz and that $\| \grad f(x_t) \| \le L$; the second inequality uses the triangle inequality.
\end{proof} %
\subsection{\DoG can diverge on a pathological instance}\label{app:unstable-distance-DoG}

Consider the following variant of Nemirovski's function~\citep{nemirovski1983problem,nemirovski1994parallel} defined on $\R^{\dim}$:
\begin{equation*}
	f(x) = \max_{i \le \dim} \max\left\{ [x]_i, -\frac{1}{\sqrt{\dim}} [x]_i \right\},
\end{equation*}
where $[x]_i$ denotes the $i$'th coordinate of $x$ and $[x_0]_i = 10\reps / \sqrt{m}$  for all $i$, so that $\d[0] = 10\reps > \reps$. 
We show that applying \DoG on this function gives $\rbar[T] / \d[0] = \sqrt{T} / 10$ for all $T\le \dim$, meaning that the ratio $\rbar[T] / \d[0]$ can be made arbitrarily large by increasing $T$ and $\dim$. %

Define
\[
i(x) \defeq \min \argmax_{i \le \dim} \left\{ [x]_i, -\frac{[x]_{i}}{\sqrt{\dim}}  \right\},
\]
i.e., $i(x)$ is the smallest coordinate which is candidate for providing a subgradient.
Using this notation, a valid subgradient for $f$ is:
\[
\grad f(x) = \begin{cases} 
	e_{i(x)} & x_{i(x)} > 0 \\
	-\frac{1}{\sqrt{\dim}} e_{i(x)} & \text{otherwise}
\end{cases}
\]
where $e_{j}$ is a vector with one in the $j$th entry and zero elsewhere.
With this subgradient choice for $k \le \dim$ the iterates become:
\begin{flalign}\label{eq:xk}
	[x_{k}]_j = \begin{cases}
		10 \reps / \sqrt{\dim} - \reps & j < k \\
		10 \reps / \sqrt{\dim}	& j \ge k
	\end{cases}  
\end{flalign}
and therefore $\rbar[k] = \sqrt{k} \reps = \sqrt{k} \d[0] / 10$ as claimed.
We confirm \eqref{eq:xk} by induction. Since $[x_0]_i = 10 \reps / \sqrt{\dim}$ for all $i$, the expression \eqref{eq:xk}  holds for $k=0$. If  \eqref{eq:xk} holds for all $k \le n < \dim$ then 
\[
\grad f(x_k) = e_{k}%
\]
and therefore $\G[n] = n$ so that $\eta_n = \frac{\reps\sqrt{n}}{\sqrt{n}} = \reps$ and  $x_{n+1} = x_n - \frac{\sqrt{n}}{\sqrt{n}} \reps e_n$, meaning that
\begin{flalign*}
	[x_{n+1}]_j = \begin{cases}
		10 \reps / \sqrt{\dim} - \reps & j < n+1 \\
		10 \reps / \sqrt{\dim}	& j \ge n + 1 
	\end{cases}  
\end{flalign*}
which completes the induction. 
\subsection{Proof of \Cref{prop:distance-bound}}\label{sec:proof-prop:distance-bound}

\newcommand{\stoptime}{\mathcal{T}_{\mathrm{out}}}

To show iterate boundedness in the stochastic setting we define the stopping time
\begin{equation*}
	\stoptime = \min \crl*{ t: \rbar[t] > \D},
\end{equation*}
so that the event $\{\rbar[T] \le \D\}$ is the same as $\{\stoptime> T\}$. Let $\eta_k$ denote the sequence of \ref{eq:T-DoG}  step sizes (for given $L,T$ and $\delta$). To facilitate our analysis we also define the following truncated step size sequence:
\begin{equation}\label{eq:trunc-T-DoG}
	\etaTilde_k \defeq \begin{cases} \eta_k & k < \stoptime \\
		0 & \text{otherwise}. \end{cases}
\end{equation}
Truncating the step size allows us to rigorously handle the possibility that $\rbar[T]$ exceeds $\D$. In particular, the following holds for $\{\etaTilde_k\}$ but not for $\{\eta_k\}$.  (Recall that $\Delta_t \defeq g_t - \grad f(x_k)$).

\begin{lemma}\label{lem:properties-of-T-DoG}
Under \Cref{ass:bounded} the truncated \ref{eq:T-DoG} step sizes~\eqref{eq:trunc-T-DoG} satisfy, for all $t\le T$,
\begin{flalign}
\label{eq:predictable-eta}
\etaTilde_t & \in \sigma(g_0,\ldots,g_{t-1})~,	\\
\label{eq:eta-prob1-bound}
\abs*{\etaTilde_t\inner{\gamma}{x_t - \xopt}}&  \le \frac{6\d[0]^2}{8^2 \TimeUniformLog[T,\delta] }~\mbox{for }\gamma\in\{g_t,\grad f(x_t),\Delta_t\}~,
\\
\label{eq:bounded-sequence-eta-g-sum}
\sum_{k=0}^{t} \etaTilde_k^2 \| g_k \|^2 &\le \frac{9 \d[0]^2}{  8^4 \TimeUniformLog[T,\delta]}~\mbox{, and}
\\
\label{eq:eta-second-moment-bound}
\sum_{k=0}^{t} (\etaTilde_k\inner{g_k}{x_k - \xopt})^2 &\le \frac{ 12^2 \d[0]^4 }{  8^4 \TimeUniformLog[T,\delta] }.
\end{flalign}
\end{lemma}
\begin{proof}
	Equation \eqref{eq:predictable-eta} holds directly from the definition of \eqref{eq:T-DoG} and \eqref{eq:trunc-T-DoG}.
	
	To see the bound~\eqref{eq:eta-prob1-bound}, first note that that $\norm{\Delta_k} \le \norm{g_k} + \norm{\grad f(x_k)} \le 2\Lfunc(x_k)$. Since $\G[t]' \ge 4^2 8^4 \Lfunc^2(x_t) \TimeUniformLog[T,\delta]^2$ for all $t$,  the Cauchy-Schwartz inequality gives
	\begin{equation*}
		\abs*{\etaTilde_t\inner{\Delta_t}{x_t - \xopt}} \le 
		\frac{\rbar[t]}{\sqrt{\G[t]'}} \norm{\Delta_t} \d[t] \le \frac{1}{2\cdot 8^2\TimeUniformLog[T,\delta]} \rbar \d[t] \le \frac{6\d[0]^2}{ 8^2\TimeUniformLog[T,\delta]},
	\end{equation*}
	where the last inequality uses $\rbar[t]\le \D$ (or else $\etaTilde_t = 0$) and $\d[t] \le \d[0] + \rbar[t]$. Bounds for $\abs*{\etaTilde_t\inner{\gamma}{x_t - \xopt}}$ for $\gamma\in\{g_t,\grad f(x_t)\}$ follow by the same argument.
	
	To establish \eqref{eq:bounded-sequence-eta-g-sum}, first note that 
	$\sum_{k=0}^{t} \etaTilde_k^2 \| g_k \|^2 \le \sum_{k=0}^{\stoptime-1} \eta_k^2 \| g_k \|^2$ by the definition of $\etaTilde_k$. Furthermore
	\begin{flalign*}
		\sum_{k=0}^{\stoptime-1} \eta_k^2 \| g_k \|^2 = \sum_{k=0}^{\stoptime-1} \frac{\rbar[k]^2 \| g_k \|^2}{\G[k]'}  \overle{(i)}  \frac{ \rbar[\stoptime-1]^2}{8^4 \TimeUniformLog[T,\delta]} \sum_{k=0}^{\stoptime-1} \frac{G_k - G_{k-1}}{(G_k+\Lbar[k]^2) \log_+^2 \frac{G_k+\Lbar[k]^2}{\Lbar[0]^2}} \overle{(ii)} \frac{ 9\d[0]^2}{8^4 \TimeUniformLog[\numSteps,\delta]},
	\end{flalign*}
where $(i)$ uses that $\norm{g_k}^2 = \G[k] - \G[k-1]$ (with the shorthand $\G[-1] \defeq  0$) and 
\begin{equation*}
	\G[k]' 
	\ge 8^4  \TimeUniformLog[\numSteps,\delta]  (\G[k-1] + \norm{g_k}^2 + \Lbar[k]^2) \log_+^2\prn*{\frac{\sum_{i \le t} \Lbar[t]^2}{\Lbar[0]^2} } \ge 8^4 \TimeUniformLog[\numSteps,\delta]  (\G[k] + \Lbar[k]^2) \log_+^2 \frac{\G[k] + \Lbar[k]^2}{\Lbar[0]^2}
\end{equation*}
by \Cref{ass:bounded} which implies $\norm{g_k} \le \Lbar[k]$ for all $k$, while $(ii)$ uses \Cref{lem:bound-a-k-infinite-sum} with $a_k = \G[k] + \Lbar[k]^2$ and $\rbar[\stoptime-1] \le \D$.

The final bound \eqref{eq:eta-second-moment-bound} follows immediately from \eqref{eq:bounded-sequence-eta-g-sum} by noting that
\begin{equation*}
	\sum_{k=0}^{t} (\etaTilde_k\inner{g_k}{x_k - \xopt})^2 \le \sum_{k=0}^{t} \etaTilde_k^2 \| g_k \|^2 \d[k]^2 \le (4\d[0])^2 \sum_{k=0}^{t} \etaTilde_k^2 \| g_k \|^2,
\end{equation*}
where the first inequality follows from Cauchy-Schwartz and the second inequality from the fact that only terms with $k< \stoptime$ contribute to the sum.
\end{proof}

The above properties allow us to establish the following concentration bound.
\begin{lemma}\label{lem:distance-concentration}
	In the setting of \Cref{lem:properties-of-T-DoG},
	\begin{equation*}
		\Pr*(
		\exists t \le  T : \sum_{k=0}^{t-1} \etaTilde_k \inner{\Delta_k}{\xopt-x_k} > \d[0]^2
		) \le \delta.
	\end{equation*}
\end{lemma}

\begin{proof}
	 Consider the filtration $\filt_t = \sigma(g_0, \ldots, g_t)$ and define $X_t =  \etaTilde_t \inner{\Delta_t}{\xopt-x_t}$ and $\hat{X}_t = -\etaTilde_t \inner{\grad f(x_t)}{\xopt-x_t}$. Then, by~\eqref{eq:predictable-eta} we have that $X_t$ is a martingale difference sequence adapted to $\filt_t$ and $\hat{X}_t\in\filt_{t-1}$. Moreover, by~\eqref{eq:eta-prob1-bound} we have that $\max\{\abs{X_t},\abs{\hat{X}_t}\}\le c$ almost surely for $c= \frac{24\d[0]^2}{8^4 \TimeUniformLog[T,\delta]}$. Substituting into \Cref{cor:kickass-mg-concentration-2} (and shifting the start of the summation from $1$ to $0$) we have
	 \begin{equation*}
	 		\P\prn*{
	 		\exists t \le T:\abs[\Bigg]{\sum_{k=0}^{t-1} X_k
	 		}
	 		\ge
	 		4 \sqrt{ \TimeUniformLog[t, \delta] \sum_{k= 0}^{t-1} \left(X_{k}-\hat{X}_{k}\right)^{2} + c^2 \TimeUniformLog[t, \delta]^2}
	 		\,}
	 	\le \delta.
	 \end{equation*} 
 	Noting that $X_t - \hat{X}_t = \etaTilde_t \inner{g_t}{\xopt-x_t}$ and substituting the definition of $c$ and the bound~\eqref{eq:eta-second-moment-bound} gives, for every $t<T$,
 	 \begin{equation*}
 	 	4 \sqrt{ \TimeUniformLog[t, \delta] \sum_{k= 0}^{t-1} \left(X_{k}-\hat{X}_{k}\right)^{2} + c^2 \TimeUniformLog[t, \delta]^2}
 	 	\le 4 \sqrt{ \TimeUniformLog[t, \delta] \frac{12^2 \d[0]^4}{8^4  \TimeUniformLog[T, \delta]}+\prn*{\frac{6 \TimeUniformLog[t, \delta]\d[0]^2}{8^2\TimeUniformLog[T, \delta]}}^2 } \le \d[0]^2,
 	 \end{equation*}
   concluding the proof of lemma.
\end{proof}

Finally, we show that the event defined in \Cref{lem:distance-concentration} implies the desired distance bound. 

\begin{lemma}\label{lem:distance-bound-conclusion}
	In the setting of \Cref{prop:distance-bound}, if $\sum_{k=0}^{t-1}  \etaTilde_k \inner{\Delta_k}{\xopt-x_k} \le \d[0]^2$ for all $t\le T$ then $\stoptime > T$, i.e., $\rbar[T]\le \D$. 
\end{lemma}

\begin{proof}
	To condense notation, let $B_t \defeq \max_{t' \le t} \sum_{k=0}^{t'-1}  \etaTilde_k \inner{\Delta_k}{\xopt-x_k}$, so that the claim becomes $B_t \le \d[0]^2$ implies $\stoptime > t$ for all $t \le T$. We prove the claim by induction on $t$. The basis of the induction is that $\stoptime > 0$ always holds since $\rbar[0]=\reps \le \D$ by assumption. For the induction step, we assume that $B_{t-1}$ implies $\stoptime \ge t$ and show that $B_t\le \d[0]^2$ implies $\stoptime > t$. To that end, we use $\inner{\grad f(x_t)}{x_t-\xopt} \ge f(x_t)-\fopt \ge 0$ to rearrange  \eqref{eq:classic-subgradient-inequality} and obtain
	\begin{equation*}
		\d[k+1]^2 - \d[k]^2 \le \eta_k^2 \norm{g_k}^2 + 2\eta_k \inner{\Delta_k}{\xopt-x_k}  
	\end{equation*}
	for all $k$. Summing this inequality from $k=0$ to $k=t-1$, we get
	\begin{flalign*}
		\d[t]^2 - \d[0]^2  & \le \sum_{k=0}^{t-1}  \eta_k^2 \| g_k \|^2+  2\sum_{k=0}^{t-1}   \eta_k \inner{\Delta_k}{\xopt-x_k}
		 = %
		 \sum_{k=0}^{t-1} \etaTilde_k^2  \| g_k \|^2+ 2\sum_{k=0}^{t-1}  \etaTilde_k \inner{\Delta_k}{x_k - \xopt},
	\end{flalign*}
	where the equality holds since $\stoptime > t-1$ and therefore $\eta_k = \etaTilde_k$ for all $k\le t-1$. Now, by the bound~\eqref{eq:bounded-sequence-eta-g-sum} we have $ \sum_{k=0}^{t-1} \etaTilde_k^2  \| g_k \|^2 \le \frac{9^2}{8^4\TimeUniformLog[T,\delta]}\d[0]^2 \le \d[0]^2$. Moreover, $\sum_{k=0}^{t-1}  \etaTilde_k \inner{\Delta_k}{x_k - \xopt} \le B_t \le \d[0]^2$ by definition and assumption, from which we conclude that $\d[t]^2 \le 4\d[0]^2$ and hence $\r[t] \le \d[0] + \d[t] \le \D$. Since $\rbar[t] = \max\{\rbar[t-1], \r[t]\}$ and $\rbar[t-1]\le \D$ by the induction assumption, we have that $\rbar[t] \le \D$ as well, concluding the proof.
\end{proof}

\Cref{prop:distance-bound} follows immediately from \Cref{lem:distance-concentration,lem:distance-bound-conclusion}.

\subsection{Illustrating \DoG's guarantees for least squares problems}\label{app:least-squares-example}

In order to illustrate the advantage of \TDoG's iterate boundedness guarantee, we now instantiate \Cref{thm:final-bound} and \Cref{prop:distance-bound} for  stochastic least squares problems. Let $P$ be a distribution over pairs $(a,b)\in\R^m \times \R$, and for $(a,b)\sim P$ consider the gradient oracle \[\gradientOracle{x} = (\inner{a}{x}-b) a,\] corresponding to the objective function \[f(x) = \half\mathbb{E}_{(a,b)\sim P}(\inner{a}{x} - b)^2.
 \] For simplicity, we set $x_0=0$, and let $\xopt$ be the minimum norm minimizer of $f$. 

 If we assume that $\|a\| \le A$ and $|b| \le B$ with probability 1, then $\gradientOracle{\cdot}$ satisfies \Cref{ass:bounded} with 
 \begin{equation*}
\Lfunc(x) = A\norm{x} + B.
 \end{equation*}
 The bounds $\|a\| \le A$ and $|b| \le B$ are often easy to verify (e.g., via data normalization).

With the expression for $\Lbar$ in hand, we may apply \ref{eq:T-DoG} and use \Cref{prop:distance-bound} to guarantee that $\rbar[T] = O(\norm{\xopt})$ and hence $\Lbar[T] \le \LLstar = O(A\norm{\xopt} + B)$ with high probability. As a consequence, we may bound the observed squared gradient norm sum $\G[\tau]$ by $O((A\|x_\star\| + B)^2 T)$. Substituting into \Cref{thm:final-bound}, the optimality of \TDoG is
\[
\widetilde{O}\left(\frac{ A\| x_\star \|^2 + B\|x_\star\|}{\sqrt{T}} 
\right),\label{eq:ls-tdog-optgap-bound} \numberthis
\]
where $\widetilde{O}$ hides polylogarithmic terms in $T, \delta$, and $\reps$. 
We emphasize that the bound above depends on the value of $\norm{\xopt}$ (the smallest minimizer norm) even though \TDoG assumes no knowledge of this value. 

When the domain $\xset$ is unbounded (e.g., $\xset = \R^m$), previously-proposed general-purpose\footnote{ There exist parameter-free methods specialized to least-squares problems that obtain better without requiring an a-priori bound on the solution norm \citep{vovk2006line}.} parameter-free methods \citep[e.g.][]{orabona2016coin,cutkosky2018black,carmon2022making} do not directly apply, since there is no global bound on the magnitude of the stochastic gradients.\footnote{\citet{carmon2022making} guarantee boundedness of the point they output, but do not have guarantees on the magnitude of intermediate query points. \citet{orabona2021parameter} develop a parameter-free method with bounded iterates, but not by $O(\norm{x_0-\xopt})$.}
To use these methods, one must assume an a-priori bound $D$ on $\norm{\xopt}$ (e.g., by positing strong convexity of $f$) and constrain the domain to $\xset' = \{x\mid \norm{x} \le D\}$ where the stochastic gradients are globally bounded by $L = AD+B$. With such bounds in place, previous parameter-free methods yield optimality gap bounds of the form
\[
\widetilde{O}\prn*{
    \frac{\norm{\xopt}\sqrt{\G[T] + (AD+B)^2}}{T}
},
\]
where $\G[T]$ denotes the sum of square stochastic gradient norms observed by the algorithm. The coarseness 
of the upper bound $D \ge \norm{\xopt}$ affects the leading order term in the above display: since the iterates can be anywhere in $\xset'$, the best we can guarantee about $\G[T]$ is $\G[T] = O(L^2 T) = O((AD+B)^2 T)$. Therefore, the best deterministic upper bound on the optimality gap of previous methods is
\[
\widetilde{O}\prn*{
    \frac{AD\norm{\xopt} + B\norm{\xopt}}{\sqrt{T}} 
}.\label{eq:ls-prev-optgap-bound} \numberthis
\]

Comparing the bounds \eqref{eq:ls-tdog-optgap-bound} and \eqref{eq:ls-prev-optgap-bound}, we see that, for least squares problems, \TDoG can offer substantially stronger guarantees than previous parameter-free methods, even when we bound the domain to ensure that the latter methods apply.
\section{Guarantees for the unweighted \DoG iterate average}\label{app:unweighted}

In this section we derive guarantees similar to those presented in \Cref{sec:theory} for the \emph{unweighted} iterate average
\begin{equation*}
    \hx_t \defeq \frac{1}{t}\sum_{i=0}^{t-1} x_i.
\end{equation*}

The following proposition shows that the bound resulting from  combining \Cref{prop:error-bound} with \Cref{lem:bound-a-ratios} holds also for uniform iterate averaging.
\begin{proposition}\label{prop:error-bound-unweighted}
    For all $\delta \in (0,1)$ and $\LL > 0$,
    if \Cref{ass:convex}, \Cref{ass:bounded} hold then with probability at least $1 - \delta - \P\prn*{ \Lbar[\numSteps] > \LL }$ after $T$ iterations of any \DoG-like algorithm (\Cref{property:DoG-style-step-sizes}) we have 
    \begin{flalign}\label{eq:unweighted-error-bound}
            f(\hx_T) - f(\xopt) = O\prn*{\frac{ (\d[0] \log_+ \frac{\rbar[T]}{\reps} + \rbar[t]) \sqrt{\G[T-1]' + \G[T-1] \TimeUniformLog[T,\delta] + \LL^2 \TimeUniformLog[T,\delta]^2}}{T}}.
    \end{flalign}
    Moreover, in the noiseless case we have (with probability 1)
    \begin{flalign}\label{eq:unweighted-error-bound-noiseless}
        f(\hx_T) - f(\xopt) \le \frac{1}{T} \sum_{t=0}^{T-1} f(x_t) - f(\xopt)
        =
          O\prn*{\frac{ (\d[0] \log_+ \frac{\rbar[T]}{\reps} + \rbar[T]) \sqrt{\G[T-1]'}}{T}}.
   \end{flalign}
\end{proposition}

\begin{proof}
    Define the times
    \begin{equation}\label{eq:unweighted-error-bound-noiseless}
        \tau_i = \min\crl*{\min\crl*{i \mid \rbar[i] \ge 2 \rbar[\tau_{i-1}]},T},
    \end{equation}
    with $\tau_0 \defeq 0$. Moreover, let $K$ be the first index such that $\tau_K=T$ and note that $K\le 1+\log_2\frac{\rbar[T]}{\reps}$ by construction. 

    The argument proving \Cref{lem:bound-empirical-error} shows that
    \begin{equation*}
        \sum_{t=\tau_{k-1}}^{\tau_{k} - 1} \rbar[t] \inner{g_t}{x_t - \xopt} 
        \le \rbar[\tau_{k}] (2\dbar[\tau_k] + \rbar[\tau_k]) \sqrt{\G[\tau_k -1]'} = O\prn*{\rbar[\tau_{k}] (\d[0] + \rbar[\tau_k]) \sqrt{\G[T-1]'}}
    \end{equation*}
    for all $k\le K$. Moreover, by \Cref{lem:bound-generalization-error} we have that
    \begin{flalign*}
        \abs*{\sum_{t=\tau_{k-1}}^{\tau_{k} - 1} \rbar[t] \inner{\Delta_t}{x_t - \xopt}}
        &\le \abs*{ \sum_{t=0}^{\tau_{k} - 1} \rbar[t] \inner{\Delta_t}{x_t - \xopt} } + \abs*{\sum_{t=0}^{\tau_{k-1}-1} \rbar[t] \inner{\Delta_t}{x_t - \xopt}} 
        \\ &
        \le
        16 \rbar[\tau_{k} - 1] \dbar[\tau_{k}-1] \sqrt{\TimeUniformLog G_{T-1} + \TimeUniformLog^2 \LL^2}
    \end{flalign*}
    holds for all $k\le K$ with probability at least $1-\delta -\P(\Lbar[T]>L)$. (The rest of the analysis assumes this event holds).

    Combining these two bounds, we have
    \begin{flalign*}
        \sum_{t=\tau_{k-1}}^{\tau_{k} - 1} \brk*{f(x_k) - f(\xopt)}
        & \le 
        \frac{1}{\rbar[\tau_{k-1}]} \sum_{t=\tau_{k-1}}^{\tau_{k} - 1} \rbar[t] \brk*{f(x_k) - f(\xopt)}
        \le 
        \frac{1}{\rbar[\tau_{k-1}]} \sum_{t=\tau_{k-1}}^{\tau_{k} - 1} \rbar[t] \inner{\grad f(x_t)}{x_t-\xopt}
        \\ &
        = \frac{1}{\rbar[\tau_{k-1}]} \sum_{t=\tau_{k-1}}^{\tau_{k} - 1} \rbar[t] \inner{g_t}{x_t - \xopt} 
        - \frac{1}{\rbar[\tau_{k-1}]} \sum_{t=\tau_{k-1}}^{\tau_{k} - 1} \rbar[t] \inner{\Delta_t}{x_t - \xopt}
        \\ &
        = O\prn*{\frac{\rbar[\tau_{k}]}{\rbar[\tau_{k}-1]} (\d[0] + \rbar[\tau_k]) \sqrt{\G[T-1]' + \TimeUniformLog\G[T-1] + \TimeUniformLog^2 \LL^2}}.
        \\ &
        = O\prn*{ (\d[0] + \rbar[\tau_k]) \sqrt{\G[T-1]' + \TimeUniformLog\G[T-1] + \TimeUniformLog^2 \LL^2}}, \numberthis \label{eq:unweighted-segment-bound}
    \end{flalign*}
    where the last transition holds since, for any \DoG-like method and all $t$,
    \begin{equation*}
        \rbar[t+1] \le \rbar[t] + \norm{x_{t+1} - x_{t}}
        = \rbar[t] \prn*{ 1 + \frac{\norm{g_t}}{\sqrt{\G[t]'}}}
        \le 2\rbar[t].
    \end{equation*}

    Summing \Cref{eq:unweighted-segment-bound} over $k$ from $1$ to $K$, we obtain
    \begin{flalign*}
        \sum_{t=0}^{T - 1} \brk*{f(x_k) - f(\xopt)}
        & =
        \sum_{k=1}^{K}\sum_{t=\tau_{k-1}}^{\tau_{k} - 1} \brk*{f(x_k) - f(\xopt)}
        \\ & = O\prn*{\prn*{\d[0] K + \sum_{k=1}^K \rbar[\tau_k]}\sqrt{\G[T-1]' + \TimeUniformLog\G[T-1] + \TimeUniformLog^2 \LL^2}}.
    \end{flalign*}
    Finally, recall that $K=O\prn*{\log_+\frac{\rbar[T]}{\reps}}$ and note that 
    $\sum_{k=1}^K \rbar[\tau_k] = O(\rbar[T])$ since $\frac{\rbar[\tau_i]}{\rbar[\tau_{K-1}]} \le 2^{-K-1+i}$ for all $i\le K-1$. The proof of \Cref{eq:unweighted-error-bound} is complete upon noting that $f(\hx_T) \le \frac{1}{T}\sum_{t=0}^{T-1} f(x_t)$ by Jensen's inequality. To show \Cref{eq:unweighted-error-bound-noiseless} we simply set $\Delta_t=0$ in the bound above.
\end{proof}

Using \Cref{prop:error-bound-unweighted} we may replace the early-stopped weighted average $\bar{x}_\tau$ with the final unweighted average $\hx_T$ in \Cref{coro:error-bound} and \Cref{thm:final-bound}. 

\Cref{prop:error-bound-unweighted} also shows that (as long as iterates remain bounded) \DoG attains a $1/T$ rate of convergence in the smooth noiseless cases. In particular, if we assume that $\grad f$ is $S$-Lipschitz and noiseless gradients (i.e., $g_i = \grad f(x_i))$, then we have
\[
    G_{T-1} \le 2S \sum_{t=0}^{T-1} \brk*{f(x_t) - f(\xopt)}.
\]
Substituting this bound back into \Cref{eq:unweighted-error-bound-noiseless} (with $\G[T-1]' = \G[T-1]$ for \DoG), we obtain
\[
    \frac{1}{T} \sum_{t=0}^{T-1} f(x_t) - f(\xopt)
    =
      O\prn*{\frac{ (\d[0] \log_+ \frac{\rbar[T]}{\reps} + \rbar[T]) \sqrt{\frac{S}{T} \sum_{t=0}^{T-1} f(x_t) - f(\xopt)}}{\sqrt{T}}}.
\]
Dividing through by $\sqrt{\frac{1}{T} \sum_{t=0}^{T-1} f(x_t) - f(\xopt)}$ and squaring yields
\[
    f(\hx_{T}) - f(\xopt) \le 
    \frac{1}{T} \sum_{t=0}^{T-1} f(x_t) - f(\xopt)
    =
      O\prn*{\frac{ S(\d[0] \log_+ \frac{\rbar[T]}{\reps} + \rbar[T])^2}{T}}.
\] %
\section{Experiment Details}\label{app:exp-details}

\subsection{Environment settings}
All experiments were based on PyTorch \citep{paszke2019pytorch} (version 1.12.0). 

Language experiments were done with the \emph{transformers}~\cite{wolf2020transformers} library (version 4.21.0) and tracked using the \emph{Comet.ML}~\cite{CometML}. All datasets were provided by the \emph{Datasets} library~\cite{lhoest2021datasets} (version 2.4.0) and were left as is, including train-eval-test splits.

Vision experiments were based on the \emph{pytorch-image-models} (\texttt{timm}, version0.7.0dev0) repository \citep{Wightman2019timm}, with \emph{TensorFlow datasets} (version 4.6.0) as a dataset backend \citep{Abadi2015Tensorflow}.

To support the training and analysis of the results, we used \emph{numpy}~\cite{Harris2020Array}, \emph{scipy}~\cite{virtanen2020scipy}, \emph{pandas}~\cite{mckinney2010pandas} and \emph{scikit-learn}~\cite{pedregosa2011scikit}.

\subsection{Implementation details}\label{app-subsec:implementation} 
Whenever possible, we used existing scripts and recipes provided by \texttt{timm} and \emph{transformers} to fine-tune the models. We implemented \DoG, \LDoG and the polynomial model averaging as a subclass of PyTorch \emph{Optimizer} interface. We provide implementation of both in \url{https://github.com/formll/dog}.

\subsection{Datasets}\label{app-subset:datasets}
The datasets used in the language experiments are: \textbf{CoLA} \citep{warstadt2018neural}, \textbf{SST-2} \citep{socher2013recursive},  \textbf{MRPC} \citep{dolan2005automatically},  \textbf{QQP} \citep{Iyer2017First},  \textbf{STS-B} \citep{cer2017semeval},  \textbf{MNLI} \citep{williams2018broad}, \textbf{QNLI} \citep{rajpurkar2016squad}, and  \textbf{RTE} \citep{Dagan2006Recognising,bar2006recognising,giampiccolo2007recognizing,Bentivogli2009Recognizing}. Following \citet{liu2019roberta}, we discard WNLI \citep{levesque2011winograd} as it was found to be ill-defined and was reformulated differently in SuperGLUE \citep{wang2019superglue}.

The datasets used in the vision experiments are: \textbf{Caltech101} \citep{fei2004learning}, \textbf{CIFAR-100} \citep{krizhevsky2009learning}, \textbf{CLEVR-Dist} \citep{johnson2017clevr}, \textbf{DMLab} \citep{beattie2016deepmind}, \textbf{dSprites-Ori} \citep{dsprites17}, \textbf{DTD} \citep{cimpoi2014describing}, \textbf{Flowers102} \citep{nilsback2008automated}, \textbf{Pets} \citep{parkhi2012cats}, \textbf{Resisc45} \citep{cheng2017remote}, \textbf{Retinopathy} \citep{kaggle-diabetic-retinopathy}, \textbf{Sun397} \citep{xiao2010sun,xiao2016sun}, and \textbf{SVHN} \citep{netzer2011reading}.

\subsection{Models}
When fine-tuning RoBERTa (from the `roberta-base' checkpoint) on classification tasks, we follow the common technique of prepending a \emph{CLS} token to the input, 
and feeding its final representation to a one hidden-layer, randomly initialized MLP that is used as a classification head. 
For SQuAD, the classification head is tasked with multi-label classification, predicting the probability that each word (token) in the input is the beginning/end of the answer span, and we then used the span that has the maximum likelihood as the model's output.
When fine-tuning T5 (from the `t5-base' checkpoint), we treated all tasks as sequence-to-sequence tasks, translating classification labels to appropriate words (e.g. 0/1 to positive/negative) and then evaluated accuracy with exact match.
The computer vision pre-trained models were accessed via \texttt{timm}, and had randomly initialized classification heads. 
The strings used to load the models were: `convnext\textunderscore tiny', `resnet50', `densenet121', `vit\textunderscore base\textunderscore patch32\textunderscore 224\textunderscore in21k' and `vgg11'.

\newcommand{\TasksTableCaption}{\label{table:tasks-configurations}
Configuration used for each dataset in our testbed (\Cref{subsec:settings}).
For all language tasks, we used the batch size as in \citet{liu2019roberta}, and
 at least 150\% the number of steps used there, in order to ensure convergence. Learning rate (LR) warmup and annealing refers to tuning with SGD and Adam.  In all cases, both \DoG and \LDoG used neither warmup nor annealing. 
}

\begin{table}[t]
    \centering
    \footnotesize
    
    \notarxiv{\caption{\TasksTableCaption}}
    \resizebox{\linewidth}{!}{\begin{tabular}{@{\extracolsep{4pt}}lcccccc} 
    
    \toprule
    
    \textbf{Task} & \textbf{Batch size} & \textbf{Steps} & \textbf{Metric} & \textbf{LR warmup} & \textbf{LR annealing} & \textbf{Grad. clipping} \\ \cline{1-1}\cline{2-2}\cline{3-3}\cline{4-4}\cline{5-5}\cline{6-6}\cline{7-7}

    \rule{-2pt}{2.6ex}

    \textbf{VTAB datasets} & 128 & 20K & Accuracy & None & Cosine & None \\
    \textbf{SQuAD} & 48 & 5475 & $F_1$ & 10\% & Cosine & 1 \\
    \textbf{SST-2} & 32 & 31407 & Accuracy & 10\% & Cosine & 1 \\
    \textbf{CoLA} & 32 & 10000 & Matthews correlation & 10\% & Cosine & 1 \\
    \textbf{MRPC} & 32 & 1734 & $F_1$ & 10\% & Cosine & 1 \\
    \textbf{STSB} & 32 & 3281 & Pearson correlation & 10\% & Cosine & 1 \\ 
    \textbf{QNLI} & 32 & 49218 & Accuracy & 10\% & Cosine & 1 \\
    \textbf{RTE} & 32 & 10000 & Accuracy & 10\% & Cosine & 1 \\
    \textbf{QQP} & 32 & 160625 & $F_1$ & 10\% & Cosine & 1 \\
    \textbf{MNLI} & 32 & 184218 & Accuracy & 10\% & Cosine & 1 \\

    \bottomrule
    \end{tabular}}

    \arxiv{\caption{\TasksTableCaption}}
    \end{table} %

\subsection{Hyper-parameters}\label{app:experiments-hparams}
We trained each model/task combination a fixed number of steps (see \Cref{table:tasks-configurations}), performing evaluation every 500 update steps (except for the smaller datasets Caltech101, DTD, Flowers102 and Pets where we evaluated every 200) with both the latest checkpoint, and the polynomial averaged one (see below).
We did not use any weight decay. For language models, we left dropout at its default value in the transformers library.
We used batch sizes as is common practice for each task, as detailed in \Cref{table:tasks-configurations}.

\paragraph{Data augmentation in vision experiments.}
The VTAB suite~\citep{zhai2019large} divides its datasets into three categories: natural, specialized and structured, and we used a suitable data augmentation strategy for each of the categories. In particular, for structured datasets we simply resized the images to a (224, 224) resolution, while for the natural and specialized datasets we used the standard  ``inception crop''~\citep{szegedy2015going} at training time and a 0.875 center crop at test time. For natural datasets we additionally applied a color jitter operation with parameter 0.4 (as implemented in \texttt{timm}). Finally, we applied a random horizontal flip for all datasets except SVHN and dSprites-Ori, where such augmentation clearly interferes with the task.

\paragraph{Model selection in vision experiments.}
For computer vision experiments, we used the VTAB evaluation splits to select the best checkpoint, and then reported performance on the test split.  
Unlike the experiments accompanying the VTAB suite~\citep{zhai2019large},
we did not retrain selected models on the combination of training and validation data.

\paragraph{Repeated runs.}
To account for randomness, we repeated our fine-tuning experiments using multiple seeds. In most cases (with exceptions listed below) we repeated each \DoG and \LDoG training 5 times. For SGD and Adam repeating the training with all learning rates was computationally prohibitive, so instead for each task / model pair we repeated 5 times only the best-performing LR (i.e., instance-tuned LR) and the best-performing LR across all tasks for that model (i.e., model-tuned LR) according the validation split. A few experiments were too computationally expensive to repeat: for QQP and MNLI (which require a large step budget) we have only 1--3 repetitions per training configuration, and for ConvNeXt-T (which takes a long time per step) we did repeat the training runs. 

Each relative error difference (RED) score combines the error of two optimization methods (one being \DoG) on a particular model task combination. Given multiple seeds for each optimization method, we computed the RED scores for each possible seed combination. 
 In  \Cref{fig:mean-results,fig:hpt-results,fig:mean-results-convex,fig:hpt-results-convex} (which aggregate multiple tasks) we average over those multiple RED values and compute the statistics of the average RED. In per-task breakdowns such as \Cref{fig:per-task-results,table:language-results,table:vision-results} we report the statistics over the multiple RED values.

\paragraph{Baseline optimizers.} For both SGD and Adam, we used cosine learning rate decay, and searched over appropriate values for peak learning rate.
The base learning rate search space used when performing fine-tuning for each model/task combination can be found in \Cref{table:language-results,table:vision-results}. We did not use momentum for SGD. For Adam we used $\beta_1=0.9$ for all experiments, and $\beta_2=0.999$ for language experiments and $\beta_2=0.99$ for vision experiments. For language models only, we used warmup of 10\% of the maximum steps count, and gradient clipping at global norm 1. We did not perform learning rate warmup or gradient clipping for the vision experiment since we did not encounter any training stability issues there.

\paragraph{Setting $\reps$.}
As explained in \Cref{subsec:settings}, setting $\reps \defeq \alpha (1+\norm{x_0})$ generally works well for $\alpha=10^{-4}$. However, in some cases such as with T5, $\lVert x_0 \rVert$ can be very large, causing destructively large first updates, with $\eta_t$ increasing exponentially and the model diverging. This is easily detectable early during training, as usually $\eta_t$ exceeds $1000$ within the first $100$ steps. Since the theory requires $\reps$ to be small, we simply decreased $\alpha$ by a factor of 100. While preliminary experiments with RoBERTa indicated that \DoG also performed well with $\alpha=10^{-4}$, for the sake of consistency we use the same values in all models of the same domain. Thus, models fine-tuned on vision tasks used $\alpha=10^{-4}$, while language models used $10^{-6}$ for \DoG{} and $10^{-8}$ for \LDoG{}. 

\paragraph{Model averaging.} As mentioned in \Cref{subsec:settings}, we used the polynomial decay averaging as proposed by \citet{shamir2013stochastic}. Namely, we kept an additional copy of the model weights, and in every update step we updated our running average of the model parameters as follows:
\begin{equation}\label{eq:constant8}
	\bar{x}_t^\gamma = \left(1 - \frac{1+\gamma}{t+\gamma} \right)\bar{x}_{t-1}^\gamma + \frac{1+\gamma}{t+\gamma}x_t
\end{equation}
The vector $\bar{x}_t^\gamma$ roughly corresponds to an average of the last $t/\gamma$ iterates preceding iteration $t$. 
For all models, we set $\gamma=8$. We did not perform any tuning of the parameter $\gamma$; we chose the value $8$ because $1/8$ seemed like a good fraction of iterates to average, and because it worked well in the experiments of~\citet{levy2020large}. 

To ensure that iterate averaging is never harmful, for each optimization method we selected the best-performing checkpoint across both $x_t$ and $\bar{x}_t^\gamma$ (i.e., with or without averaging).

\subsection{\Cref{fig:1} details}\label{app:experiments-figure-1}
We generated \Cref{fig:1} as part of our fine-tuning testbed. In particular, SGD used a cosine learning rate annealing (without warmup), both algorithms use polynomial decay averaging, and we report test performance on the best checkpoint selected on a validation set.

\subsection{Fine-tuning ImageNet}\label{app:experiments-imagenet}
Our training setup mostly followed the default configuration in \citet{wortsman2022model}. In particular, we used batch size 512 and the default \texttt{timm} augmentation (as in our main computer vision experiments) which \citet{wortsman2022model} refer to as `medium aug.' We trained for 25K steps, corresponding to roughly 10 passes over the data. However (keeping with our computer vision testbed setting) we did not perform learning rate warmup or gradient clipping, and we initialized the classification head to be random. 

For AdamW~\citep{loshchilov2019decoupled} we used weight decay 0.1 and cosine learning rate annealing as in \citet{wortsman2022model}. We obtained accuracies within 0.5\% of the numbers reported in Appendix L of \citet{wortsman2022model}.

For \DoG and \LDoG we used weight decay 0 since the value 0.1 is meant for decoupled weight decay and we did not wish to re-tune a weight decay parameter. We set $\reps$ to be $10^{-6} \cdot (1+\norm{x_0})$ without trying different values of this parameter.

For SGD we used cosine learning rate annealing and set weight decay to 0 for a more direct comparison to \DoG.

\subsection{Training from scratch}\label{app:experiments-from-scratch}
Our training setup followed the basic training configuration of~\citet{cubuk2019autoaugment}, which is typical for training ResNets on CIFAR-10: data augmentations comprising a random crop after 4 pixel padding and random horizontal flip, batch size of 128, weight decay of 0.0005 and 200 epochs of training. SGD used cosine learning weight annealing and (when applicable) Nesterov momentum. We did not use dropout or any other additional form of regularization. For \DoG and \LDoG, 
we set $\reps=10^{-4} \cdot (1+\norm{x_0})$ without trying different values of this parameter.
The accuracies we obtained using SGD and \DoG are consistent (and slightly better) than the baseline number reported in Table 2 of \citet{cubuk2019autoaugment} and within 0.1\% of the one reported in Table 1 of~\citet{carmon2019unlabeled}.

\notarxiv{%
\newcommand{\FromScratchableCaption}{\label{table:from-scratch}CIFAR-10 test accuracies after training a Wide ResNet 28-10 model from scratch for 200 epochs, with and without polynomial decay averaging (see \Cref{subsec:from-scratch}). $\dagger$ denotes the standard training configuration \citep[cf.][Table 2]{cubuk2019autoaugment}.} 
	
\begin{table}
	\centering
	\notarxiv{\caption{\FromScratchableCaption} \vspace{4pt}} 
		\begin{tabular}{@{\extracolsep{4pt}}lcccc}
		\toprule
		\arxiv{Algorithm & LR & Acc.\ w/o averaging & Acc.\ with averaging \\
		\cline{1-1} \cline{2-2} \cline{3-3} \cline{4-4}
		\rule{-2pt}{2.6ex}
	}
		\notarxiv{Algorithm & LR & Acc.\ w/o avg. & Acc.\ w/ avg. \\
		\cline{1-1} \cline{2-2} \cline{3-3} \cline{4-4}
		\rule{-2pt}{2.6ex}
	}

	\multirow[c]{5}{*}{SGD} 
	& 0.1 & 94.9\% & 94.9\% \\
	& 0.3 & 95.8\% & 95.6\% \\
	& 1 & \textbf{96.4\%} & 84.4\% \\
	& 3 & 95.9\% & 21.7\% \\
	& 10 & 10.0\% & 10.0\% \\
	\midrule
	\multirow[c]{5}{*}{\shortstack[l]{SGD w/\\ mom.\ 0.9} }
	& 0.01 & 95.0\% & 95.1\% \\
	& 0.03 & 95.8\% & 95.7\% \\
	& ~~0.1{\small~$\dagger$} & \underline{96.3\%} & 88.5\% \\
	& 0.3 & 95.8\% & 27.5\% \\
	& 1 & 42.0\% & 63.4\% \\
	\midrule
	\DoG & - & 85.2\% & \textbf{96.4\%} \\
	\midrule
	\multirow[c]{4}{*}{Adam} & 3e-05 & 91.1\% & 91.1\% \\
	& 1e-04 & 94.0\% & 94.0\% \\
	& 3e-04 & 93.5\% & 93.8\% \\
	& 1e-03 & 91.4\% & 91.6\% \\
	\midrule
	\LDoG & - & 83.2\% & 93.5\% \\
	\bottomrule
	\end{tabular}
    \arxiv{\caption{\FromScratchableCaption}}
\end{table} %
}

\section{Additional experiment results}\label{app-subsec:results}

\subsection{Full breakdown of main experiment results}\label{app:experiments-breakdown}
\Cref{fig:per-task-results}
as well as \Cref{table:language-results,table:vision-results} provide the full breakdown of our main fine-tuning results, comparing \DoG and \LDoG to SGD and Adam with different learning rates for each model/task combination.

\begin{figure*}[t]
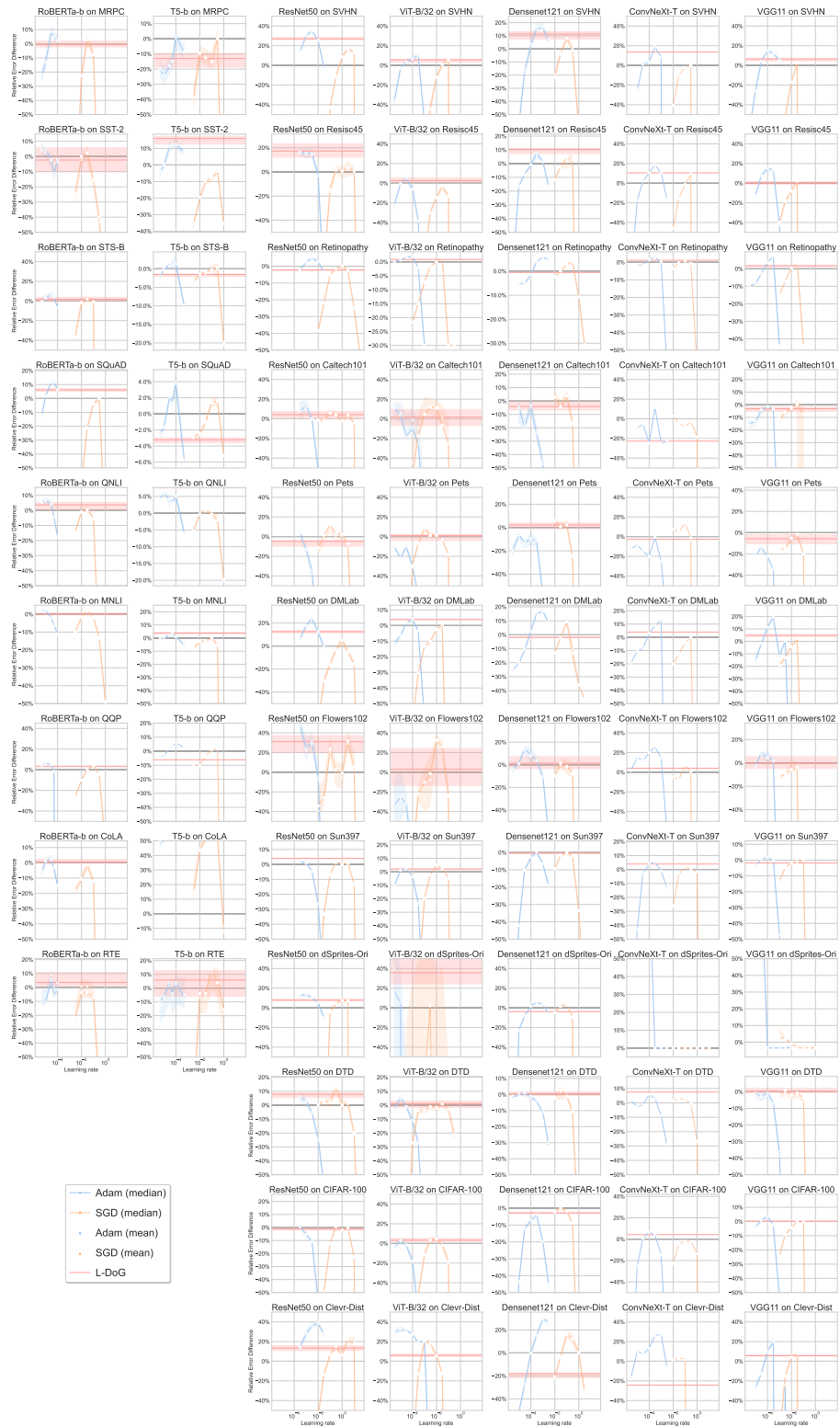

	\begin{center}
	\centerline{
	\notarxiv{\includegraphics[height=0.95\textheight,keepaspectratio]{figures/per_task_results.pdf}}
	\arxiv{\includegraphics[height=0.9\textheight,keepaspectratio]{figures/per_task_results.pdf}}
	}
	\caption{Relative error difference (RED) statistics across seeds (median, mean and IQR shown as shaded region) for all model/task combinations. The red horizontal line shows the median RED of \LDoG.
	} 
	\label{fig:per-task-results}
	\end{center}
\end{figure*}

\newcommand{\LanguageTableCaption}{ \label{table:language-results}Average (std) performance of RoBERTa-b and T5-b on language tasks, when fine-tuned with different optimization algorithms and their respective base learning rate (when applicable). \DoG{} uses $\reps=10^{-6}(1+\norm{x_0})$ and \LDoG{} uses $\reps=10^{-8}(1+\norm{x_0})$. Scores are reported as mean across seeds, measured in the corresponding performance metric as detailed in \Cref{table:tasks-configurations}.}

\begin{table}[t]
    \centering
    \footnotesize
       \setlength{\belowcaptionskip}{-10pt}
    \notarxiv{\caption{\LanguageTableCaption}}
    \resizebox{\linewidth}{!}{\begin{tabular}{@{\extracolsep{4pt}}lcccccccccccc} 
    
        \toprule

     \textbf{Model} & \textbf{Optimizer} & \textbf{LR} &  CoLA & MNLI & MRPC & QNLI & QQP & RTE & SQuAD & SST-2 & STS-B & Avg. \\ \cline{1-1} \cline{2-2} \cline{3-3} \cline{4-4} \cline{5-5} \cline{6-6} \cline{7-7} \cline{8-8} \cline{9-9} \cline{10-10} \cline{11-11} \cline{12-12} \cline{13-13}
     \rule{-2pt}{2.6ex}

     \multirow[c]{13}{*}{\textbf{RoBERTa-b}} & \multirow[c]{5}{*}{\textbf{Adam}} & \textbf{5e-06} & 60.8 & 87.8 & 89.8 & 93.0 & 88.7 & 77.6 & 90.3 & 95.1 {\smaller (0.34)} & 90.4 & 85.46 \\
 &  & \textbf{1e-05} & 63.4 & 87.9 {\smaller (0.05)} & 90.5 & 93.1 {\smaller (0.22)} & 89.0 {\smaller (0.11)} & 77.6 & 91.5 & 95.1 & 90.8 & 86.22 \\
 &  & \textbf{3e-05} & 63.5 {\smaller (1.33)} & 87.3 & 92.3 {\smaller (0.44)} & 92.9 {\smaller (0.15)} & 88.8 & 80.6 {\smaller (1.19)} & 92.3 {\smaller (0.05)} & 94.8 {\smaller (0.16)} & 91.1 {\smaller (0.19)} & 86.94 \\
 &  & \textbf{5e-05} & 61.8 & 86.8 & 92.0 & 92.3 & 88.0 & 78.7 & 92.4 {\smaller (0.07)} & 94.3 & 90.9 & 85.93 \\
 &  & \textbf{0.0001} & 57.5 & 86.2 & 91.8 & 91.3 & 0.0 & 79.4 & 91.9 & 94.7 & 89.8 & 75.33 \\
\cline{2-13}
 & \multirow[c]{6}{*}{\textbf{SGD}} & \textbf{0.003} & 56.3 & 86.4 & 81.9 & 91.5 & 85.1 & 75.5 & 79.4 & 93.6 & 87.0 & 81.44 \\
 &  & \textbf{0.01} & 59.1 & 87.4 & 89.7 & 92.5 & 87.0 & 79.0 {\smaller (0.69)} & 86.2 & 94.8 & 90.3 {\smaller (0.25)} & 84.81 \\
 &  & \textbf{0.03} & 62.3 {\smaller (1.38)} & 87.8 {\smaller (0.04)} & 91.8 {\smaller (0.21)} & 92.7 {\smaller (0.18)} & 88.3 & 78.9 {\smaller (0.79)} & 89.5 {\smaller (0.12)} & 95.0 {\smaller (0.15)} & 90.7 {\smaller (0.12)} & 86.30 \\
 &  & \textbf{0.1} & 58.7 & 87.4 & 91.0 & 92.2 & 88.7 {\smaller (0.06)} & 78.3 & 91.0 & 94.0 & 90.4 & 85.52 \\
 &  & \textbf{0.3} & 0.0 & 86.0 & 81.2 & 85.3 & 87.7 & 64.6 & 91.3 {\smaller (0.11)} & 92.8 & 27.0 & 68.14 \\
 &  & \textbf{1.0} & 0.0 & 81.5 & 81.2 & 83.8 & 79.4 & 52.7 & 82.8 & 89.7 & 13.0 & 62.22 \\
\cline{2-13}
 & \textbf{\DoG} & \textbf{-} & 62.8 {\smaller (1.17)} & 87.7 {\smaller (0.12)} & 91.6 {\smaller (0.29)} & 92.6 {\smaller (0.15)} & 88.2 {\smaller (0.02)} & 78.5 {\smaller (2.91)} & 91.3 {\smaller (0.17)} & 94.9 {\smaller (0.26)} & 90.5 {\smaller (0.33)} & 86.46 \\
\cline{2-13}
 & \textbf{\LDoG} & \textbf{-} & 63.3 {\smaller (0.32)} & 87.7 {\smaller (0.12)} & 91.5 {\smaller (0.19)} & 92.8 {\smaller (0.28)} & 88.7 {\smaller (0.14)} & 80.1 {\smaller (1.00)} & 91.8 {\smaller (0.18)} & 94.8 {\smaller (0.54)} & 90.6 {\smaller (0.34)} & 86.81 \\
\cline{1-13} \cline{2-13}
\multirow[c]{14}{*}{\textbf{T5-b}} & \multirow[c]{6}{*}{\textbf{Adam}} & \textbf{5e-06} & 53.4 {\smaller (0.93)} & 86.8 & 91.4 & 93.4 & 88.0 & 79.8 & 90.3 & 93.9 & 90.4 & 84.82 \\
 &  & \textbf{1e-05} & 56.0 {\smaller (0.63)} & 86.9 & 91.2 & 93.5 & 88.2 & 80.5 & 90.4 & 94.2 & 90.6 & 85.33 \\
 &  & \textbf{3e-05} & 58.9 {\smaller (1.10)} & 87.1 & 91.6 & 93.4 {\smaller (0.16)} & 88.8 & 82.3 & 90.8 & 94.8 {\smaller (0.24)} & 90.7 & 86.12 \\
 &  & \textbf{5e-05} & 58.9 {\smaller (0.80)} & 87.3 & 91.8 & 93.3 & 89.0 & 80.9 & 90.7 & 94.8 & 90.8 {\smaller (0.10)} & 85.97 \\
 &  & \textbf{0.0001} & 58.3 {\smaller (0.80)} & 86.9 & 92.9 {\smaller (0.35)} & 93.5 {\smaller (0.10)} & 89.2 & 82.5 {\smaller (0.48)} & 90.9 {\smaller (0.15)} & 94.9 {\smaller (0.26)} & 90.8 {\smaller (0.18)} & 86.53 \\
 &  & \textbf{0.0005} & 55.4 {\smaller (0.45)} & 86.1 & 92.3 & 92.7 & 88.8 & 81.2 {\smaller (1.48)} & 90.0 & 94.6 & 89.7 & 85.29 \\
\cline{2-13}
 & \multirow[c]{6}{*}{\textbf{SGD}} & \textbf{0.003} & 22.9 {\smaller (1.64)} & 85.8 & 90.1 & 92.7 & 87.5 & 66.8 & 90.3 & 92.1 & 90.3 & 79.43 \\
 &  & \textbf{0.01} & 49.4 {\smaller (0.27)} & 86.4 & 92.2 & 93.1 & 87.4 & 80.9 & 90.3 & 93.0 & 90.5 & 84.49 \\
 &  & \textbf{0.03} & 56.4 {\smaller (0.52)} & 86.5 & 92.0 & 93.2 {\smaller (0.03)} & 88.1 & 80.9 & 90.5 & 93.6 & 90.6 & 85.40 \\
 &  & \textbf{0.1} & 58.9 {\smaller (0.82)} & 86.8 & 91.8 & 93.1 & 88.7 & 84.1 & 90.7 {\smaller (0.04)} & 93.7 & 90.6 {\smaller (0.09)} & 86.13 \\
 &  & \textbf{0.3} & 56.7 {\smaller (0.83)} & 86.1 & 92.8 {\smaller (0.47)} & 93.0 {\smaller (0.13)} & 88.7 & 82.8 {\smaller (1.24)} & 90.7 {\smaller (0.05)} & 93.9 {\smaller (0.15)} & 90.7 {\smaller (0.21)} & 86.07 \\
 &  & \textbf{1.0} & 0.0 {\smaller (0.00)} & 32.8 & 81.2 & 91.7 & 56.9 & 78.7 & 90.1 & 92.1 & 88.7 & 67.56 \\
\cline{2-13}
 & \textbf{\DoG} & \textbf{-} & 7.3 {\smaller (6.78)} & 86.9 {\smaller (0.21)} & 92.8 {\smaller (0.35)} & 93.1 {\smaller (0.09)} & 88.5 & 81.7 {\smaller (3.06)} & 90.6 {\smaller (0.05)} & 94.1 {\smaller (0.19)} & 90.7 {\smaller (0.09)} & 80.58 \\
\cline{2-13}
 & \textbf{\LDoG} & \textbf{-} & 59.9 {\smaller (1.43)} & 87.3 {\smaller (0.10)} & 91.9 {\smaller (0.32)} & 93.6 {\smaller (0.02)} & 87.8 & 83.1 {\smaller (0.78)} & 90.3 {\smaller (0.02)} & 95.0 {\smaller (0.19)} & 90.5 {\smaller (0.05)} & 86.51 \\

 \bottomrule
    \end{tabular}}
    \arxiv{\caption{\LanguageTableCaption}}

    \end{table} %
\newcommand{\VisionTableCaption}{\label{table:vision-results}Average (std) test accuracy across seeds for vision tasks, when fine-tuned with different optimization algorithms and their respective base learning rate when applicable. \DoG{} and \LDoG{} use $\reps=10^{-4}(1+\norm{x_0})$.}

\begin{table}[t]
    \centering
    \footnotesize
\notarxiv{\caption{\VisionTableCaption}}
    
    \resizebox{\linewidth}{!}{\begin{tabular}{@{\extracolsep{4pt}}lccccccccccccccc} 
        \toprule
    \textbf{Model} & \textbf{Optimizer} & \textbf{LR} & Caltech101 & CIFAR-100 & Clevr-Dist & DMLab & dSprites-Ori & DTD & Flowers102 & Pets & Resisc45 & Retinopathy & Sun397 & SVHN & Avg. \\ \cline{1-1} \cline{2-2} \cline{3-3} \cline{4-4} \cline{5-5} \cline{6-6} \cline{7-7} \cline{8-8} \cline{9-9} \cline{10-10} \cline{11-11} \cline{12-12} \cline{13-13} \cline{14-14} \cline{15-15} \cline{16-16}
    \rule{-2pt}{2.6ex}
    
    \multirow[c]{16}{*}{\textbf{ConvNeXt-T}} & \multirow[c]{7}{*}{\textbf{Adam}} & \textbf{3e-06} & - & 70.0 & 89.2 & 70.5 & 86.5 & 72.1 & 92.5 & 93.1 & 91.2 & - & 36.7 & - & - \\
 &  & \textbf{1e-05} & 89.0 & 84.7 & 91.7 & 72.7 & 95.7 & 70.9 & 93.5 & 93.3 & 94.7 & 83.2 & 64.2 & 96.6 & 85.33 \\
 &  & \textbf{3e-05} & 89.4 & 87.8 & 91.4 & 73.2 & 96.3 & 71.3 & 93.3 & 92.9 & 95.6 & 83.2 & 74.3 & 97.3 & 86.75 \\
 &  & \textbf{0.0001} & 87.5 & 88.5 & 91.9 & 76.0 & 96.4 & 73.3 & 93.9 & 92.6 & 95.9 & 83.7 & 76.0 & 97.4 & 87.25 \\
 &  & \textbf{0.0003} & 91.1 & 88.2 & 93.2 & 77.5 & 7.6 & 72.5 & 94.3 & 93.8 & 96.3 & 83.9 & 76.3 & 97.8 & 80.58 \\
 &  & \textbf{0.001} & 87.5 & 85.9 & 93.2 & 78.5 & 7.6 & 69.1 & 93.4 & 92.3 & 95.9 & 83.5 & 74.7 & 97.5 & 79.42 \\
 &  & \textbf{0.003} & 87.6 & 73.7 & 90.2 & 22.2 & 7.6 & 63.4 & 87.9 & 88.3 & 94.7 & 73.6 & 71.7 & 19.6 & 64.50 \\
\cline{2-16}
 & \multirow[c]{7}{*}{\textbf{SGD}} & \textbf{0.01} & 90.2 & 85.0 & 90.8 & 70.3 & 7.6 & 72.0 & 92.0 & 94.3 & 93.4 & 82.9 & 68.4 & 96.2 & 78.25 \\
 &  & \textbf{0.03} & 89.5 & 87.2 & 91.1 & 72.2 & 7.6 & 72.6 & 91.8 & 94.1 & 94.8 & 83.4 & 74.2 & 97.0 & 79.25 \\
 &  & \textbf{0.1} & 89.1 & 87.5 & 90.9 & 74.3 & 7.6 & 72.3 & 92.9 & 94.7 & 95.4 & 83.6 & 75.4 & 97.2 & 79.58 \\
 &  & \textbf{0.3} & 89.7 & 87.4 & 24.5 & 75.2 & 7.6 & 71.4 & 92.4 & 93.8 & 95.9 & 83.4 & 75.1 & 97.3 & 74.00 \\
 &  & \textbf{1.0} & 88.1 & 86.1 & 24.5 & 22.2 & 7.6 & 64.0 & 0.4 & 13.4 & 2.1 & 73.6 & 74.5 & 19.6 & 39.33 \\
 &  & \textbf{3.0} & 0.4 & 1.0 & 20.0 & 22.2 & 7.4 & 2.1 & 0.3 & 2.7 & 2.2 & 73.6 & 0.5 & 6.7 & 11.25 \\
 &  & \textbf{10.0} & - & 1.0 & 20.0 & 22.2 & 7.4 & 2.1 & 0.3 & 2.7 & 2.2 & - & 0.5 & - & - \\
\cline{2-16}
 & \textbf{\DoG} & \textbf{-} & 89.9 & 87.7 & 90.7 & 75.3 & 7.6 & 71.6 & 92.4 & 93.8 & 95.5 & 83.5 & 75.0 & 97.3 & 79.50 \\
\cline{2-16}
 & \textbf{\LDoG} & \textbf{-} & 87.7 & 88.2 & 88.5 & 76.3 & 96.4 & 73.8 & 92.7 & 93.7 & 95.9 & 83.7 & 75.9 & 97.7 & 86.92 \\
\cline{1-16} \cline{2-16}
\multirow[c]{15}{*}{\textbf{Densenet121}} & \multirow[c]{7}{*}{\textbf{Adam}} & \textbf{3e-06} & - & 62.3 & 84.3 & 65.0 & 84.8 & 65.9 & 88.1 & 88.4 & 90.6 & - & 37.9 & - & - \\
 &  & \textbf{1e-05} & 86.5 {\smaller (1.24)} & 76.9 & 86.9 & 66.3 & 95.4 & 66.4 & 88.6 & 89.7 & 93.9 & 81.1 & 59.6 & 95.6 & 81.71 \\
 &  & \textbf{3e-05} & 85.2 & 82.0 & 89.0 & 69.0 & 95.9 & 66.4 {\smaller (0.76)} & 90.0 & 89.0 & 94.4 & 81.0 & 68.7 & 96.8 & 83.70 \\
 &  & \textbf{0.0001} & 86.7 {\smaller (1.40)} & 82.8 {\smaller (0.26)} & 91.4 {\smaller (0.16)} & 72.7 {\smaller (0.69)} & 96.3 {\smaller (0.07)} & 65.8 {\smaller (0.39)} & 89.4 {\smaller (1.08)} & 89.4 {\smaller (0.57)} & 94.8 {\smaller (0.16)} & 81.7 {\smaller (0.10)} & 70.8 {\smaller (0.45)} & 97.3 {\smaller (0.05)} & 84.92 \\
 &  & \textbf{0.0003} & 84.5 & 83.3 & 92.7 & 76.3 & 96.4 {\smaller (0.03)} & 65.1 & 88.9 & 89.2 & 95.1 {\smaller (0.20)} & 82.7 & 71.6 & 97.7 {\smaller (0.08)} & 84.93 \\
 &  & \textbf{0.001} & 81.8 & 80.8 & 93.9 & 76.6 {\smaller (0.39)} & 96.4 & 62.8 & 87.2 & 84.7 & 94.9 & 83.0 {\smaller (0.08)} & 69.7 & 97.7 & 83.55 \\
 &  & \textbf{0.003} & 76.0 & 76.7 & 93.7 {\smaller (0.17)} & 74.6 & 96.0 & 56.1 & 81.5 & 82.7 & 93.9 & 82.8 & 66.2 & 97.4 & 81.06 \\
\cline{2-16}
 & \multirow[c]{6}{*}{\textbf{SGD}} & \textbf{0.01} & 87.7 {\smaller (0.74)} & 83.6 & 89.5 & 68.5 & 96.1 & 65.7 & 87.4 & 90.9 & 94.2 & 81.6 & 68.9 & 96.7 & 83.73 \\
 &  & \textbf{0.03} & 87.0 & 84.0 {\smaller (0.08)} & 91.1 & 71.4 & 96.3 {\smaller (0.07)} & 66.6 {\smaller (1.59)} & 88.6 & 90.5 {\smaller (0.47)} & 94.7 & 82.0 & 71.1 & 97.1 & 84.87 \\
 &  & \textbf{0.1} & 87.9 {\smaller (0.57)} & 83.7 {\smaller (0.31)} & 92.8 {\smaller (0.26)} & 74.5 {\smaller (0.53)} & 96.4 {\smaller (0.05)} & 65.9 {\smaller (0.49)} & 88.3 {\smaller (0.82)} & 90.6 {\smaller (0.28)} & 94.9 {\smaller (0.36)} & 82.5 {\smaller (0.08)} & 71.5 {\smaller (0.21)} & 97.4 {\smaller (0.09)} & 85.53 \\
 &  & \textbf{0.3} & 85.5 & 82.5 & 92.4 {\smaller (0.56)} & 69.1 {\smaller (3.67)} & 95.9 & 62.9 & 87.6 & 87.9 & 95.0 {\smaller (0.25)} & 82.6 {\smaller (0.28)} & 70.9 & 97.2 & 83.67 \\
 &  & \textbf{1.0} & 61.6 & 65.1 & 91.4 & 61.9 & 7.6 & 35.4 & 55.4 & 64.6 & 82.6 & 80.0 & 62.0 & 95.6 & 63.17 \\
 &  & \textbf{3.0} & 50.1 & 61.6 & 88.5 & 59.2 & 7.6 & 23.4 & 44.1 & 39.8 & 85.8 & 76.5 & 45.6 & 94.7 & 55.92 \\
\cline{2-16}
 & \textbf{\DoG} & \textbf{-} & 87.4 {\smaller (0.65)} & 84.0 {\smaller (0.18)} & 91.4 {\smaller (0.19)} & 71.9 {\smaller (0.43)} & 96.2 {\smaller (0.04)} & 66.1 {\smaller (0.90)} & 88.5 {\smaller (0.92)} & 90.3 {\smaller (0.40)} & 94.8 {\smaller (0.12)} & 82.0 {\smaller (0.08)} & 71.6 {\smaller (0.22)} & 97.2 {\smaller (0.13)} & 85.12 \\
\cline{2-16}
 & \textbf{\LDoG} & \textbf{-} & 86.9 {\smaller (0.27)} & 83.5 {\smaller (0.18)} & 89.8 {\smaller (0.31)} & 71.5 {\smaller (0.24)} & 96.1 {\smaller (0.03)} & 66.4 {\smaller (0.61)} & 88.7 {\smaller (0.70)} & 90.6 {\smaller (0.13)} & 95.3 {\smaller (0.17)} & 81.9 {\smaller (0.10)} & 71.4 {\smaller (0.25)} & 97.5 {\smaller (0.09)} & 84.97 \\
\cline{1-16} \cline{2-16}
\multirow[c]{15}{*}{\textbf{ResNet50}} & \multirow[c]{5}{*}{\textbf{Adam}} & \textbf{0.0003} & 87.8 {\smaller (1.52)} & 84.8 & 91.0 & 73.3 & 96.2 & 68.5 {\smaller (0.97)} & 92.7 & 93.1 {\smaller (0.27)} & 95.5 & 82.2 & 73.9 & 97.0 & 86.03 \\
 &  & \textbf{0.001} & 88.3 {\smaller (1.18)} & 83.9 {\smaller (0.31)} & 92.4 {\smaller (0.21)} & 76.5 {\smaller (0.53)} & 96.3 {\smaller (0.04)} & 67.2 {\smaller (0.94)} & 89.2 {\smaller (1.83)} & 92.0 {\smaller (0.37)} & 95.5 {\smaller (0.25)} & 83.0 {\smaller (0.18)} & 73.7 {\smaller (0.33)} & 97.6 {\smaller (0.07)} & 86.30 \\
 &  & \textbf{0.003} & 86.9 & 83.1 & 93.3 {\smaller (0.37)} & 78.3 {\smaller (0.30)} & 96.1 & 64.4 & 90.1 & 90.2 & 95.4 & 83.4 {\smaller (0.15)} & 72.2 & 97.7 & 85.67 \\
 &  & \textbf{0.01} & 79.9 & 75.9 & 93.5 & 75.0 & 95.9 & 58.0 & 85.1 & 85.0 & 94.4 & 83.0 & 66.5 & 97.4 & 82.08 \\
 &  & \textbf{0.03} & 62.8 & 64.9 & 92.4 & 71.3 & 95.3 & 46.9 & 63.1 & 67.0 & 89.3 & 82.0 & 50.8 & 96.4 & 73.08 \\
\cline{2-16}
 & \multirow[c]{8}{*}{\textbf{SGD}} & \textbf{0.01} & 86.7 & 62.9 & 83.7 & 52.5 & 68.3 & 66.2 & 80.9 & 91.9 & 84.0 & 75.9 & 48.0 & 80.2 & 72.92 \\
 &  & \textbf{0.03} & 86.7 & 77.0 & 88.0 & 63.0 & 88.5 & 67.2 & 82.1 & 92.8 & 90.5 & 78.8 & 64.7 & 91.4 & 80.50 \\
 &  & \textbf{0.1} & 87.6 {\smaller (0.66)} & 82.8 & 90.2 & 67.0 & 95.5 & 67.5 {\smaller (1.71)} & 89.2 & 93.5 & 93.8 & 81.7 & 71.6 & 95.0 & 84.26 \\
 &  & \textbf{0.3} & 87.4 & 84.8 & 91.1 & 70.7 & 95.9 & 70.5 & 85.8 {\smaller (2.54)} & 92.9 & 94.7 & 82.3 & 73.9 & 96.1 & 84.98 \\
 &  & \textbf{1.0} & 86.6 {\smaller (0.50)} & 84.5 {\smaller (0.46)} & 90.1 {\smaller (0.51)} & 72.6 {\smaller (1.56)} & 96.0 {\smaller (0.08)} & 66.4 {\smaller (1.16)} & 85.5 {\smaller (2.39)} & 93.0 {\smaller (0.30)} & 94.6 {\smaller (0.28)} & 82.6 {\smaller (0.09)} & 73.6 {\smaller (0.44)} & 96.8 {\smaller (0.06)} & 85.19 \\
 &  & \textbf{3.0} & 87.4 & 84.7 & 91.8 & 70.1 & 96.0 & 66.9 & 90.3 & 92.1 & 94.8 {\smaller (0.49)} & 82.1 & 73.7 & 97.0 {\smaller (0.10)} & 85.23 \\
 &  & \textbf{10.0} & 86.2 & 81.2 & 91.4 {\smaller (0.71)} & 67.2 & 7.6 & 59.4 & 86.6 & 82.2 & 94.6 & 78.2 & 70.0 & 96.9 & 74.78 \\
 &  & \textbf{30.0} & 0.4 & 12.1 & 20.0 & 22.2 & 7.4 & 24.1 & 43.6 & 16.5 & 83.6 & 73.6 & 3.1 & 6.7 & 25.75 \\
\cline{2-16}
 & \textbf{\DoG} & \textbf{-} & 86.8 {\smaller (0.62)} & 84.8 {\smaller (0.37)} & 89.3 {\smaller (0.58)} & 71.4 {\smaller (0.77)} & 95.7 {\smaller (0.09)} & 66.4 {\smaller (1.48)} & 85.8 {\smaller (2.64)} & 92.9 {\smaller (0.38)} & 94.6 {\smaller (0.33)} & 82.6 {\smaller (0.16)} & 73.5 {\smaller (0.41)} & 96.5 {\smaller (0.10)} & 85.02 \\
\cline{2-16}
 & \textbf{\LDoG} & \textbf{-} & 87.6 {\smaller (1.15)} & 84.6 {\smaller (0.38)} & 90.8 {\smaller (0.38)} & 75.0 {\smaller (0.44)} & 96.0 {\smaller (0.05)} & 69.1 {\smaller (1.30)} & 90.2 {\smaller (2.02)} & 92.4 {\smaller (0.36)} & 95.6 {\smaller (0.46)} & 82.2 {\smaller (0.15)} & 74.6 {\smaller (0.32)} & 97.4 {\smaller (0.08)} & 86.29 \\
\cline{1-16} \cline{2-16}
\multirow[c]{17}{*}{\textbf{VGG11}} & \multirow[c]{8}{*}{\textbf{Adam}} & \textbf{3e-06} & 80.2 {\smaller (0.71)} & - & - & - & - & - & - & - & - & 79.8 {\smaller (0.07)} & - & 93.8 {\smaller (0.15)} & - \\
 &  & \textbf{1e-05} & 80.5 & 73.7 & 89.3 & 63.5 & 94.3 & 61.5 & 82.1 & 87.3 & 91.4 & 80.0 & 65.4 & 95.3 & 80.00 \\
 &  & \textbf{3e-05} & 82.2 {\smaller (0.48)} & 74.9 & 90.5 & 67.7 & 96.1 {\smaller (0.06)} & 61.4 {\smaller (0.84)} & 84.0 & 88.1 {\smaller (0.17)} & 92.9 & 81.1 & 66.6 & 96.4 & 81.48 \\
 &  & \textbf{0.0001} & 82.6 & 75.6 {\smaller (0.23)} & 92.4 & 71.9 & 7.6 & 61.6 & 83.5 {\smaller (1.06)} & 87.2 & 93.5 {\smaller (0.23)} & 82.3 & 66.9 {\smaller (0.12)} & 96.8 & 74.79 \\
 &  & \textbf{0.0003} & 82.3 & 74.1 & 93.2 {\smaller (0.25)} & 74.4 {\smaller (0.47)} & 7.5 & 59.6 & 83.0 & 86.0 & 93.4 & 82.9 {\smaller (0.08)} & 66.3 & 96.8 {\smaller (0.14)} & 74.77 \\
 &  & \textbf{0.001} & 61.7 & 61.1 & 24.5 & 64.6 & 7.5 & 48.0 & 53.7 & 65.0 & 89.1 & 73.6 & 50.5 & 96.5 & 57.58 \\
 &  & \textbf{0.003} & - & 1.0 & 89.4 & 68.4 & 7.6 & 2.1 & 0.5 & 2.7 & 76.2 & - & 30.5 & - & - \\
 &  & \textbf{0.01} & - & 1.0 & 24.5 & 22.2 & 7.6 & 2.1 & 1.2 & 2.7 & 2.2 & - & 2.0 & - & - \\
\cline{2-16}
 & \multirow[c]{7}{*}{\textbf{SGD}} & \textbf{0.001} & 81.0 & 68.5 & 85.0 & 62.3 & 14.8 {\smaller (3.59)} & 61.3 & 80.3 & 88.0 & 89.3 & 78.9 & 61.9 & 92.0 & 71.65 \\
 &  & \textbf{0.003} & 81.8 & 72.4 & 90.3 & 64.2 & 12.3 & 62.2 & 80.9 & 88.0 & 90.9 & 80.3 & 64.9 & 94.4 & 73.08 \\
 &  & \textbf{0.01} & 82.4 & 73.5 & 91.9 {\smaller (0.19)} & 67.0 & 9.8 & 61.2 & 82.0 {\smaller (0.99)} & 89.0 {\smaller (0.37)} & 91.7 & 81.7 & 65.8 & 95.7 & 73.91 \\
 &  & \textbf{0.03} & 83.0 {\smaller (0.50)} & 74.7 {\smaller (0.21)} & 92.1 {\smaller (0.19)} & 69.1 {\smaller (0.36)} & 7.6 {\smaller (0.04)} & 61.9 {\smaller (0.41)} & 82.3 {\smaller (1.09)} & 89.4 {\smaller (0.32)} & 92.5 {\smaller (0.28)} & 82.2 {\smaller (0.06)} & 66.1 {\smaller (0.10)} & 96.4 {\smaller (0.08)} & 74.77 \\
 &  & \textbf{0.1} & 49.6 {\smaller (44.91)} & 74.6 {\smaller (0.08)} & 20.0 & 22.2 & 7.6 & 60.5 & 0.3 & 87.5 & 2.2 & 73.6 & 53.2 {\smaller (29.49)} & 6.7 & 37.87 \\
 &  & \textbf{0.3} & - & 1.0 & 20.0 & 22.2 & 7.4 & 2.1 & 0.3 & 2.7 & 2.2 & - & 0.5 & - & - \\
 &  & \textbf{1.0} & - & 1.0 & 20.0 & 22.2 & 7.4 & 2.1 & 0.3 & 2.7 & 2.2 & - & 0.5 & - & - \\
\cline{2-16}
 & \textbf{\DoG} & \textbf{-} & 82.9 {\smaller (0.45)} & 74.7 {\smaller (0.22)} & 91.5 {\smaller (0.17)} & 68.4 {\smaller (0.53)} & 10.4 {\smaller (0.82)} & 62.5 {\smaller (1.19)} & 82.8 {\smaller (0.95)} & 89.5 {\smaller (0.23)} & 92.4 {\smaller (0.22)} & 81.6 {\smaller (0.07)} & 66.3 {\smaller (0.33)} & 96.3 {\smaller (0.09)} & 74.94 \\
\cline{2-16}
 & \textbf{\LDoG} & \textbf{-} & 82.4 {\smaller (0.60)} & 74.8 {\smaller (0.16)} & 92.0 {\smaller (0.11)} & 69.9 {\smaller (0.44)} & 92.1 {\smaller (8.59)} & 62.6 {\smaller (0.95)} & 82.7 {\smaller (1.15)} & 88.9 {\smaller (0.62)} & 92.4 {\smaller (0.15)} & 81.9 {\smaller (0.18)} & 65.8 {\smaller (0.09)} & 96.5 {\smaller (0.12)} & 81.83 \\
\cline{1-16} \cline{2-16}
\multirow[c]{18}{*}{\textbf{ViT-B/32}} & \multirow[c]{7}{*}{\textbf{Adam}} & \textbf{3e-06} & 90.7 {\smaller (0.76)} & 92.3 {\smaller (0.22)} & 89.5 {\smaller (0.35)} & 66.0 {\smaller (0.69)} & 94.0 {\smaller (0.24)} & 74.9 {\smaller (0.24)} & 98.5 {\smaller (0.53)} & 91.6 {\smaller (0.11)} & 95.5 {\smaller (0.07)} & 79.7 {\smaller (0.08)} & 75.5 {\smaller (0.23)} & 96.8 {\smaller (0.06)} & 87.08 \\
 &  & \textbf{1e-05} & 90.2 {\smaller (0.47)} & 92.8 {\smaller (0.21)} & 89.9 {\smaller (0.52)} & 67.5 & 51.5 {\smaller (48.10)} & 76.9 & 98.7 {\smaller (0.29)} & 90.7 & 96.3 & 79.8 & 78.0 {\smaller (0.12)} & 97.6 & 83.84 \\
 &  & \textbf{3e-05} & 88.6 & 92.5 & 89.7 & 70.1 & 7.6 & 75.3 {\smaller (0.24)} & 98.7 & 91.6 {\smaller (0.21)} & 96.5 {\smaller (0.13)} & 80.1 {\smaller (0.09)} & 78.3 & 97.7 & 80.21 \\
 &  & \textbf{0.0001} & 89.5 & 91.2 & 89.1 & 70.8 {\smaller (0.32)} & 7.6 & 72.9 & 98.1 & 90.0 & 96.1 & 80.1 & 77.0 & 97.8 {\smaller (0.07)} & 79.80 \\
 &  & \textbf{0.0003} & 87.9 & 87.3 & 87.9 & 68.5 & 7.6 & 69.0 & 94.9 & 87.9 & 94.9 & 79.3 & 72.5 & 97.9 & 77.33 \\
 &  & \textbf{0.001} & 80.3 & 62.1 & 88.1 & 51.2 & 7.6 & 50.4 & 71.0 & 58.3 & 88.6 & 73.6 & 53.6 & 96.3 & 64.75 \\
 &  & \textbf{0.003} & - & 13.5 & 64.6 & 29.3 & 7.6 & 14.1 & 26.7 & 12.1 & 71.9 & - & 19.5 & - & - \\
\cline{2-16}
 & \multirow[c]{9}{*}{\textbf{SGD}} & \textbf{3e-05} & - & 6.1 & 52.7 & 40.8 & 32.7 & 45.9 & 61.4 & 80.3 & 59.7 & - & 5.2 & - & - \\
 &  & \textbf{0.0001} & 85.0 & 73.7 & 71.3 & 50.4 & 57.0 & 69.0 & 98.1 & 90.1 & 83.0 & 75.3 & 24.7 & 79.1 & 71.17 \\
 &  & \textbf{0.0003} & 88.7 & 88.7 & 83.1 & 60.3 & 69.6 & 74.7 & 98.8 & 91.9 & 90.2 & 76.6 & 59.6 & 90.7 & 80.50 \\
 &  & \textbf{0.001} & 90.8 & 91.7 & 88.1 & 65.6 & 87.5 & 74.8 & 98.7 {\smaller (0.51)} & 93.0 {\smaller (0.21)} & 93.5 & 78.2 & 73.4 & 95.2 & 85.48 \\
 &  & \textbf{0.003} & 90.9 {\smaller (0.89)} & 92.8 {\smaller (0.10)} & 87.3 {\smaller (1.28)} & 66.3 {\smaller (0.40)} & 85.6 {\smaller (15.77)} & 75.3 {\smaller (0.29)} & 98.9 {\smaller (0.31)} & 92.5 {\smaller (0.27)} & 95.3 {\smaller (0.10)} & 79.4 {\smaller (0.02)} & 77.3 {\smaller (0.16)} & 96.6 {\smaller (0.17)} & 86.52 \\
 &  & \textbf{0.01} & 90.7 {\smaller (0.63)} & 92.9 {\smaller (0.14)} & 85.9 & 68.8 & 56.1 & 75.2 & 99.3 & 92.5 & 95.8 & 79.7 & 78.9 {\smaller (0.08)} & 97.4 & 84.04 \\
 &  & \textbf{0.03} & 89.8 & 92.5 & 83.1 & 69.7 {\smaller (0.29)} & 65.8 & 75.8 {\smaller (0.59)} & 99.3 & 92.2 & 96.2 {\smaller (0.06)} & 78.9 {\smaller (2.51)} & 78.3 & 97.7 {\smaller (0.05)} & 84.69 \\
 &  & \textbf{0.1} & 88.0 & 91.2 & 25.2 & 22.2 & 7.6 & 74.9 & 98.7 & 91.0 & 95.9 & 73.6 & 76.8 & 97.8 & 69.75 \\
 &  & \textbf{0.3} & 0.4 & 1.0 & 24.5 & 22.2 & 7.6 & 70.5 & 1.8 & 2.7 & 2.3 & 73.6 & 0.5 & 19.6 & 18.42 \\
\cline{2-16}
 & \textbf{\DoG} & \textbf{-} & 89.5 {\smaller (1.26)} & 92.5 {\smaller (0.22)} & 85.0 {\smaller (0.27)} & 69.5 {\smaller (0.11)} & 67.7 {\smaller (36.66)} & 75.5 {\smaller (0.71)} & 98.9 {\smaller (0.25)} & 92.4 {\smaller (0.16)} & 96.4 {\smaller (0.10)} & 79.7 {\smaller (0.01)} & 77.8 {\smaller (0.13)} & 97.7 {\smaller (0.08)} & 85.22 \\
\cline{2-16}
 & \textbf{\LDoG} & \textbf{-} & 89.6 {\smaller (0.81)} & 92.8 {\smaller (0.15)} & 86.0 {\smaller (0.40)} & 70.7 {\smaller (0.29)} & 95.3 {\smaller (0.08)} & 75.8 {\smaller (0.71)} & 99.0 {\smaller (0.26)} & 92.3 {\smaller (0.45)} & 96.5 {\smaller (0.17)} & 79.8 {\smaller (0.07)} & 78.3 {\smaller (0.26)} & 97.8 {\smaller (0.04)} & 87.82 \\

 \bottomrule
    \end{tabular}}
    \arxiv{\caption{\VisionTableCaption}}
    \end{table} 

\begin{landscape}
\begin{figure}[t]
	\begin{center}
	\centerline{
	\notarxiv{\resizebox{\linewidth}{!}{\includegraphics[width=\textwidth,height=\textheight,keepaspectratio]{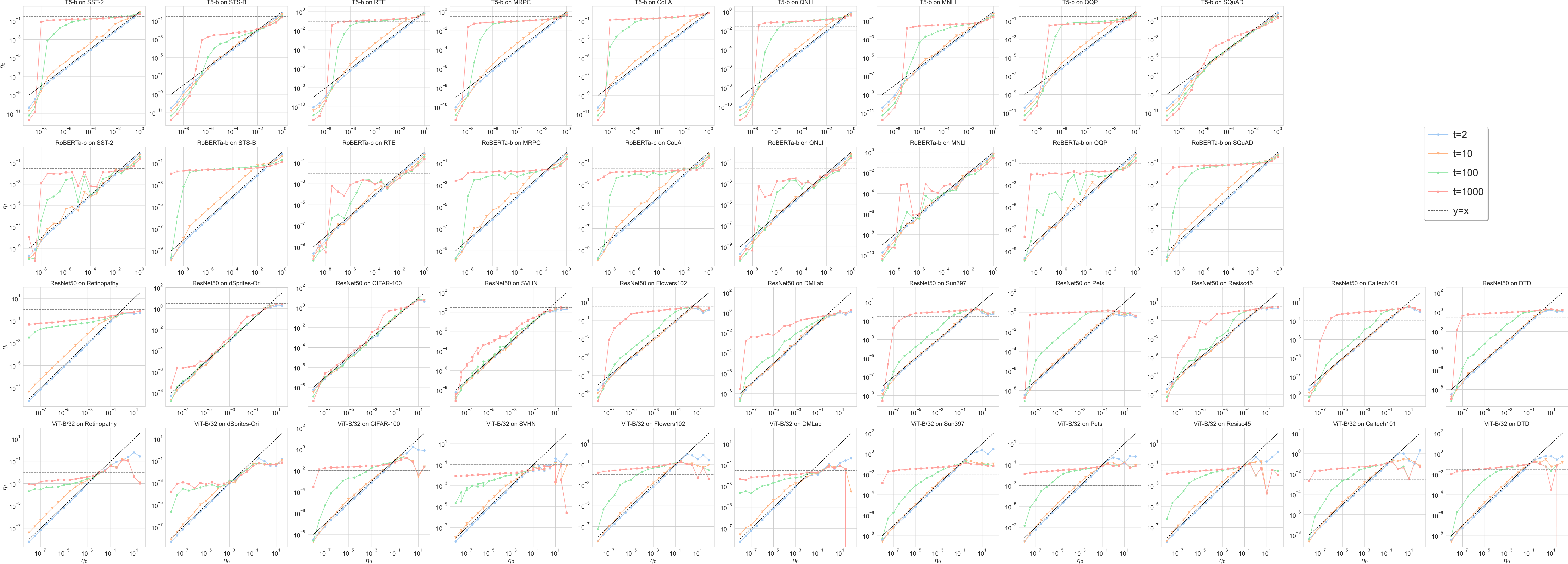}}}
	\arxiv{\resizebox{\linewidth}{!}{\includegraphics[height=0.9\textheight,keepaspectratio]{figures/eta_sensitivity.pdf}}}
	}
	\caption{Stabilizing behavior of \DoG{} on $\eta_t$ as a function of $\eta_0$ ($x$-axis) and $t$ (color). In most cases $\eta_t$ quickly stabilizes around a value close to the optimal SGD base learning rate (dashed horizontal line) for all sufficiently small $\eta_0=\reps/\norm{g_0}$. The main exceptions (where $\eta_t$ depends strongly on $\eta_0$) are dSprites-Ori, CIFAR-100 and SVHN when trained with ResNet50; see \ref{app:experiments-reps} for further discussion.
	}
	\label{fig:eta-sensitivity}
	\end{center}
\end{figure}
\end{landscape}

\subsection{Comparison with equalized compute budget}\label{app:expeirments-equalized}
Throughout the paper, our experiments focus on comparing different methods by the final test performance they are able to reach given sufficient compute budget to essentially fully converge. As a consequence, SGD and Adam---which require learning rate tuning---use up significantly more computation than \DoG and \LDoG. More specifically, for each model/task combination we tune the SGD and Adam learning rates over a grid of at least 5 values (and often 6 or more), resulting in computational cost increased by the same factor.

In this subsection only, we compare different optimizers using roughly the same computational budget, by measuring the performance of Adam and SGD after roughly 20\% of their overall step budget.\footnote{Since these results are just a re-analysis of our original experiments, for language experiments we take all the warmup iterates plus the first 20\% of the remaining iterates, overall using 28\% of the budget.} \Cref{fig:bisection} shows the result of this comparison, contrasting it to our main experiment. The figure shows that \DoG often exceeds the performance of instance-tuned SGD with equalized step budget.

We note a number of caveats regarding the equalized compute budget comparison:
\begin{enumerate}
	\item Since our experiments are focused on getting the best possible generalization, we substantially over-provisioned the iteration budget, and hence the performance of instance-tuned SGD and Adam declines only mildly when we cut the budget by roughly 5. Our tightest budget was for RoBERTa (150\% the iterations in \citet{liu2019roberta}), and there we can see that performance degraded more substantially. If we were to instead take the number of iterations \DoG actually needs to reach its peak performance, its advantage over equalized-compute SGD would likely be far larger.
	\item Since the results reported here are obtained by re-analysis of our original experiments, the cosine learning rate schedule for SGD and Adam is not properly set for using 20\% of the iteration budget; in particular, the learning rate does not decay to zero at the end of the training. Running these methods with a properly set cosine schedule would likely improve their performance. However, we note that the addition of iterate averaging appears to partially compensate for the lack of sharp learning rate decay. 
	\item Given sufficient resources, it is possible to run all the different learning rates of SGD and Adam in parallel. Therefore, the comparison equalizes overall computational cost, but not necessarily wall time.
	\item The comparison also does not take into account more sophisticated learning rate tuning schemes that early-stop unpromising learning rate candidates. However, such schemes run the risk of choosing a suboptimal learning rate.
\end{enumerate}

\begin{figure}[t]
	\begin{center}
		\centerline{
			\subfloat[SGD/Adam have $>5$x the \DoG/\LDoG budget.]{
			\includegraphics[width=0.47\linewidth]{figures/hpt_results.pdf}
		}
		\subfloat[SGD/Adam have roughly the same budget as \DoG/\LDoG]{
			\includegraphics[width=0.47\linewidth]{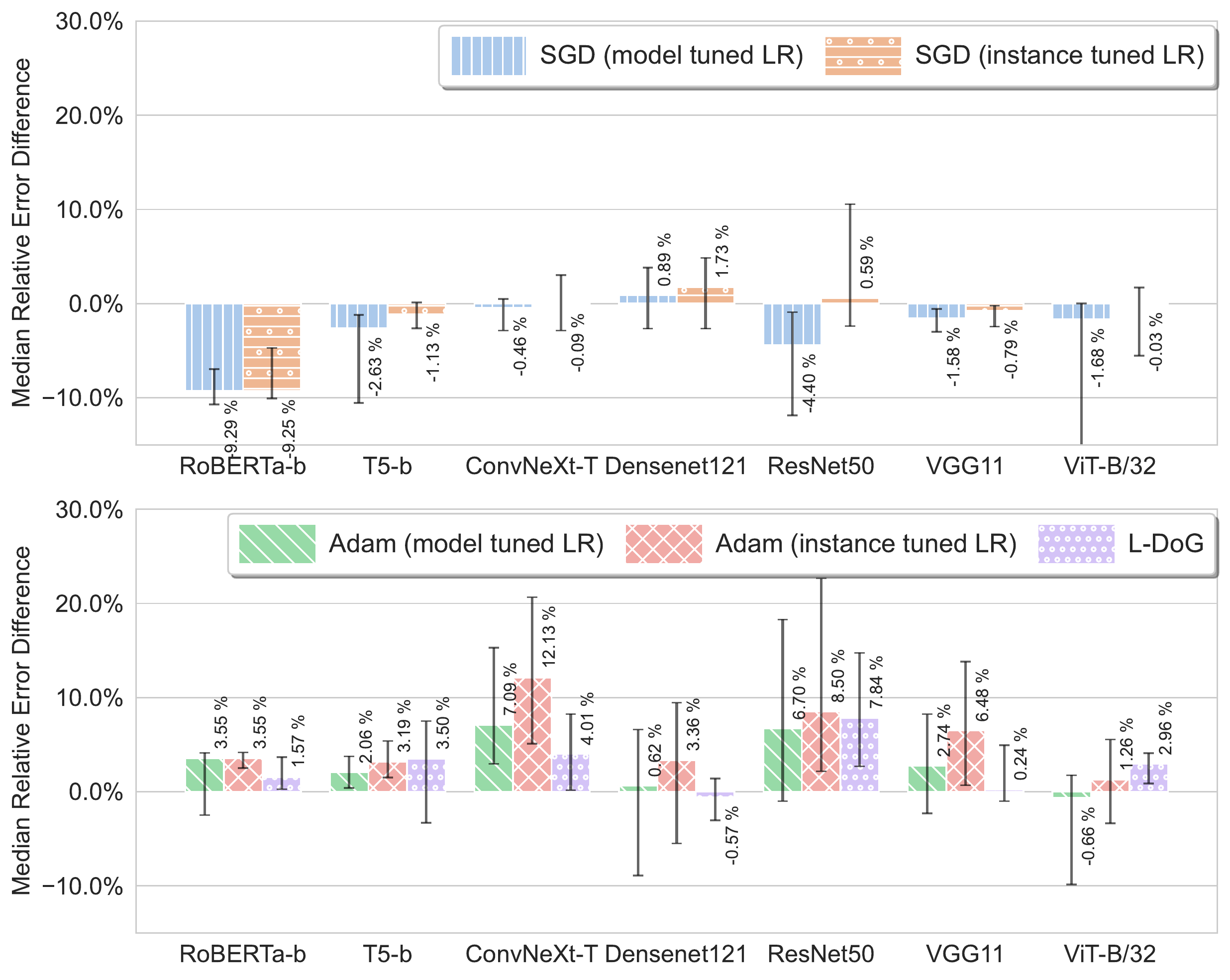}
		}
		}
		\caption{RED median (bar chart) and IQR (error bars) of each model on the set of applicable tasks, where we either \textbf{(a)} give the same iteration budget for each optimizer run, resulting in SGD and Adam using more than 5x total compute than \DoG and \LDoG due to learning rate tuning (this is a reproduction of \Cref{fig:hpt-results}), or \textbf{(b)} give each algorithm \emph{roughly equal compute budget } by running SGD and Adam (with 5 or more learning rates) for roughly 20\% of the steps that \DoG and \LDoG use. With equalized compute budget, \DoG outperforms model-tuned SGD almost always and often outperforms the instance-tuned SGD as well, while 
			\LDoG{} closes most of the gap to Adam.
		}
		\label{fig:bisection}
	\end{center}
\end{figure}

\subsection{Fine-tuning CoLA}\label{app:experiments-cola}
As discussed in \Cref{subsec:results}, 
\DoG with $\reps = 10^{-6}(1+\norm{x_0})$ failed in fine-tuning T5-b on CoLA.
To investigate this issue, we ran \DoG and \LDoG with different choices of $\reps$.
\Cref{fig:cola} depicts the results of this experiment as well as the performance of SGD and Adam with different learning rates. The figure shows that using lower values of $\reps$ allows \DoG to reach reasonable results, but with some seeds still failing. In contrast, \LDoG shows consistent and superior performance across a large range of $\reps$ values. We leave further investigations on the cause of failure in CoLA to future work.

\subsection{Sensitivity of \DoG{} to $\reps$ and the effect of batch normalization}\label{app:experiments-reps}
In \Cref{subsec:DoG-sensitivity}, we discuss \DoG's insensitivity to the choice of $\reps$ 
as long as it is small enough. Here, we expand on this analysis by testing how the \DoG step size at iteration $t$, denoted $\eta_t$, depends on its initial step size $\eta_0=\reps/\norm{g_0}$. 
For each task in our testbed and for 4 models, 
we perform short training runs with a large number of $\eta_0$ values. 
In \Cref{fig:eta-sensitivity} we plot 
$\eta_t$ vs.\ $\eta_0$ for $t\in\{2,10,100,1000\}$. We also show a horizontal line for the learning rate of of SGD reaching the best validation accuracy, and the $y=x$ diagonal. The figure shows that for most model/task combinations, $\eta_t$ converges quickly (within the first $1000$ steps) to a value near the optimal one for SGD, and mostly independent of $\eta_0$ as long as it is small enough.

However, we also observe some failure cases where $\eta_t$ strongly depends on $\eta_0$, such as fine-tuning ResNet50 on CIFAR-100. This provides a complementary perspective on the fact \DoG is sensitive to $\reps$ in this setting, as already shown in \Cref{fig:hpt-results}: when $\eta_0$ is to low, \DoG fails to reach a suitable value of $\eta_t$ in a reasonable time. We hypothesize that this is due to the batch normalization (BN) layers in the model causing many different step size to ``look'' like solutions to the implicit equation motivating \DoG. To test this hypothesis, we repeat the CIFAR-100 training experiment but without BN (we disable BN by fine-tuning the model in evaluation mode). 
\Cref{fig:no-bn}(a) shows that removing BN allows \DoG{} to recover its stabilizing behavior. Moreover, \Cref{fig:no-bn}(b) further shows that without batch normalization, the performance of \DoG{} again becomes insensitive to the choice of $\reps$ provided it is sufficiently small. Unsurprisingly, we also observe that removing BN slightly hurts generalization performance in this task. 
As mentioned in \Cref{sec:discussion}, improving \DoG to be more robust in the presence of normalization layers in general and batch normalization in particular is an important direction for future research.

\begin{figure}[t]
	\begin{center}
	\centerline{
	\subfloat[]{{\includegraphics[height=5.25cm,keepaspectratio]{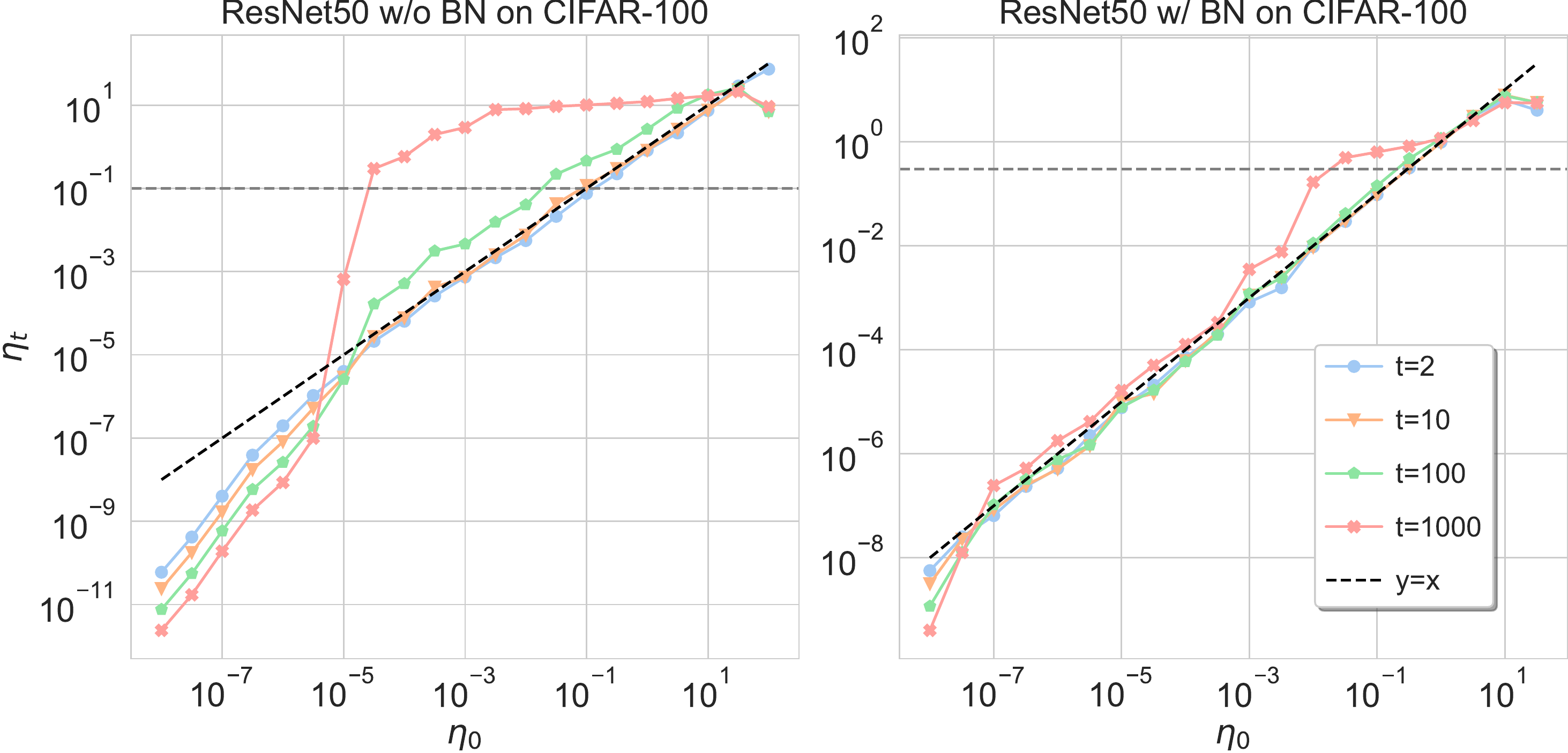}}}%
    ~~~
    \subfloat[]{{\includegraphics[height=5.25cm,keepaspectratio]{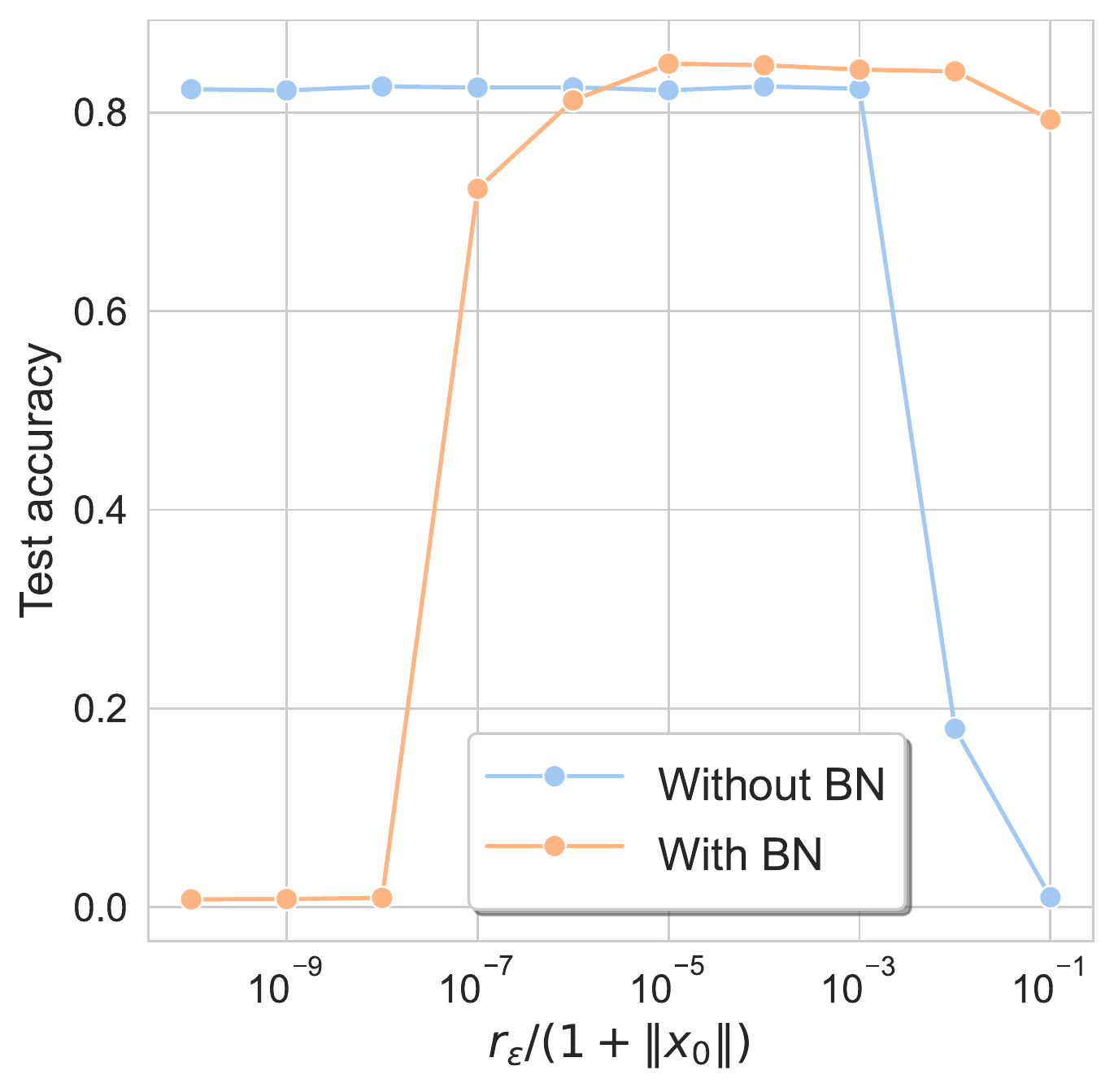}}}%
	}
	\caption{ResNet50 fine-tuned on CIFAR-100 with and without batch normalization. The dashed horizontal line indicates the best SGD learning rate.
	\textbf{(a)} Stabilizing behavior of \DoG{} on $\eta_t$ as a function of $\eta_0$ ($x$-axis) and $t$ (color). Turning off batch normalization (left) mitigates the sensitivity of $\eta_t$ to $\eta_0$ observed in batch normalized model (right). 
	\textbf{(b)} Accuracies of models trained with \DoG (for 20K steps) as a function of $\reps$. Without batch normalization, \DoG{} is robust to smaller values of $\reps$.  
	}
	 
	\label{fig:no-bn}
	\end{center}
\end{figure}

\notarxiv{
\subsection{Additional convex optimization results}\label{app:experiments-convex}
\begin{figure}[t]
	\begin{center}
		\centerline{
				\includegraphics[width=0.6\linewidth]{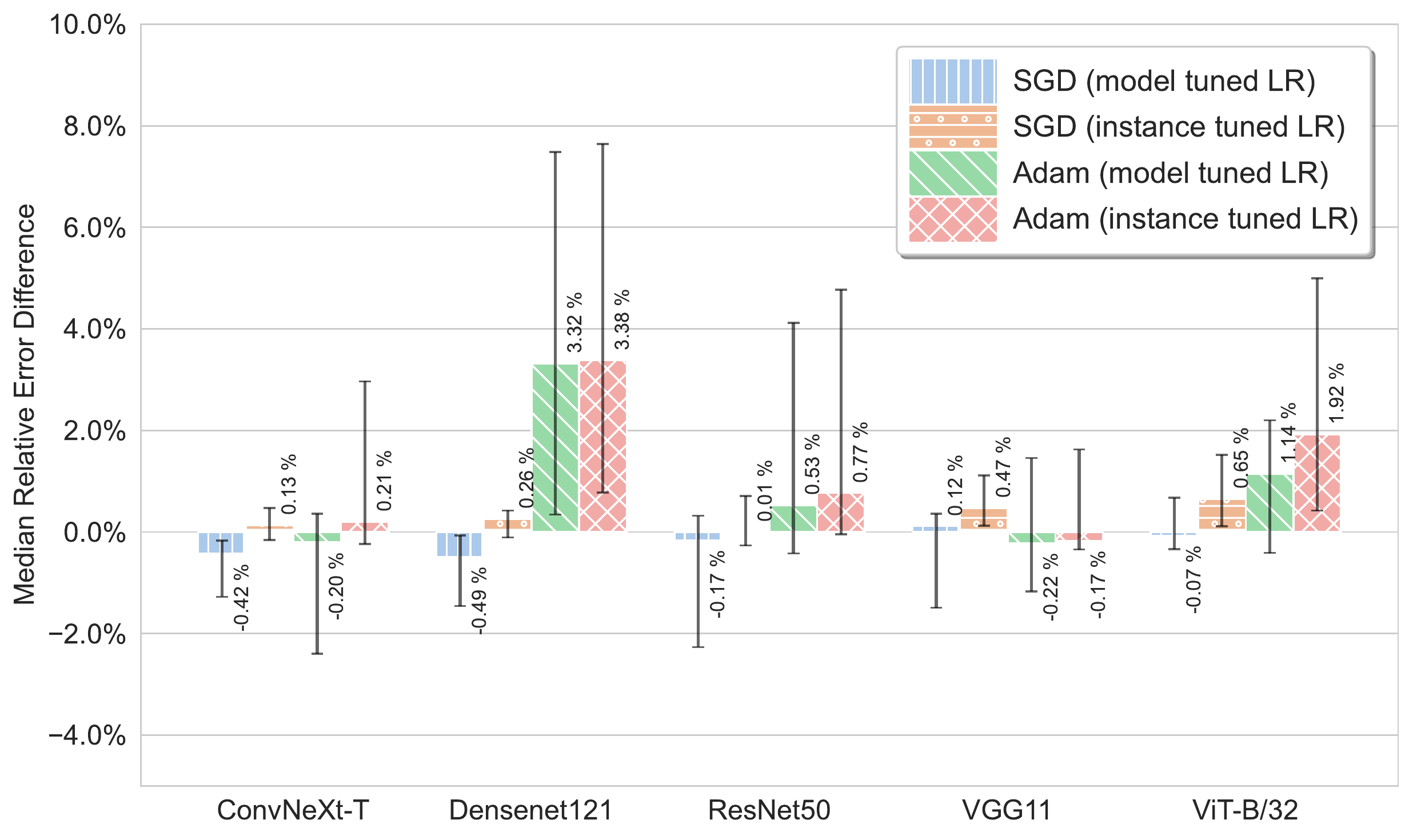}

		}
		\caption{RED median and IQR (as in \Cref{fig:hpt-results}) in tn the \emph{convex optimization} setting (\Cref{subsec:convex-opt}).
		}
		\label{fig:hpt-results-convex}
	\end{center}
\end{figure}

\newcommand{\MeanResultsConvexCaption}{Per-learning rate RED statistics (as in \Cref{fig:mean-results}) in the \emph{convex optimization} setting (\Cref{subsec:convex-opt}).}

	\begin{figure}[t]
		\begin{center}
			\centerline{
				\includegraphics[width=\textwidth,height=\textheight,keepaspectratio]{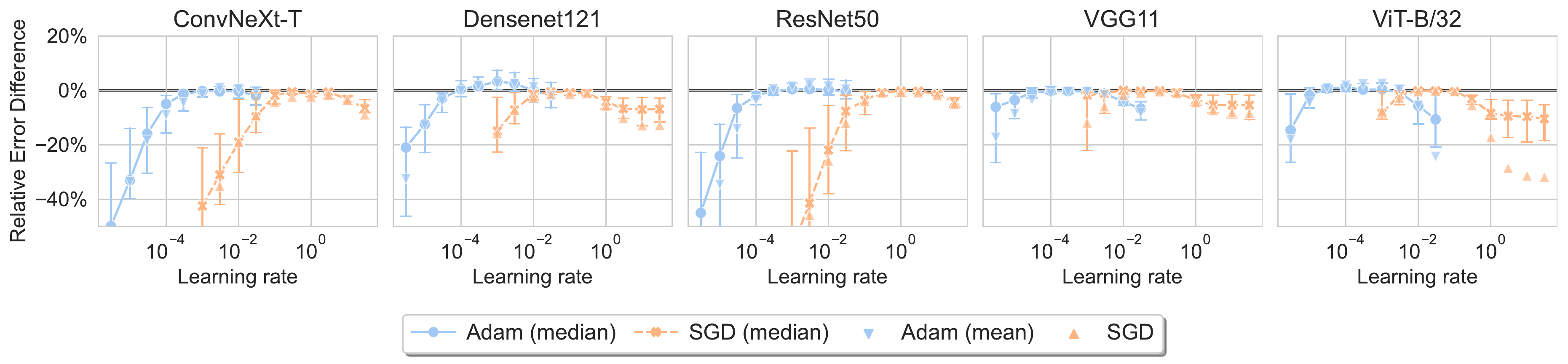}
			}
			\caption{\MeanResultsConvexCaption}
			
			\label{fig:mean-results-convex}
		\end{center}
	\end{figure}
 \Cref{fig:mean-results-convex,fig:hpt-results-convex} show results for learning lines probes on out computer vision fine-tuning testbed (see \Cref{subsec:convex-opt}). The figures show that \DoG attains results on par with instance-tuned SGD and Adam.

}

\begin{figure}[t]
	\begin{center}
	\centerline{
	\includegraphics[width=0.6\textwidth,keepaspectratio]{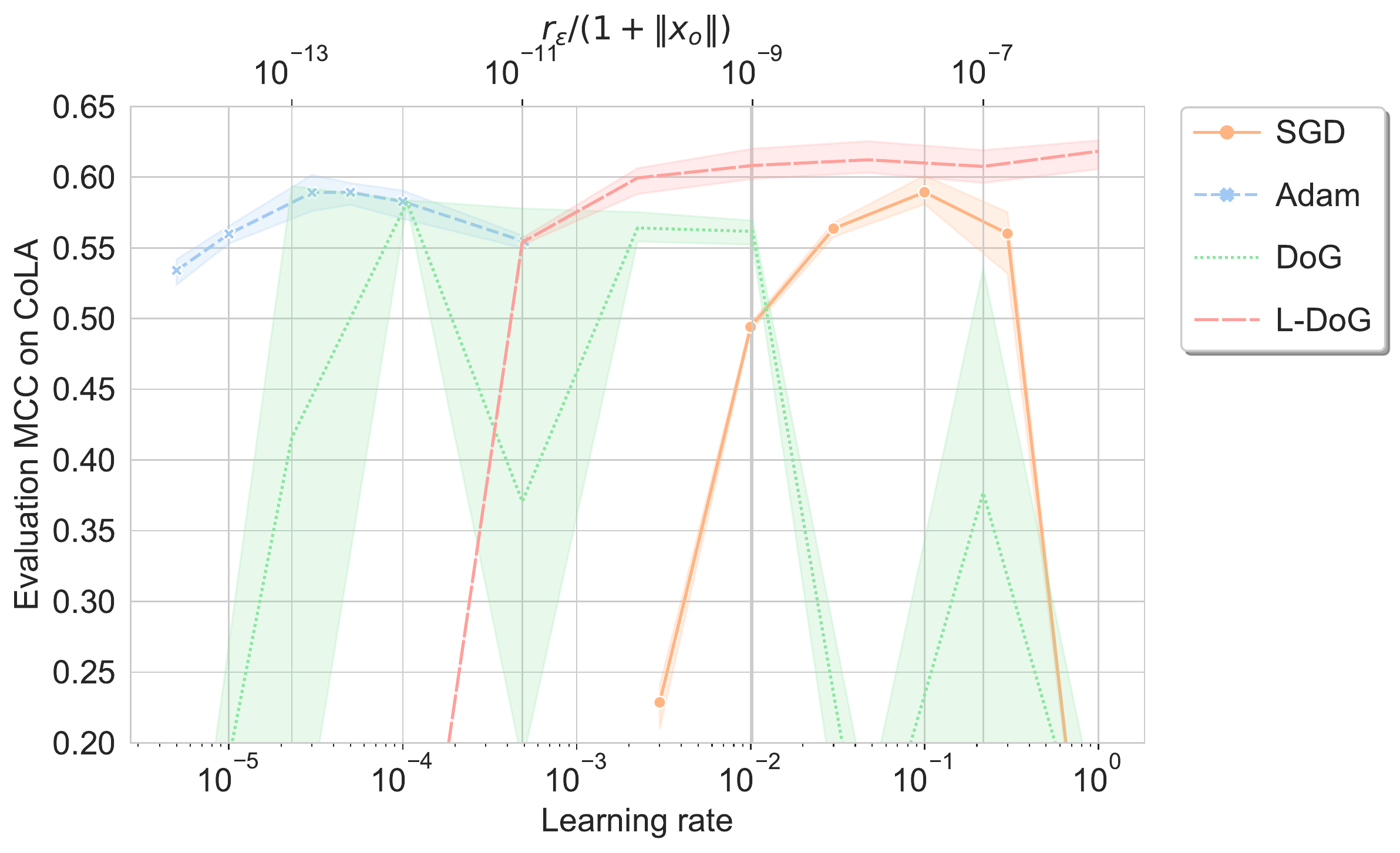}
	}
	\caption{Matthews correlation of T5-base fine-tuned on CoLA with SGD and Adam with different base learning rates (bottom axis), as well as with \DoG{} and \LDoG with different $\reps$ (top axis). Only  \LDoG{} and Adam perform consistently well across different parameters. The lines and shaded regions show the average Matthews correlation and the min-max range, respectively, computed over 3 seeds.
	}
	 
	\label{fig:cola}
	\end{center}
\end{figure}

\subsection{The growth rate of $\rbar[t]$}
\Cref{fig:rbar} plots $\rbar[t]$ for \DoG as a function of the iteration index $t$. As the figure shows, $\rbar[t]$ grows very rapidly and then approximately plateaus. Therefore, the quantity $\sum_{i\le t} \frac{\rbar[i]}{\rbar[t]}$ grows roughly linearly in $t$, implying a near-optimal rate of convergence for \DoG, as discussed in \Cref{sec:error-bound-assuming-bounded-iterates}.

\begin{figure}[t]
	\begin{center}
	\centerline{
	\includegraphics[width=0.9\textwidth,height=\textheight,keepaspectratio]{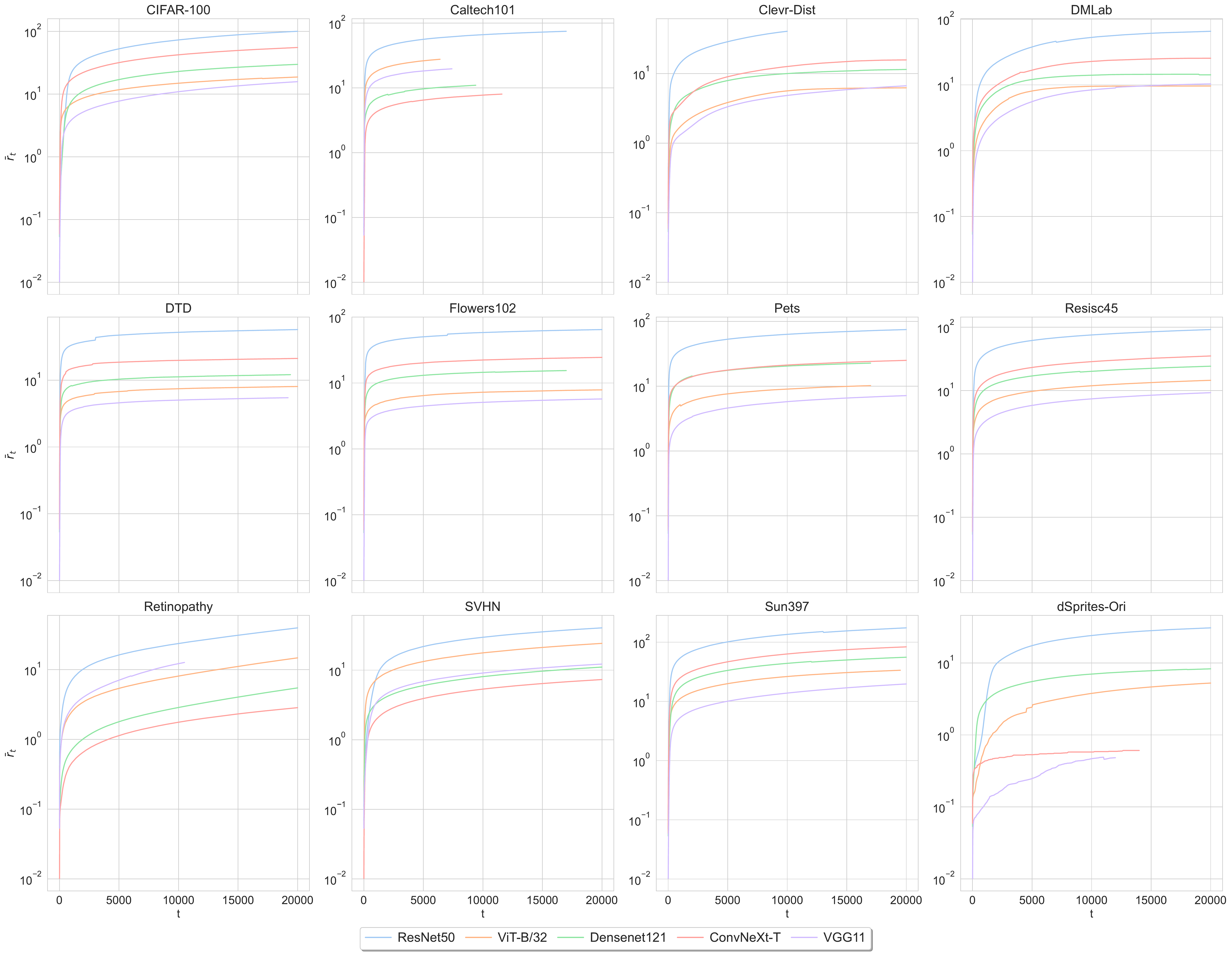}
	}
	\caption{The quantity $\rbar[t]=\max_{i \le t}{\lVert x_i - x_0 \rVert}$ as a function of the number of steps $t$ in our computer vision testbed. The value of $\rbar[t]$ grows rapidly at first and then almost plateaus.
	\label{fig:rbar}
	}
	\end{center}
\end{figure}

\section{Comparison to Other Tuning-Free Methods} \label{app:other-optimizers}

\Cref{subsec:other-optimizers} discusses a comparison between \DoG{} and other parameter-free optimizers. In this section, we provide further details on the experiments. 

\subsection{Parameter-free SGD}\label{app-subsec:pf-sgd}
\citet{carmon2022making} propose a bisection procedure for tuning the SGD learning rate. A direct implementation of this method would need to perform at least 4 or 5 bisection steps and therefore, \emph{in the best case}, perform similarly to our instance-tuned SGD baseline. Since our learning rate tuning employs a tight grid of values selected using some prior knowledge of the problem, and since we select learning rates based on validation set performance and not a step size certificate, instance-tuned SGD is likely a conservative upper bound on the performance of bisection approach.

Similar to instance tuned SGD, the bisection procedure has increased computational cost relative to \DoG that is due to the need for multiple SGD runs. That is, performing 5 steps of bisection where each SGD call has the same step budget as \DoG consumes 5 times more compute than \DoG. We may also consider a situation where each bisection step uses only 20\% of the \DoG compute budget, leading to equal overall cost. In this setting, the ``equalized compute budget'' comparison we perform in \Cref{app:expeirments-equalized} and \Cref{fig:bisection} provides a conservative upper bound on the bisection performance, indicating it is likely to under-perform \DoG.

\subsection{Stochastic Polyak step-size}\label{app-subsec:sps}
We apply the Stochastic Polyak Step (SPS) proposed by \citet{loizou2021stochastic} using their open-source implementation\footnote{\url{https://github.com/IssamLaradji/sps}} to a subset of our fine-tuning testbed, and present the results in \Cref{table:language-comp-results,table:vision-comp-results}. For the vision experiments, the SPS with the hyper-parameters proposed in the paper ($c=0.2$, $\tau=2$) and initial step size of 1.0 (the default in the code) worked reasonably well, but not as well as DoG. For the language experiments the same algorithm diverges; we find initial learning rate of 0.01 worked reasonably well, but again not as well as DoG (we also attempted an initial learning rate of 0.0001, which produced worse results). For vision tasks, similarly tuning the initial step size did not significantly improve performance. We run 5 random seeds per experiment, and average the results across seeds. \DoG outperforms SPS in 22 out of 34 task/model combinations, and by 2.3 percentage points on average. \LDoG further increases this gap by outperforming SPS in 24 out of 34 pairs, with an average of 5.3 percentage points.

\subsection{D-adaptation}\label{app-subsec:dadaptation}

\paragraph{Empirical comparison to D-adapt SGD and Adam.} We perform a preliminary empirical evaluation of the practical algorithms proposed in \citet{defazio2023learning} using the code they release\footnote{\url{https://github.com/facebookresearch/dadaptation}} and a subset of our fine-tuning testbed. As \Cref{table:language-comp-results,table:vision-comp-results} show, D-adapt SGD and and D-adapt Adam perform reasonably well but slightly worse than \DoG, and noticeably worse than \LDoG and Adam. 
\DoG outperforms D-adapt SGD in 24 out of 34 task/model combinations, and by 1.9 percentage points on average. \LDoG further increases this gap by outperforming D-adapt SGD in 30 out of 34 pairs, with an average of 4.9 percentage points.
D-adapt Adam is less stable on many of the tasks in our testbed, being outperformed by \DoG in 26 out of 34 task/model combinations, and by \LDoG in 27, with an average of 10.8 and  13.8 percentage points respectively.

\paragraph{Theoretical comparison to Algorithm 1 of \citet{defazio2023learning}.}
\citet{defazio2023learning} carry out their main theoretical analysis on a ``Parameter Free Dual Averaging'' (PFDA) method. We now provide some additional remarks comparing PFDA and \DoG. The iterate $x_t$ in PFDA is
\[x_t = x_0 - \frac{1}{\sqrt{G_t}} \sum_{i \le t} q_i g_i\] where $G_ t = \sum_{i\le t} \|g_i\|^2$ and $q_i$ is a lower bound on the distance to optimality (denoted $d_i$ in~\citep{defazio2023learning}). In contrast, the \DoG iterates are \[x_t = x_0 - \sum_{i\le t}\frac{\bar{r}_i}{\sqrt{G_i}} g_i\] where $\bar{r}_t = \max_{i \le t}\| x_i - x_0 \|$. While both $d_t$ in PFDA and $\bar{r}_t$ are lower bounds for (a constant factor times) the distance to optimality, only DoG aims to approximate $\eta_\star = \frac{\|x_0 -x_\star \|}{\sqrt{G_t}}$; PFDA instead approximates the optimal step size for dual averaging. 

The dual averaging prescription of putting the factor $1/\sqrt{G_t}$ outside the summation defining $x_t$ likely hurts performance in practice. The practical D-adapt SGD and D-adapt Adam methods that~\citet{defazio2023learning} propose do not follow this prescription. Consequently, these algorithms are very different from PFDA and have no theoretical guarantees. 

\subsection{Continuous Coin Betting}\label{app-subsec:cocob}
\citet{orabona2017training} propose a parameter-free algorithm based on coin-betting schemes and demonstrate its effectiveness in training neural networks. We use an open source implementation\footnote{\url{https://github.com/bremen79/parameterfree}} to apply COCOB on the same subset of our testbed as the other experiments in this section, and present the results in \Cref{table:language-comp-results,table:vision-comp-results}. While the default parameters work well for ResNet50 (slightly better than \DoG and similar to \LDoG), they produced poor results on the other models. We found that increasing the $\alpha$ parameter (which roughly corresponds to the number of ``warmup'' steps) from a constant 100 to 10\% of the step budget improves the results dramatically for all transformer-based models, though in most cases \DoG continues to significantly outperform it. Additionally, we find that in almost all cases, our averaging scheme benefits COCOB.

\newcommand{\LanguageCompTableCaption}{ \label{table:language-comp-results}Average (std) performance of RoBERTa-b and T5-b on language tasks, when fine-tuned with different optimization algorithms. \DoG{} uses $\reps=10^{-6}(1+\norm{x_0})$ and \LDoG{} uses $\reps=10^{-8}(1+\norm{x_0})$. SPS uses $c=0.2$, $\tau=2$ as recommended by \citet{loizou2021stochastic}, but initial step size of 0.01 as the recommended value 1.0 diverged for some tasks. Still, when fine-tuning RoBERTa-b on RTE, 3 out of 5 training runs diverged; for this case we report the mean of the two successful runs and omit the standard deviation. For COCOB we set $\alpha$ to be 10\% of the total steps per task, as the default of $\alpha=100$ failed to outperform a random guess. We measure performance as detailed in \Cref{table:tasks-configurations}.}

\begin{table}[t]
    \centering
    \footnotesize
       \setlength{\belowcaptionskip}{-10pt}
    \notarxiv{\caption{\LanguageCompTableCaption}}
    \resizebox{\linewidth}{!}{\begin{tabular}{@{\extracolsep{4pt}}lcccccccc} 
    
        \toprule

     \textbf{Model} & \textbf{Optimizer} &  CoLA & MRPC & QNLI & RTE & SQuAD & SST-2 & Avg. \\ \cline{1-1} \cline{2-2} \cline{3-3} \cline{4-4} \cline{5-5} \cline{6-6} \cline{7-7} \cline{8-8} \cline{9-9} 
     \rule{-2pt}{2.6ex}

     \multirow[c]{5}{*}{\textbf{RoBERTa-b}} & \textbf{D-Adapt (Adam)} & 0.0 {\smaller (0.00)} & 83.4 {\smaller (4.93)} & 66.7 {\smaller (22.13)} & 68.4 {\smaller (14.40)} & 7.0 {\smaller (0.16)} & 58.2 {\smaller (17.88)} & 47.28 \\
 & \textbf{D-Adapt (SGD)} & 49.4 {\smaller (27.59)} & \bfseries 91.6 {\smaller (0.43)} & 91.9 {\smaller (0.67)} & \bfseries 81.4 {\smaller (1.04)} & 91.5 {\smaller (0.13)} & 94.1 {\smaller (0.38)} & 83.32 \\
 & \textbf{SPS} & 49.8 {\smaller (6.79)} & 88.4 {\smaller (3.21)} & 91.8 {\smaller (0.21)} & 52.7 & 90.1 {\smaller (0.10)} & 94.1 {\smaller (0.28)} & 77.70 \\
 & \textbf{COCOB} & 0.0 {\smaller (0.00)} & 81.2 {\smaller (0.00)} & 92.5 {\smaller (0.33)} & 52.7 {\smaller (0.00)} & 1.1 {\smaller (0.00)} & 92.0 {\smaller (1.19)} & 53.25 \\
 & \textbf{\DoG} & 62.8 {\smaller (1.17)} & 91.6 {\smaller (0.29)} & 92.6 {\smaller (0.15)} & 78.5 {\smaller (2.91)} & 91.3 {\smaller (0.17)} & \bfseries 94.9 {\smaller (0.26)} & 85.28 \\
 & \textbf{\LDoG} & \bfseries 63.3 {\smaller (0.32)} & 91.5 {\smaller (0.19)} & \bfseries 92.8 {\smaller (0.28)} & 80.1 {\smaller (1.00)} & \bfseries 91.8 {\smaller (0.18)} & 94.8 {\smaller (0.54)} & \bfseries 85.72 \\
\cline{1-9}
\multirow[c]{5}{*}{\textbf{T5-b}} & \textbf{D-Adapt (Adam)} & 11.1 {\smaller (2.51)} & 81.8 {\smaller (0.65)} & 85.0 {\smaller (1.91)} & 58.1 {\smaller (1.35)} & 90.4 {\smaller (0.06)} & 87.1 {\smaller (3.00)} & 68.92 \\
 & \textbf{D-Adapt (SGD)} & 0.0 {\smaller (0.00)} & 92.5 {\smaller (0.59)} & 92.9 {\smaller (0.04)} & 81.2 {\smaller (0.78)} & 90.4 {\smaller (0.06)} & 80.3 {\smaller (2.67)} & 72.88 \\
 & \textbf{SPS} & 39.1 {\smaller (2.35)} & \bfseries 92.9 {\smaller (1.21)} & 93.2 {\smaller (0.10)} & 80.9 {\smaller (1.82)} & 90.5 {\smaller (0.04)} & 94.3 {\smaller (0.31)} & 81.82 \\
 & \textbf{COCOB} & 59.2 {\smaller (1.02)} & 86.2 {\smaller (2.93)} & 93.4 {\smaller (0.12)} & \bfseries 83.5 {\smaller (1.66)} & 90.5 {\smaller (0.05)} & 94.8 {\smaller (0.35)} & 84.60 \\
 & \textbf{\DoG} & 7.3 {\smaller (6.78)} & 92.8 {\smaller (0.35)} & 93.1 {\smaller (0.09)} & 81.7 {\smaller (3.06)} & \bfseries 90.6 {\smaller (0.05)} & 94.1 {\smaller (0.19)} & 76.60 \\
 & \textbf{\LDoG} & \bfseries 59.9 {\smaller (1.43)} & 91.9 {\smaller (0.32)} & \bfseries 93.6 {\smaller (0.02)} & 83.1 {\smaller (0.78)} & 90.3 {\smaller (0.02)} & \bfseries 95.0 {\smaller (0.19)} & \bfseries 85.63 \\

 \bottomrule
    \end{tabular}}
    \arxiv{\caption{\LanguageCompTableCaption}\vspace{\baselineskip}}
    \end{table} %
\newcommand{\VisionCompTableCaption}{Average (std) test accuracy across seeds for vision tasks, when fine-tuned with different optimization algorithms. \DoG{} and \LDoG{} use $\reps=10^{-4}(1+\norm{x_0})$. SPS uses $c=0.2$, $\tau=2$ and initial step size of 1 as recommended by \citet{loizou2021stochastic}. For COCOB we set $\alpha$ to be 10\% of the total steps per task; the default value of $\alpha=100$ performed well on ResNet50 but very badly on ViT-B/32.}

\begin{table}[t]
    \centering
    \footnotesize
\notarxiv{\caption{\VisionCompTableCaption}\label{table:vision-comp-results}}
    
    \resizebox{\linewidth}{!}{\begin{tabular}{@{\extracolsep{4pt}}lcccccccccccccc} 
        \toprule
    \textbf{Model} & \textbf{Optimizer} & Caltech101 & CIFAR-100 & Clevr-Dist & DMLab & dSprites-Ori & DTD & Flowers102 & Pets & Resisc45 & Retinopathy & Sun397 & SVHN & Avg. \\ \cline{1-1} \cline{2-2} \cline{3-3} \cline{4-4} \cline{5-5} \cline{6-6} \cline{7-7} \cline{8-8} \cline{9-9} \cline{10-10} \cline{11-11} \cline{12-12} \cline{13-13} \cline{14-14} \cline{15-15}
    \rule{-2pt}{2.6ex}
    
    \multirow[c]{5}{*}{\textbf{ResNet50}} & \textbf{D-Adapt (Adam)} & 79.0 {\smaller (1.60)} & 82.9 {\smaller (0.23)} & 92.3 {\smaller (0.50)} & 75.2 {\smaller (0.97)} & \bfseries 96.3 {\smaller (0.01)} & 60.1 {\smaller (0.51)} & 84.6 {\smaller (1.73)} & 85.1 {\smaller (0.95)} & 94.1 {\smaller (0.28)} & \bfseries 83.4 {\smaller (0.07)} & 72.9 {\smaller (0.11)} & 97.0 {\smaller (0.21)} & 83.58 \\
 & \textbf{D-Adapt (SGD)} & 86.2 {\smaller (0.93)} & 84.3 {\smaller (0.15)} & 89.6 {\smaller (0.60)} & 74.8 {\smaller (0.73)} & 95.9 {\smaller (0.04)} & 64.4 {\smaller (1.31)} & 84.9 {\smaller (1.12)} & 92.0 {\smaller (0.25)} & 94.7 {\smaller (0.17)} & 82.7 {\smaller (0.12)} & 73.0 {\smaller (0.09)} & 97.1 {\smaller (0.06)} & 84.97 \\
 & \textbf{SPS} & 86.4 {\smaller (1.14)} & 80.2 {\smaller (1.20)} & \bfseries 92.9 {\smaller (0.25)} & \bfseries 76.2 {\smaller (0.60)} & 96.0 {\smaller (0.05)} & 63.7 {\smaller (0.83)} & 84.3 {\smaller (1.45)} & 92.4 {\smaller (0.19)} & 94.9 {\smaller (0.22)} & 83.3 {\smaller (0.17)} & 70.9 {\smaller (0.82)} & \bfseries 97.5 {\smaller (0.03)} & 84.89 \\
 & \textbf{COCOB} & 86.3 {\smaller (1.53)} & 83.5 {\smaller (0.21)} & 87.4 {\smaller (0.43)} & 68.4 {\smaller (0.72)} & 95.8 {\smaller (0.03)} & 66.4 {\smaller (0.31)} & 87.0 {\smaller (1.11)} & 92.4 {\smaller (0.35)} & 93.8 {\smaller (0.18)} & 82.0 {\smaller (0.10)} & 72.2 {\smaller (0.16)} & 96.8 {\smaller (0.07)} & 84.33 \\
 & \textbf{\DoG} & 86.8 {\smaller (0.62)} & \bfseries 84.8 {\smaller (0.37)} & 89.3 {\smaller (0.58)} & 71.4 {\smaller (0.77)} & 95.7 {\smaller (0.09)} & 66.4 {\smaller (1.48)} & 85.8 {\smaller (2.64)} & \bfseries 92.9 {\smaller (0.38)} & 94.6 {\smaller (0.33)} & 82.6 {\smaller (0.16)} & 73.5 {\smaller (0.41)} & 96.5 {\smaller (0.10)} & 85.02 \\
 & \textbf{\LDoG} & \bfseries 87.6 {\smaller (1.15)} & 84.6 {\smaller (0.38)} & 90.8 {\smaller (0.38)} & 75.0 {\smaller (0.44)} & 96.0 {\smaller (0.05)} & \bfseries 69.1 {\smaller (1.30)} & \bfseries 90.2 {\smaller (2.02)} & 92.4 {\smaller (0.36)} & \bfseries 95.6 {\smaller (0.46)} & 82.2 {\smaller (0.15)} & \bfseries 74.6 {\smaller (0.32)} & 97.4 {\smaller (0.08)} & \bfseries 86.29 \\
\cline{1-15}
\multirow[c]{5}{*}{\textbf{ViT-B/32}} & \textbf{D-Adapt (Adam)} & 87.6 {\smaller (0.48)} & 91.8 {\smaller (0.31)} & \bfseries 89.0 {\smaller (0.42)} & 70.7 {\smaller (0.48)} & 7.6 {\smaller (0.00)} & 69.4 {\smaller (1.55)} & 94.7 {\smaller (0.92)} & 86.6 {\smaller (1.42)} & 95.5 {\smaller (0.32)} & \bfseries 80.1 {\smaller (0.09)} & 76.2 {\smaller (0.38)} & 97.9 {\smaller (0.05)} & 78.92 \\
 & \textbf{D-Adapt (SGD)} & 88.7 {\smaller (0.58)} & 91.8 {\smaller (0.20)} & 84.2 {\smaller (2.39)} & 70.5 {\smaller (0.17)} & 44.0 {\smaller (31.87)} & 74.9 {\smaller (0.38)} & 98.2 {\smaller (0.43)} & 91.5 {\smaller (0.72)} & 96.3 {\smaller (0.06)} & 79.8 {\smaller (0.05)} & 76.3 {\smaller (0.13)} & 97.7 {\smaller (0.03)} & 82.82 \\
 & \textbf{SPS} & \bfseries 89.7 {\smaller (0.51)} & 91.4 {\smaller (0.18)} & 84.6 {\smaller (1.18)} & \bfseries 71.3 {\smaller (0.27)} & 7.6 {\smaller (0.00)} & 74.4 {\smaller (0.96)} & 98.4 {\smaller (0.33)} & 91.1 {\smaller (0.49)} & 96.2 {\smaller (0.17)} & 79.9 {\smaller (0.09)} & 77.5 {\smaller (0.46)} & \bfseries 98.0 {\smaller (0.04)} & 80.01 \\
 & \textbf{COCOB} & 89.2 {\smaller (1.01)} & 91.5 {\smaller (0.19)} & 84.9 {\smaller (0.88)} & 70.7 {\smaller (0.70)} & 7.6 {\smaller (0.00)} & 74.7 {\smaller (0.92)} & 98.4 {\smaller (0.48)} & 90.8 {\smaller (0.26)} & \bfseries 96.5 {\smaller (0.12)} & 80.0 {\smaller (0.15)} & 77.1 {\smaller (0.22)} & 97.9 {\smaller (0.07)} & 79.94 \\
 & \textbf{\DoG} & 89.5 {\smaller (1.26)} & 92.5 {\smaller (0.22)} & 85.0 {\smaller (0.27)} & 69.5 {\smaller (0.11)} & 67.7 {\smaller (36.66)} & 75.5 {\smaller (0.71)} & 98.9 {\smaller (0.25)} & \bfseries 92.4 {\smaller (0.16)} & 96.4 {\smaller (0.10)} & 79.7 {\smaller (0.01)} & 77.8 {\smaller (0.13)} & 97.7 {\smaller (0.08)} & 85.22 \\
 & \textbf{\LDoG} & 89.6 {\smaller (0.81)} & \bfseries 92.8 {\smaller (0.15)} & 86.0 {\smaller (0.40)} & 70.7 {\smaller (0.29)} & \bfseries 95.3 {\smaller (0.08)} & \bfseries 75.8 {\smaller (0.71)} & \bfseries 99.0 {\smaller (0.26)} & 92.3 {\smaller (0.45)} & 96.5 {\smaller (0.17)} & 79.8 {\smaller (0.07)} & \bfseries 78.3 {\smaller (0.26)} & 97.8 {\smaller (0.04)} & \bfseries 87.82 \\

 \bottomrule
    \end{tabular}}
    \arxiv{\caption{\VisionCompTableCaption}\label{table:vision-comp-results}}
    \end{table}

\end{document}